\documentclass{article}

\usepackage{arxiv}

\usepackage[utf8]{inputenc} 
\usepackage[T1]{fontenc}    
\usepackage{hyperref}       
\usepackage{url}            
\usepackage{booktabs}       
\usepackage{amsfonts}       
\usepackage{nicefrac}       
\usepackage{microtype}      
\usepackage{lipsum}
\usepackage{graphicx}
\graphicspath{ {./images/} }
\usepackage{subfigure}
\usepackage{algorithm}
\usepackage{algorithmic}
\usepackage{amssymb,amsfonts,amsmath,amsthm}
\usepackage{multicol,multirow}
\usepackage{graphicx}
\usepackage{caption}
\usepackage{color}

\newtheorem{theorem}{Theorem}
\newtheorem{remark}{Remark}

\newtheorem{lemma}{Lemma}
\newtheorem{definition}{Definition}
\newtheorem{assumption}{Assumption}

\title{Improved Learning Rates for Stochastic Optimization}

\author{%
  Shaojie Li,  Pengwei Tang,  Yong Liu\thanks{Corresponding Author.} \\
  Renmin University of China\\
  \texttt{\{2020000277, tangpwei, liuyonggsai\}@ruc.edu.cn} 
}

\begin{document}
\maketitle
\begin{abstract}
Stochastic optimization is a cornerstone of modern machine learning. This paper studies the generalization performance of two classical stochastic optimization algorithms: stochastic gradient descent (SGD) and Nesterov's accelerated gradient (NAG). We establish new learning rates for both algorithms, with improved guarantees in some settings or comparable rates under weaker assumptions in others. We also provide numerical experiments to support the theory.
\end{abstract}

\section{Introduction}
Stochastic optimization plays an essential role in machine learning, as many problems can be cast in this framework \cite{shalev2014understanding}. The goal is to optimize an expected objective over a feasible set $\mathcal{W} \subseteq \mathbb{R}^d$:
\begin{align}\label{eq111}
\min_{\mathbf{w} \in \mathcal{W}} F(\mathbf{w}) := \mathbb{E}_{z \sim \rho}  [f(\mathbf{w};z)],
\end{align}
where $f(\cdot;z): \mathcal{W} \mapsto \mathbb{R}$ depends on a random variable $z \in \mathcal{Z}$ drawn from a distribution $\rho$. In statistical learning, this is also referred to as risk minimization problem: $F(\mathbf{w})$ is the population risk, $z$ denotes a single example, $\mathbf{w}$ represents a hypothesis/model, and $f(\cdot;\cdot)$ is a loss. A canonical instance is supervised learning \cite{vapnik2013nature}:
\begin{align}\label{eq112}
\min_{\mathbf{w} \in \mathcal{W}} F(\mathbf{w}) := \mathbb{E}_{(x,y) \sim \rho} [\ell(\mathbf{w};x,y)],
\end{align}
where $z\in \mathcal{Z} = \mathcal{X} \times \mathcal{Y} \subset \mathbb{R}^d \times \mathbb{R}$ and $ f(\mathbf{w};z) = \ell(\mathbf{w};x,y)$. Since the underlying distribution $\rho$ is unknown, the minimization of $F(\mathbf{w})$ is not accessible, and we instead minimize the empirical risk over $n$ i.i.d. samples $S=\{ z_1,...,z_n  \} \sim \rho^n$:
\begin{align}\label{eq:empirical risk}
 \min_{\mathbf{w} \in \mathcal{W}} F_S(\mathbf{w}) := \frac{1}{n} \sum_{i=1}^n f(\mathbf{w} ; z_i).
\end{align} 
Stochastic optimization is typically used to approximately solve (\ref{eq:empirical risk}) and output a learned model $\mathbf{w}_S$. Its generalization performance is then evaluated by how well $\mathbf{w}_S$ performs on unseen data, often through the excess risk
\begin{align}\label{excess-risk}
F( \mathbf{w}_S) -F^{\ast}, \qquad F^{\ast} := \inf_{\mathbf{w} \in \mathcal{W}} F(\mathbf{w}).
\end{align}

In this work, we study the \emph{learning rate}, namely, the convergence rate of an upper bound on (\ref{excess-risk}), for two classical stochastic optimization algorithms for solving \eqref{eq:empirical risk}: stochastic gradient descent (SGD) and Nesterov’s accelerated gradient (NAG). We establish $\mathcal{O}(1/n^2)$-type learning rates for both algorithms in several settings; compared with prior work, our analysis gives improved bounds in some regimes and comparable rates under weaker assumptions in others. Our results also suggest that, under suitable curvature conditions, the generalization performance improves with optimization accuracy and therefore do not exhibit the early-stopping tradeoff that appears in some prior analyses.

\textbf{Organization.} The remainder of this paper is organized as follows. In Section \ref{section2}, we discuss related works. Notations and assumptions used are introduced in Section \ref{section23}. In Section \ref{mainresults}, we provide our main results. In Section \ref{sectionexp}, experiments are presented to support our theory. Section \ref{section6} concludes this paper. All proofs are postponed to the appendix.

\section{Related Work}\label{section2}
There are two mainstream approaches to studying generalization in stochastic optimization: algorithmic stability and uniform convergence. We review the related literature accordingly.

\textbf{Algorithmic Stability.}
Algorithmic stability is a fundamental tool in statistical learning theory \cite{shalev2010learnability}. Informally, an algorithm is stable if a small perturbation of the training set leads to only a small change in the output. A classical notion is \emph{uniform stability} \cite{bousquet2002stability}: an algorithm $\mathbf{w}:\mathcal{Z}^n \mapsto \mathcal{W}$ is uniformly $\varepsilon$-stable if for all datasets $S,S' \in \mathcal{Z}^n$ that differ in at most one example,
\begin{align} \label{uniform stability}
\left| f( \mathbf{w}_S; z)  - f( \mathbf{w}_{S'} ;z) \right| \leq \varepsilon, \qquad  \forall z \in \mathcal{Z},
\end{align}
which in turn yields a bound on the generalization error. Other variants of stability include uniform argument stability \cite{liu2017algorithmic,bassily2020stability}, hypothesis stability \cite{bousquet2002stability,charles2018stability}, hypothesis set stability \cite{foster2019hypothesis}, on-average stability \cite{shalev2010learnability,kuzborskij2018data,zhang2021stability}, and locally elastic stability \cite{deng2020toward}. 
Despite this breadth, stability-based bounds are often obtained only in expectation \cite{koren2015fast,gonen2017average,vavskevivcius2020suboptimality,mourtada2021distribution,lei2021sharper}. More recent work has focused on establishing high-probability bounds for uniformly stable algorithms \cite{feldman2018generalization,feldman2019high,bousquet2020sharper,klochkov2021stability,fan2024high,zhu2025stability}. In particular, the fastest $\mathcal{O}(1/n^2)$ rate in this line was recently obtained by \cite{zhu2025stability}. However, their analysis relies on a uniformly bounded gradient condition, which in the differentiable setting is implied by a Lipschitz continuity assumption on the loss and is typically required in stability-based arguments. Such a condition can be restrictive in modern stochastic optimization problems. This motivates the search for an alternative route to fast high-probability rates under weaker assumptions.

\textbf{Uniform Convergence.}
Uniform convergence provides such an alternative perspective and is another central approach to generalization \cite{shalev2010learnability}. It controls the deviation between empirical and population quantities directly. Formally, it requires that, for any distribution $\rho$ over $\mathcal{Z}$, the empirical risks of all hypotheses in the class converge uniformly to their population risks \cite{shalev2010learnability}:
\begin{align}\label{uniform convergence}
\mathbb{P}_{S \sim \rho^n} \left[\sup_{\mathbf{w} \in \mathcal{W}}\left|F(\mathbf{w}) - F_S(\mathbf{w}) \right| > \varepsilon \right] \stackrel{n \rightarrow  \infty }{\longrightarrow} 0,\qquad \forall \varepsilon >0. 
\end{align}
In this paper, we focus on \emph{uniform convergence of gradients}, namely the uniform deviation between the population gradient and the empirical gradient, obtained by replacing $F$ and $F_S$ in (\ref{uniform convergence}) with $\nabla F$ and $\nabla F_S$. This gradient-based viewpoint has become a standard tool in stochastic optimization \cite{mei2018landscape,zhang2017empirical,foster2018uniform,zhang2019stochastic,lei2021learning,davis2018uniform,xu2025towards}. Existing techniques include covering-number arguments \cite{mei2018landscape}, vector-valued Rademacher complexity for generalized linear models \cite{foster2018uniform}, Rademacher chaos of order two \cite{lei2021learning,de2012decoupling}, and graphical uniform convergence for certain nonsmooth settings \cite{davis2018uniform}.
Among these developments, the localized technique of \cite{xu2025towards} is particularly relevant to our work, since it substantially sharpens earlier global uniform-convergence arguments. Building on this line of research, our analysis uses modern localized uniform-convergence tools \cite{xu2025towards} to derive improved bounds for stochastic optimization and, in several settings, to obtain comparable rates under weaker assumptions.

\section{Preliminaries}\label{section23}
We begin with notation. For $\mathbf w=(w_1,\dots,w_d)\in\mathbb R^d$, let $\|\mathbf w\|$ denote the Euclidean norm, i.e., $\|\mathbf w\|=\sqrt{\sum_{j=1}^d w_j^2}$. 
Let $B(\mathbf{w}_0 , R) := \{ \mathbf{w} \in \mathbb{R}^d: \|\mathbf{w}-\mathbf{w}_0 \| \leq R \}$ denote a ball with center $\mathbf{w}_0 \in \mathbb{R}^d$ and radius $R$. We assume that the set $\mathcal{W}$ satisfies the condition $\mathcal{W} \subseteq B(\mathbf{w}^{\ast} , R) $, where $\mathbf w^\ast \in \arg\min_{\mathbf w\in\mathcal W}F(\mathbf w)$,
and that all iterates stay in the domain, i.e., for all $t\ge 1$, $\mathbf w_t,\mathbf y_t\in\mathcal W$ (e.g., enforced by projection or by a suitable stepsize choice). We write $A\asymp B$ if there exist universal constants $C_1,C_2>0$ such that $C_1A\le B\le C_2A$. Standard order of magnitude notation (e.g., $\mathcal O(\cdot)$) is used throughout. 

\subsection{Standard Assumptions in Stochastic Optimization}
We consider differentiable losses $f:\mathcal W\times\mathcal Z\to\mathbb R_+$ and introduce some standard assumptions. 
\begin{assumption}[Lipschitz Continuity]\label{assum666}
There exists $L>0$ such that for all $z\in\mathcal Z$ and $\mathbf w_1,\mathbf w_2\in\mathcal W$,
\begin{align*}
| f(\mathbf{w}_1;z)  - f( \mathbf{w}_2;z) | \leq L \| \mathbf{w}_1 - \mathbf{w}_2 \|.
\end{align*}
\end{assumption}
\begin{remark}\rm{}\label{remark1ci}
When $f$ is differentiable, the lipschitz continuity of $f(\mathbf w;z)$ is equivalent to uniform boundedness of its gradient \cite{hardt2016train}
\begin{align*}
\|\nabla f(\mathbf w;z)\| \le L, \quad \textup{for all $z\in\mathcal Z$ and $\mathbf w \in\mathcal W$}.
\end{align*}
This condition can be relaxed through the following Assumptions~\ref{assu7} and~\ref{assu5}.
\end{remark}
\begin{assumption}[Smoothness]\label{assu4}
There exists $\beta>0$ such that for all $z\in\mathcal Z$ and $\mathbf w_1,\mathbf w_2\in\mathcal W$,
\begin{align*}
\|\nabla f(\mathbf{w}_1;z) - \nabla f(\mathbf{w}_2;z) \| \leq \beta \| \mathbf{w}_1 - \mathbf{w}_2 \|.
\end{align*}
\end{assumption}
\begin{remark}\rm{}\label{assnonsmo}
Smoothness means Lipschitz continuity of the gradient \cite{reddi2016stochastic}.
\end{remark}

\begin{assumption}[Relaxed Boundedness of Gradient]\label{assu7}
Assume the existence of $G>0$ satisfying for $\forall t \in \mathbb{N}, z\in \mathcal{Z}$
\begin{align*}
&\sqrt{\eta_t} \|\nabla f(\mathbf{w}_{t};z) \| \leq G; \quad\sqrt{\eta_t} \|\nabla f(\mathbf{y}_{t};z) \| \leq G .
\end{align*}
\end{assumption}
\begin{remark}\rm{}\label{shiremark}
Compared with the common global bounded-gradient condition   
\begin{align*}
\|\nabla f(\mathbf w;z)\| \le L, \quad \textup{for all $z\in\mathcal Z$ and $\mathbf w \in\mathcal W$},
\end{align*}
Assumption \ref{assu7} is milder in our setting because the stochastic gradient is scaled by the stepsize, and the stepsizes considered in this paper decrease over the course of training \cite{lei2021learning}.
\end{remark}

\begin{assumption}[Stochastic Gradient Noise]\label{assu8}
Assume the existence of $\sigma > 0$ satisfying
\begin{align*}
\mathbb{E}_{j_t} \left[ \| \nabla f(\mathbf{w}_{t};z_{j_t}) - \nabla F_S (\mathbf{w}_{t}) \|^2 \right]\leq \sigma^2   , \quad \forall t \in \mathbb{N},
\end{align*}
where $\mathbb{E}_{j_t}$ denotes the expectation w.r.t. $j_t$.
\end{assumption}
\begin{remark}\rm{}\label{remark6}
Assumption~\ref{assu8} states that the conditional variance of the stochastic gradient is bounded; it is standard in stochastic optimization \cite{nemirovski2009robust,ghadimi2013stochastic,bottou2018optimization}. The related works \cite{li2022high,lihigh} instead assumed a \emph{sub-Weibull} tail (study stochastic gradient gradient and Polyak’s momentum, respectively) for the gradient noise $\nabla f(\mathbf w_t;z_{j_t})-\nabla F_S(\mathbf w_t)$: 
\begin{align*}
\mathbb{E}_{j_t}\Big[\exp \big((\| \nabla f(\mathbf{w}_t;z_{j_t}) - \nabla F_S(\mathbf{w}_{t})\|/K)^{\frac{1}{\theta}} \big) \Big] \leq 2
\end{align*}
such that $\theta \geq \frac{1}{2}$, which implies polynomial moment control $\mathbb E\|X\|^k\le (K'k^\theta)^k$ for all $k\ge1$ (see Theorem 2.1 in \cite{vladimirova2020sub} or \cite{kuchibhotla2018moving}), where $K'$ and $K$ differ each by a constant depend only on $\theta$.
In particular, the variance is finite. Therefore, Assumption~\ref{assu8} is strictly weaker than the sub-Weibull noise assumption and covers a broader class of distributions with finite variance.
\end{remark}

\begin{assumption}[Polyak-{\L}ojasiewicz]\label{assu10}
Fix a set $\mathcal{W}$ and let $f^{\ast} = \inf_{\mathbf{w} \in \mathcal{W}} f(\mathbf{w})$. For any function $f: \mathcal{W} \mapsto \mathbb{R}$, we say it satisfies the Polyak-{\L}ojasiewicz (PL) condition with parameter $\mu > 0$ on $\mathcal{W}$ if for all $\mathbf{w} \in \mathcal{W}$,
\begin{align*}
f (\mathbf{w}) - f^{\ast} \leq \frac{1}{2 \mu} \| \nabla f(\mathbf{w}) \|^2.
\end{align*}
\end{assumption}
\begin{remark}\rm{}
The PL inequality simply requires that the gradient
grows faster than a quadratic function as we move away
from the optimal function value. It is a weak curvature condition that facilitates linear convergence without requiring strong convexity \cite{karimi2016linear}. Relationships between PL and other curvature notions are shown in Lemma~\ref{appendixa}. 
\end{remark}
\subsection{Standard Assumptions in Generalization Theory}
\begin{assumption}[Bernstein Condition] \label{assu5}
There exists $B_\ast>0$ such that for all integers $k \ge 2$,
\begin{align*}
\mathbb{E}\left [ \|\nabla f(\mathbf{w}^{\ast}; z) \|^k \right] \leq \frac{1}{2} k! \mathbb{E} \left[ \|\nabla f(\mathbf{w}^{\ast}; z) \|^2 \right] B_{\ast}^{k-2}.
\end{align*}
\end{assumption}
\begin{remark}\rm{}
Assumption \ref{assu5} imposes the Bernstein condition on the random variable $\|\nabla f(\mathbf{w}^\ast; z)\|$. This condition is essentially equivalent to be a sub-exponential random variable \cite{wainwright2019high}, and is therefore weaker than assuming that $\|\nabla f(\mathbf{w}^\ast; z)\|$ is uniformly bounded for all $z \in \mathcal{Z}$. Note that the latter boundedness requirement is imposed only at the optimum $\mathbf{w}^\ast$, rather than uniformly over $\mathcal{W}$. Hence, the Bernstein condition is strictly weaker than boundedness of the underlying random variable. 
\end{remark}
\section{Main Results}\label{mainresults}
This section gives learning rates for SGD and NAG. 

\subsection{Stochastic Gradient Descent}\label{section3}
SGD is widely used due to its simplicity of implementation, low memory requirement, and strong empirical performance \cite{bottou201113,bottou2018optimization}. At iteration $t$, with stepsize $\eta_t \ge 0$, SGD samples an index $j_t$ i.i.d. uniformly from $\{1,\ldots,n\}$ and updates
\begin{align}\label{eq1}
\mathbf{w}_{t+1} = \mathbf{w}_t - \eta_t \nabla f (\mathbf{w}_t;z_{j_t}),
\end{align}
where $\nabla f (\mathbf{w}_t;z_{j_t})$ denotes the gradient of $f$ with respect to (w.r.t.) the first argument. The randomness in $\mathbf{w}_t$ comes from both the sample draw $S$ and the index sequence $\{j_t\}_t$. We evaluate generalization performance through both the averaged iterate and the last iterate. 

We begin with an average-iterate excess risk bound $(\sum_{t = 1}^T \eta_t )^{-1} \sum_{t = 1}^T \eta_t F(\mathbf{w}_{t}) - F^{\ast}$ and then turn to the last-iterate excess risk bound $F(\mathbf{w}_{T+1})-F^\ast$. For simplicity, we assume $\mathbf{w}_1=0$ without loss of generality, since all results extend to arbitrary $\mathbf{w}_1 \in \mathcal{W}$.
 
\begin{theorem} \label{theo67}
Suppose Assumptions \ref{assu4}, \ref{assu7}, \ref{assu8} and \ref{assu5} hold,  and suppose the population risk $F$ satisfies Assumption \ref{assu10} with parameter $\mu$. Let $\{ \mathbf{w}_t\}_t$ be the sequence produced by (\ref{eq1}) with $\eta_t = \eta_1 t^{- 1/2}$ such that $\eta_1 \leq \frac{1}{2\beta}$. 
 When $n \geq \frac{c\beta^2(d+ \log(\frac{16 \log(2n R +2)}{\delta}))}{\mu^2}$, for any $\delta >0$, with probability at least $1 - \delta$, choosing $T \asymp  n^4$ yields
\begin{align*}
\left(\sum_{t = 1}^T \eta_t \right)^{-1} \sum_{t = 1}^T \eta_t F(\mathbf{w}_{t}) - F^{\ast}  =  \mathcal{O} \left(\frac{\log^2(\frac{1}{\delta})}{n^2} + \frac{ F^{\ast}  \log(\frac{1}{\delta})}{n} \right);
\end{align*}
if further assuming $F^{\ast}  = \mathcal{O}(1/n)$, we have 
\begin{align*}
\left(\sum_{t = 1}^T \eta_t \right)^{-1} \sum_{t = 1}^T \eta_t F(\mathbf{w}_{t}) - F^{\ast} = \mathcal{O} \left( \frac{\log^2(1/\delta)}{ n^2}  \right),
\end{align*}
where $c$ is an absolute constant. 
\end{theorem}

Under additional curvature assumptions on both the empirical and population risks, SGD admits the following last-iterate guarantee.
\begin{theorem} \label{theo7}
Suppose Assumptions \ref{assu4}, \ref{assu7}, \ref{assu8} and \ref{assu5} hold, suppose the empirical risk $F_S$ satisfies Assumption \ref{assu10} with parameter $2\mu_S$, and suppose the population risk $F$ satisfies Assumption \ref{assu10} with parameter $2\mu$.
Let $\{ \mathbf{w}_t\}_t$ be the sequence produced by (\ref{eq1}) with $\eta_t = \frac{2}{\mu_S (t+t_0)}$ such that $t_0 \geq \max \left\{\frac{4\beta}{\mu_S},1 \right\}$ for all $t \in \mathbb{N}$. When $n \geq \frac{c\beta^2(d+ \log(\frac{16 \log(2n R +2)}{\delta}))}{\mu^2}$, 
for any $\delta >0$, with probability at least $1 - \delta$, choosing $T \asymp n^2$ yields
\begin{align*}
F(\mathbf{w}_{T+1}) - F^{\ast} 
= \mathcal{O} \left(\frac{\log n \log^3(\frac{1}{\delta})}{n^2}  + \frac{ F^{\ast} \log(\frac{1}{\delta})}{ n} \right);
\end{align*}
if further assuming $F^{\ast}  = \mathcal{O}(1/n)$, we have 
\begin{align*}
F(\mathbf{w}_{T+1}) - F^{\ast}  = \mathcal{O} \left( \frac{\log n \log^3(1/\delta)}{n^2}  \right),
\end{align*}
where $c$ is an absolute constant. 
\end{theorem} 

\begin{remark}\rm{}\label{remark17}
Compared with Theorem~\ref{theo67}, when both $F$ and $F_S$ satisfy the PL condition, Theorem~\ref{theo7} not only strengthens the guarantee from an averaged-iterate bound to a last-iterate bound, but also improves the required iteration complexity from $T\asymp n^4$ to $T\asymp n^2$. 
\end{remark}

\paragraph{Comparison with prior work.}
Fast generalization guarantees for SGD have been studied extensively. Classical $\mathcal{O}(1/n)$-type rates are typically derived under strong convexity \cite{kakade2008on,kar2013on,lei2018stochastic,klochkov2021stability}. The work \cite{lei2021learning} relaxes strong convexity to the PL condition, but still obtains only an $\mathcal{O}(1/n)$-type rate for SGD: $F(\mathbf{w}_{T+1}) - F^{\ast}  = \mathcal{O}\big(\frac{(d+ \log (1/\delta)) \log^2 n \log^2 (1/\delta)}{n}\big)$. In contrast, Theorem~\ref{theo7} yields an $\mathcal{O}(1/n^2)$-type bound under PL-type curvature, without requiring convexity. Several more recent works also obtain $\mathcal{O}(1/n^2)$-type rates, but under different assumptions or in different problem settings. The works \cite{li2022high,lihigh} derive such $\mathcal{O}(1/n^2)$-type rates for stochastic gradient descent and Polyak's momentum, respectively, under sub-Weibull gradient noise assumptions, whereas the results in Section \ref{section3} require only a bounded-variance-type condition, which is weaker in our framework (see Remark~\ref{remark6}). The recent stability-based analysis of \cite{zhu2025stability} also attains an $\mathcal{O}(1/n^2)$-type rate, but requires a uniformly bounded gradient condition (Assumption \ref{assum666}); by contrast, our analysis does not rely on bounded gradients. Finally, \cite{li2023learning} and \cite{zhu2024towards} establish $\mathcal{O}(1/n^2)$-type rates for the specialized settings of pairwise learning and minimax problems, whereas our result is for the standard (single-sample/parameter) risk minimization setting.

\paragraph{Theoretical insight.} The proof uses the PL condition to convert the excess risk into the population gradient term $\| \nabla F(\mathbf{w}_{T+1}) \|^2$, which is then bounded by the sum of
\begin{align}\label{afirstcom}
2\| \nabla F_S(\mathbf{w}_{T+1}) \|^2 \quad \textup{and} \quad 2\| \nabla F(\mathbf{w}_{T+1}) - \nabla F_S(\mathbf{w}_{T+1}) \|^2.
\end{align}
The first term is an optimization error, since it is related to how well the optimization algorithm minimizes the empirical risk $F_S$; the second term is a generalization error, since it is related to the gap between the population gradient and its empirical counterpart based on random samples. These two terms are bounded by
\begin{align*}
\mathcal{O} \Big(\frac{\log T \log^3(1/\delta)}{T} \Big)
\quad \textup{and} \quad
\mathcal{O} \Big(\frac{\log T \log^3(1/\delta)}{T} +\frac{\log^2(1/\delta)}{n^2} \Big),
\end{align*}
respectively. One can see that the optimization error decreases as the iteration number increases and, interestingly, the generalization error also decreases as the iteration number increases. This contrasts with Theorem 5 in \cite{lei2021learning}, where the generalization term increases along training. Thus, one should select an appropriate iteration number $T$ for early stopping to balance the iteration complexity (optimization error) and generalization so as to achieve a good learning rate, which is the intuition behind resisting overfitting. By contrast, Theorem~\ref{theo7} reveals that, if the population and empirical risks satisfy certain curvature conditions such as the PL condition and a suitable sample complexity condition holds, generalization continues to improve as training accuracy increases, which means that the overfitting phenomenon would not happen and that the same early-stopping tradeoff does not arise. This perspective is also broadly consistent with empirical observations in overparameterized neural networks, where PL-like geometry may hold near global minima \cite{zhou2019sgd,kleinberg2018alternative}.

\subsection{Nesterov’s Accelerated Gradient}\label{section4}
NAG is also widely used in modern machine learning due to its strong empirical performance \cite{nesterov1983method,sutskever2013importance}. At iteration $t$,
with stepsize $\eta_t \ge 0$ and momentum $\gamma\in[0,1)$, NAG samples an index $j_t$ i.i.d. uniformly from $\{1,\ldots,n\}$ and performs the following look-ahead update 
\begin{align}\label{eq:nag}\nonumber
\mathbf{y}_t \;=\; \mathbf{w}_t + \gamma \mathbf{m}_t,\qquad
\mathbf{g}_t \;=\; \nabla f(\mathbf{y}_t; z_{j_t}),\qquad \\
\mathbf{m}_{t+1} \;=\; \gamma \mathbf{m}_t - \eta_t\, \mathbf{g}_t,\qquad
\mathbf{w}_{t+1} \;=\; \mathbf{w}_t + \mathbf{m}_{t+1}.
\end{align}
Here $\nabla f(\mathbf{y}_t; z_{j_t})$ denotes the gradient of $f$ w.r.t. its first argument at the look-ahead point $\mathbf{y}_t$. As in the SGD case, the randomness arises from both the sample draw $S$ and the index sequence $\{j_t\}_t$. For simplicity, we also assume $\mathbf{w}_1=\mathbf{m}_1=0$ without loss of generality, since all results extend to arbitrary initialization.

Compared with SGD, the generalization analysis of NAG is more delicate because of the coupling among the iterate $\mathbf{w}_t$, the look-ahead point $\mathbf{y}_t$, and the momentum variable $\mathbf{m}_t$. Accordingly, we present our results in two stages. We first control population gradient under weak curvature, and then upgrade it to excess risk bounds by invoking the PL condition.

We begin with gradient control in the general nonconvex regime by establishing a high-probability bound for the averaged population gradient norm
 $(\sum_{t = 1}^T \eta_t )^{-1} \sum_{t = 1}^T \eta_t \| \nabla F(\mathbf{w}_t) \|^2$. 
\begin{theorem} \label{theo6nag}
Suppose Assumptions \ref{assu4}, \ref{assu7}, \ref{assu8} and \ref{assu5} hold. Let $\{ \mathbf{w}_t\}_t$ be the sequence produced by (\ref{eq:nag}) with $\eta_t = \eta_1 t^{-1/2} $ such that $\eta_1 \leq   \min\!\left\{\frac{1-\gamma}{2\sqrt{2}\,\gamma \beta},\ 
\frac{(1-\gamma)^2}{32\,C_m(\gamma, \beta)}\right\}$. Then, for any $\delta >0$, with probability $1 - \delta$, choosing $T \asymp nd^{-1} $ yields
\begin{align*}
& \Big(\sum_{t = 1}^T \eta_t \Big)^{-1} \sum_{t = 1}^T \eta_t \| \nabla F(\mathbf{w}_t) \|^2  =
             \mathcal{O} \left( \sqrt{\frac{d}{n}} \log(\frac{n}{d\delta})  \log^3(\frac{1}{\delta}) \right),
\end{align*}
where $C_m(\gamma,\beta):=\frac{1}{1-\gamma}(\beta\gamma+ \frac{\beta\gamma(1-\gamma)}{4\sqrt{2}})+\frac{\beta}{2}.$
\end{theorem}
\paragraph{Comparison with prior work.}
To our knowledge, only two prior works have investigated the generalization performance of Nesterov's accelerated gradient \cite{attia2021algorithmic,chen2018stability}. \cite{attia2021algorithmic} derives a uniform stability bound for deterministic (full-gradient) NAG under general convexity, whereas \cite{chen2018stability} develops guarantees for convex quadratic losses. The result in \cite{chen2018stability} scales as $\mathcal{O}(1/\sqrt{n})$, which is a slow rate. Overall, generalization analyses for NAG remain notably scarce compared to the extensive literature on SGD, particularly in stochastic, nonconvex settings. Although \cite{li2023high,lihigh} also study the generalization performance of momentum-based methods in stochastic, nonconvex settings, their focus is on Polyak-type momentum, with clipping in \cite{li2023high} and without clipping in \cite{lihigh}, rather than Nesterov's accelerated gradient.  Motivated by these gaps, Section \ref{section4} provides generalization analysis of stochastic NAG in nonconvex regimes. 
Relative to the SGD analysis in \cite{li2022high}, the SGD part of this manuscript is primarily a refinement, whereas the main new technical ingredient of the present manuscript is the analysis of NAG.

\paragraph{Theoretical insight.}
Theorem~\ref{theo6nag} can also be interpreted through a decomposition into a sum of an empirical-gradient term and a uniform gradient deviation term:
\begin{align*}
2 \Big(\sum_{t = 1}^T \eta_t \Big)^{-1} \sum_{t = 1}^T \eta_t \| \nabla F_S(\mathbf{w}_t) \|^2
\quad \textup{and} \quad
2 \Big(\sum_{t = 1}^T \eta_t \Big)^{-1} \sum_{t = 1}^T \eta_t \| \nabla F(\mathbf{w}_t) - \nabla F_S(\mathbf{w}_t) \|^2.
\end{align*} 
Under the stated stepsize choice, these two terms scale as
\begin{align}\label{theteocom}
\mathcal{O}\Big(\frac{1}{T^{1/2}} \log(T/\delta)\Big)  \quad \textup{and} \quad \mathcal{O}\Big(\frac{ T^{1/2}}{n} (d + \log \frac{ 1 }{\delta})\log^2(1/\delta)  \log T \Big),
\end{align}
respectively. The first term decreases as $T$ increases, whereas the second term increases along $T$. Therefore, there is an inherent tradeoff between optimization and generalization in the general nonconvex regime, and early stopping serves as an effective regularization mechanism in this upper bound. The choice of $T$ in Theorem~\ref{theo6nag} balances these two terms and yields the stated averaged population-gradient guarantee.

We next impose the PL condition on the population risk $F$ to give an average-iterate excess risk bound $(\sum_{t = 1}^T \eta_t )^{-1} \sum_{t = 1}^T \eta_t F(\mathbf{w}_{t}) - F^{\ast}$. 
\begin{theorem} \label{theo67nag}
Suppose Assumptions \ref{assu4}, \ref{assu7}, \ref{assu8} and \ref{assu5} hold,  and suppose the population risk $F$ satisfies Assumption \ref{assu10} with parameter $\mu$. Let $\{ \mathbf{w}_t\}_t$ be the sequence produced by (\ref{eq:nag}) with $\eta_t = \eta_1 t^{- 1/2}$ such that $\eta_1 \leq  \min\!\left\{\frac{1-\gamma}{2\sqrt{2}\,\gamma \beta},\ 
\frac{(1-\gamma)^2}{32\,C_m(\gamma,\beta)}\right\}$. 
 When $n \geq \frac{c\beta^2(d+ \log(\frac{16 \log(2n R +2)}{\delta}))}{\mu^2}$, for any $\delta >0$, with probability at least $1 - \delta$, choosing $T \asymp  n^4$ yields
\begin{align*}
\Big(\sum_{t = 1}^T \eta_t \Big)^{-1} \sum_{t = 1}^T \eta_t F(\mathbf{w}_{t}) - F^{\ast}  =  \mathcal{O} \left(\frac{\log^2(\frac{1}{\delta})}{n^2} + \frac{ F^{\ast}  \log(\frac{1}{\delta})}{n} \right);
\end{align*}
if further assuming $F^{\ast}  = \mathcal{O}(1/n)$, we have 
\begin{align*}
\left(\sum_{t = 1}^T \eta_t \right)^{-1} \sum_{t = 1}^T \eta_t F(\mathbf{w}_{t}) - F^{\ast} = \mathcal{O} \left( \frac{\log^2(1/\delta)}{ n^2}  \right),
\end{align*}
where $c$ is an absolute constant, and where
$C_m(\gamma,\beta):=\frac{1}{1-\gamma}(\beta\gamma+ \frac{\beta\gamma(1-\gamma)}{4\sqrt{2}})+\frac{\beta}{2}.$
\end{theorem}
 
Finally, under additional PL curvature on the empirical risk $F_S$ and stronger regularity assumptions, we derive a last-iterate excess risk bound $F(\mathbf{w}_{T+1})-F^\ast$.
\begin{theorem} \label{theo7nag}
Suppose Assumptions \ref{assum666}, \ref{assu4}, \ref{assu8} and \ref{assu5} hold, suppose the empirical risk $F_S$ satisfies Assumption \ref{assu10} with parameter $2\mu_S$, and suppose the population risk $F$ satisfies Assumption \ref{assu10} with parameter $2\mu$.
Let $\{ \mathbf{w}_t\}_t$ be the sequence produced by (\ref{eq:nag}) with $\eta_t = \frac{2}{\mu_S (t+t_0)}$ such that $t_0 >0$ for all $t \in \mathbb{N}$. When $n \geq \frac{c\beta^2(d+ \log(\frac{16 \log(2n R +2)}{\delta}))}{\mu^2}$, 
for any $\delta >0$, with probability at least $1 - \delta$, choosing $T \asymp n^2$ yields
\begin{align*}
F(\mathbf{w}_{T+1}) - F^{\ast} 
= \mathcal{O} \left(\frac{ \log^2(\frac{1}{\delta})}{n^2}  + \frac{ F^{\ast} \log(\frac{1}{\delta})}{ n} \right);
\end{align*}
if further assuming $F^{\ast}  = \mathcal{O}(1/n)$, we have 
\begin{align*}
F(\mathbf{w}_{T+1}) - F^{\ast}  = \mathcal{O} \left( \frac{ \log^2(1/\delta)}{n^2}  \right),
\end{align*}
where $c$ is an absolute constant. 
\end{theorem}

\begin{remark}\rm Theorem \ref{theo7nag} shows that the $\mathcal O(1/n^2)$-type behavior extends to the last iterate with a tailored $1/t$ stepsize schedule and PL on both $F$ and $F_S$. Compared with the averaged-risk result, the last-iterate guarantee not only strengthens the output notion but also improves the iteration complexity from $T \asymp n^4$ to $T \asymp n^2$. 
\end{remark}
\paragraph{Acceleration versus generalization.}
A natural question is whether NAG improves upon SGD not only in optimization, but also in generalization. Our results suggest that this is not necessarily the case. Although NAG is well known to accelerate optimization in deterministic convex and strongly convex settings, such acceleration does not automatically translate into a sharper population excess-risk bound (the generalization), particularly in stochastic and nonconvex regimes. In our framework, NAG matches rather than improves the order of the SGD excess-risk bound.
 \paragraph{Theoretical insight.}
Theorem \ref{theo7nag} also decomposes the excess risk into an empirical-gradient (optimization) term and a uniform gradient deviation (generalization) term, as in (\ref{afirstcom}), which is scaled as 
\begin{align}\label{thethreecom}
\mathcal{O} \Big(\frac{  \log (1/\delta)}{T} \Big) \quad \textup{and} \quad  \mathcal{O} \Big(\frac{ \log (1/\delta)}{T} +\frac{\log^2(1/\delta)}{n^2} \Big),
\end{align}
respectively. Compared to (\ref{theteocom}) of Theorem \ref{theo6nag},
the two terms in (\ref{thethreecom}) decrease as $T$ grows. This shows that, like SGD, NAG enjoys the property that under suitable curvature, improving training accuracy continually improves generalization, softening the classical argument that a model should balance under-fitting and over-fitting. 
 \paragraph{Technical ingredients.}
The central insight is an optimization-driven analysis that links generalization directly to the optimization trajectory. The analysis combines two ingredients. First, instead of the more common function-value uniform convergence, we work with uniform convergence of gradients, which aligns more naturally with stochastic optimization. To obtain fast $\mathcal{O}(1/n^2)$-type rates, we adopt a localization argument from \cite{xu2025towards}, linking uniform convergence of gradients to optimization (see Lemma \ref{erfgefge}). This directly couples generalization with the optimization accuracy that the algorithm actually achieves. A fast optimization bound can thus admit a fast uniform convergence bound.  Second, while existing results commonly provide in-expectation optimization bounds, we establish new high probability optimization bounds (i) for SGD under non-smoothness, weaker than widely used smoothness \cite{lei2021learning}, both in general nonconvex regime  and under PL (see Section \ref{optimisgd}); (ii) for NAG under smoothness, both in general nonconvex regime  and under PL (see Section \ref{proof-of-nag}). Especially for NAG, such high-probability optimization guarantees are rare even in convex settings.

\begin{figure*}[htbp]
  \centering
  \subfigure[Breast Cancer.]{\includegraphics[width=0.24\textwidth,height=0.18\textwidth,trim=3 0 6 3, clip]{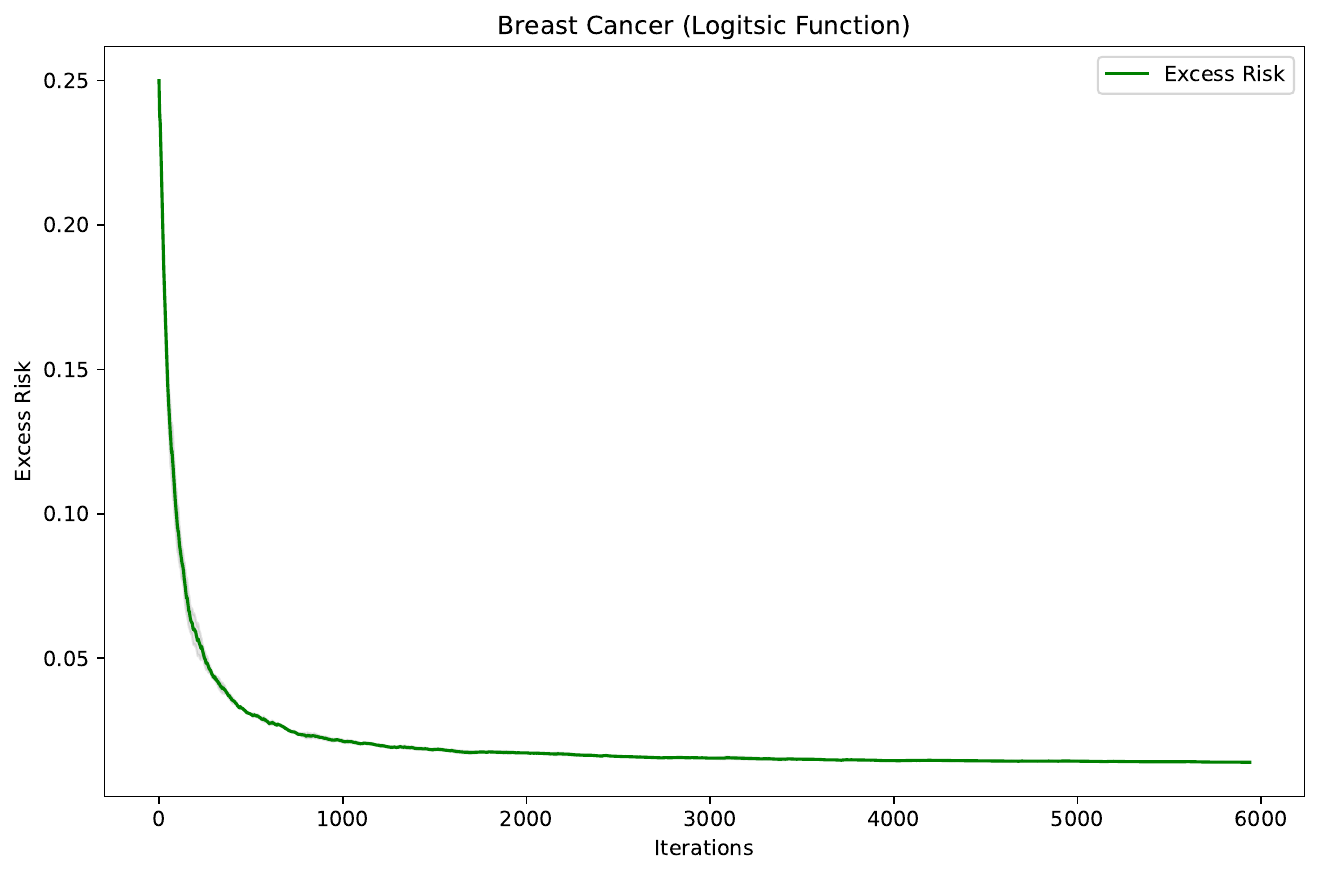}}
  \subfigure[German.]{\includegraphics[width=0.24\textwidth,height=0.18\textwidth,trim=3 0 6 3, clip]{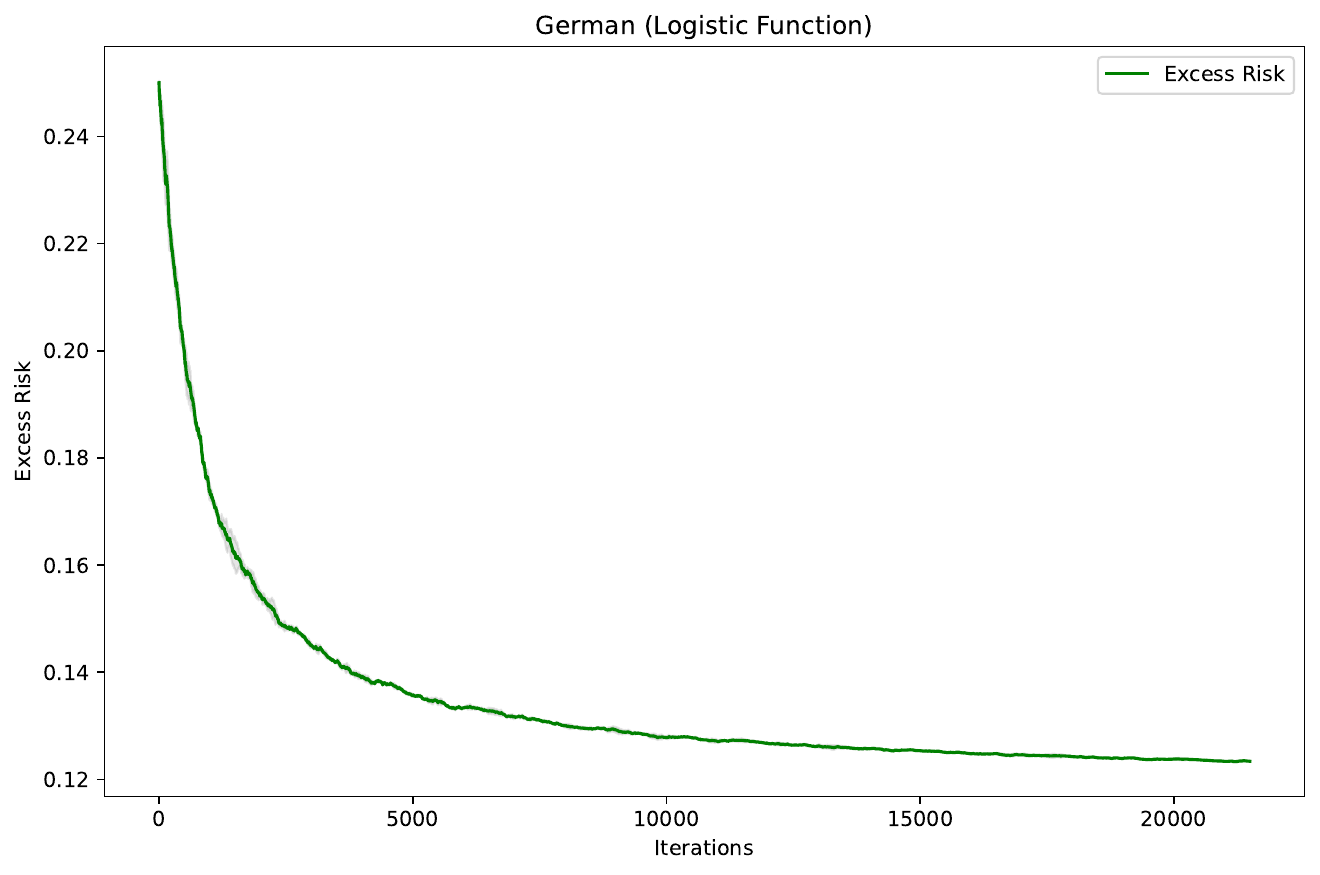}}
  \subfigure[Heart.]{\includegraphics[width=0.24\textwidth,height=0.18\textwidth,trim=3 0 6 3, clip]{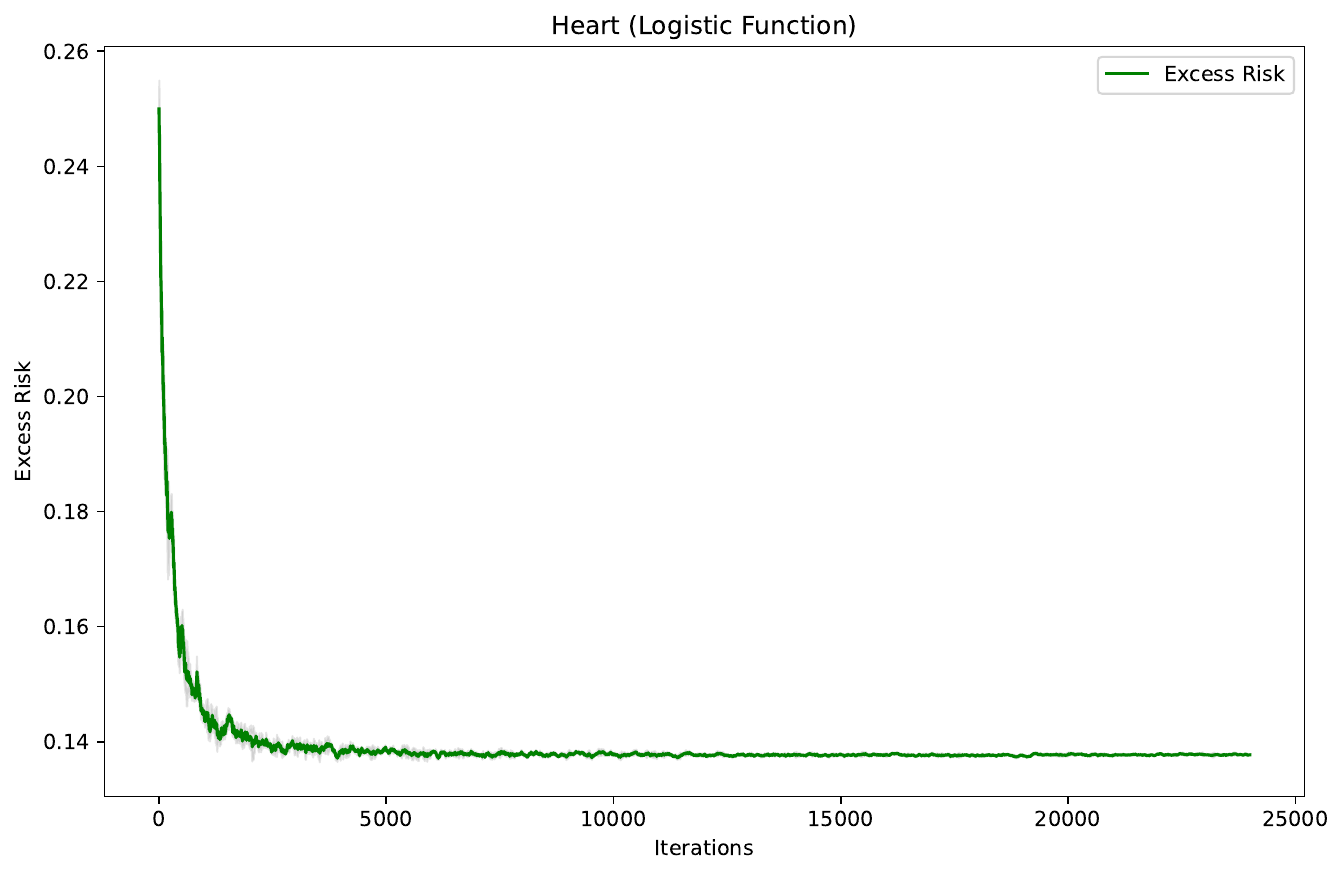}}
  \subfigure[IJCNN.]{\includegraphics[width=0.24\textwidth,height=0.18\textwidth,trim=3 0 6 3, clip]{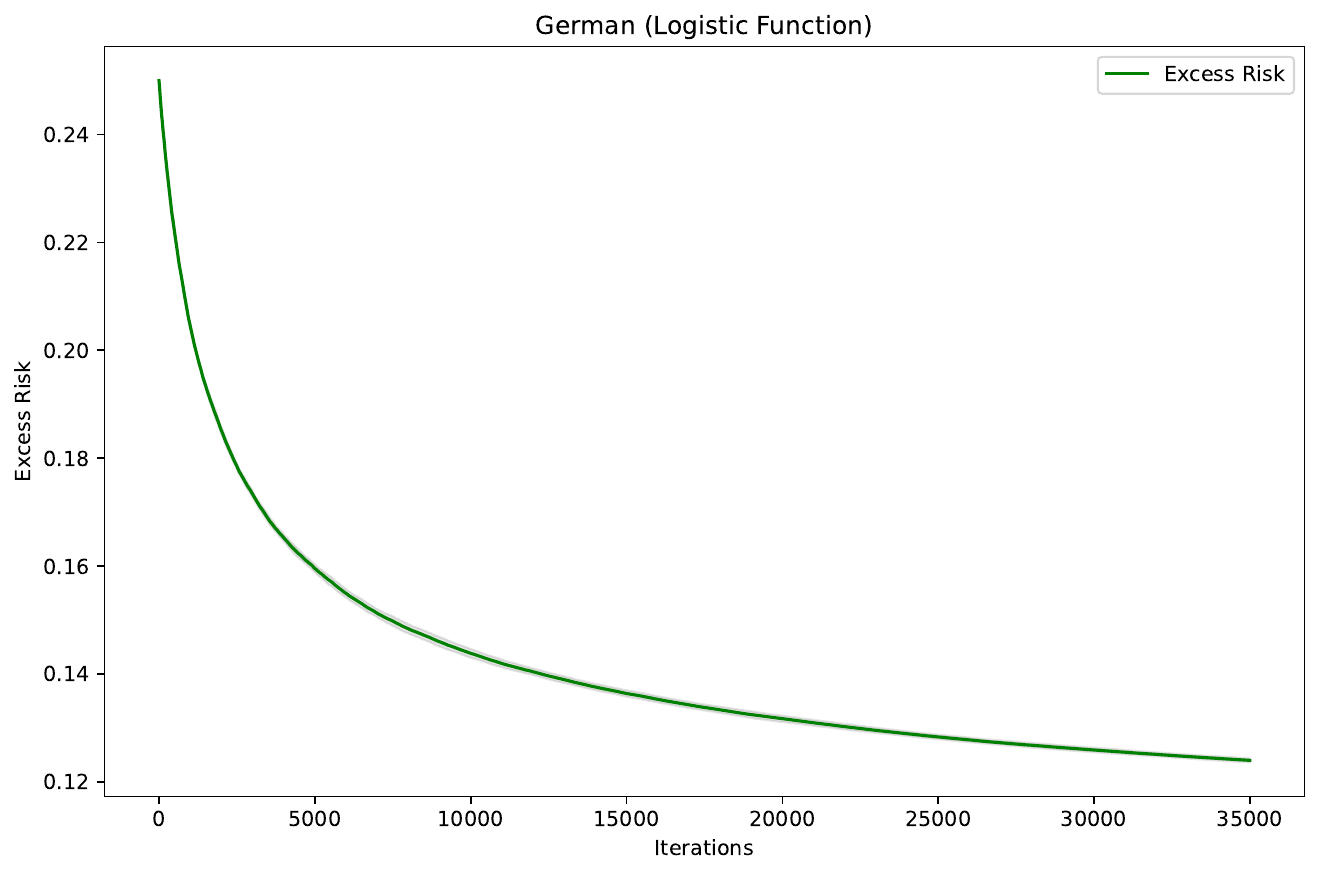}}
  \caption{The excess risk $F(\mathbf{w}) - F^{\ast}$ versus the number of iterations for the logistic link function across different datasets:  Breast-Cancer, German, Heart, and IJCNN.
  \label{fig:aera-ppjjyy}}
\end{figure*}

\begin{figure*}[htbp]
  \centering
 \subfigure[Breast Cancer.]{\includegraphics[width=0.24\textwidth,height=0.18\textwidth,trim=3 0 6 3, clip]{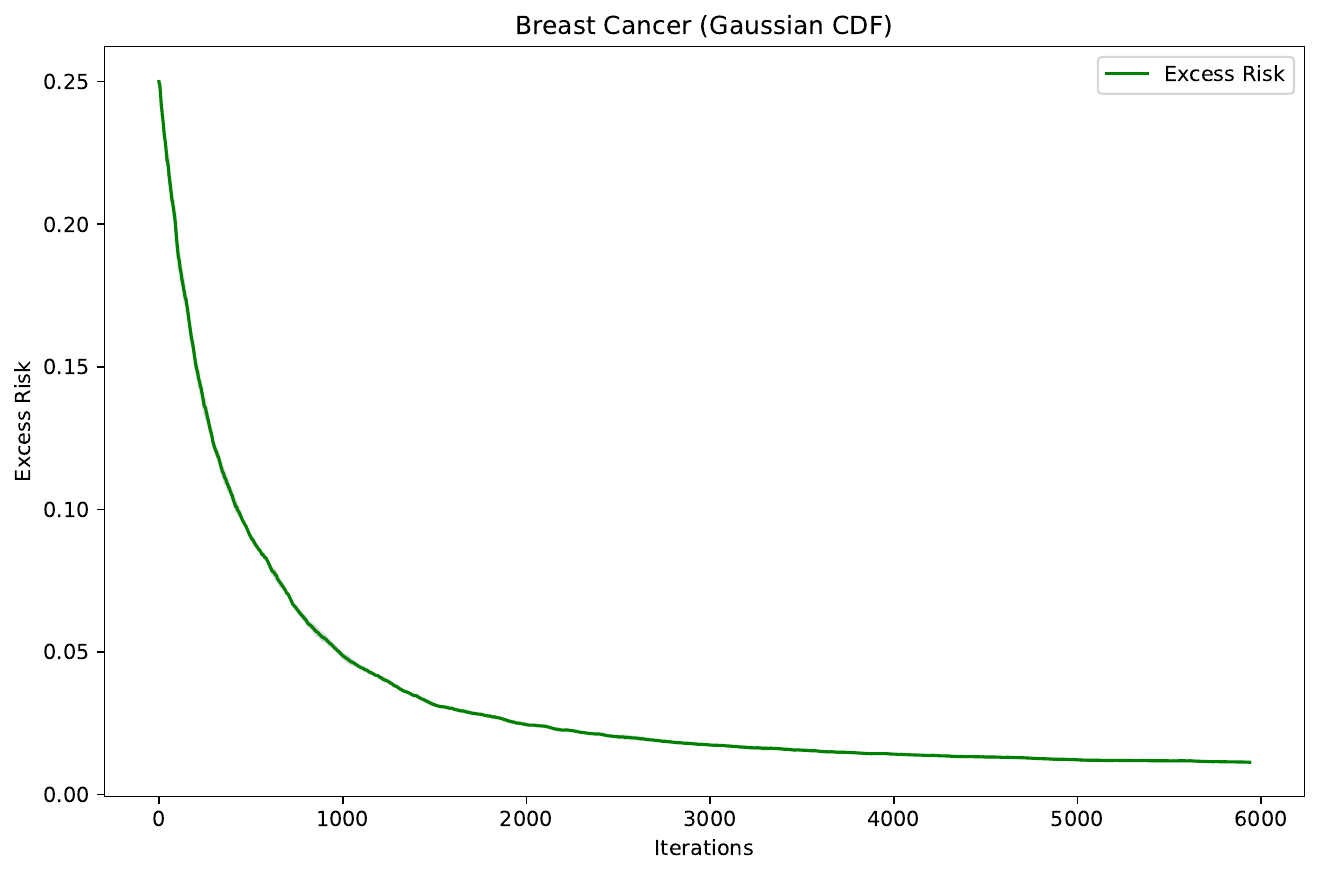}}
  \subfigure[German.]{\includegraphics[width=0.24\textwidth,height=0.18\textwidth,trim=3 0 6 3, clip]{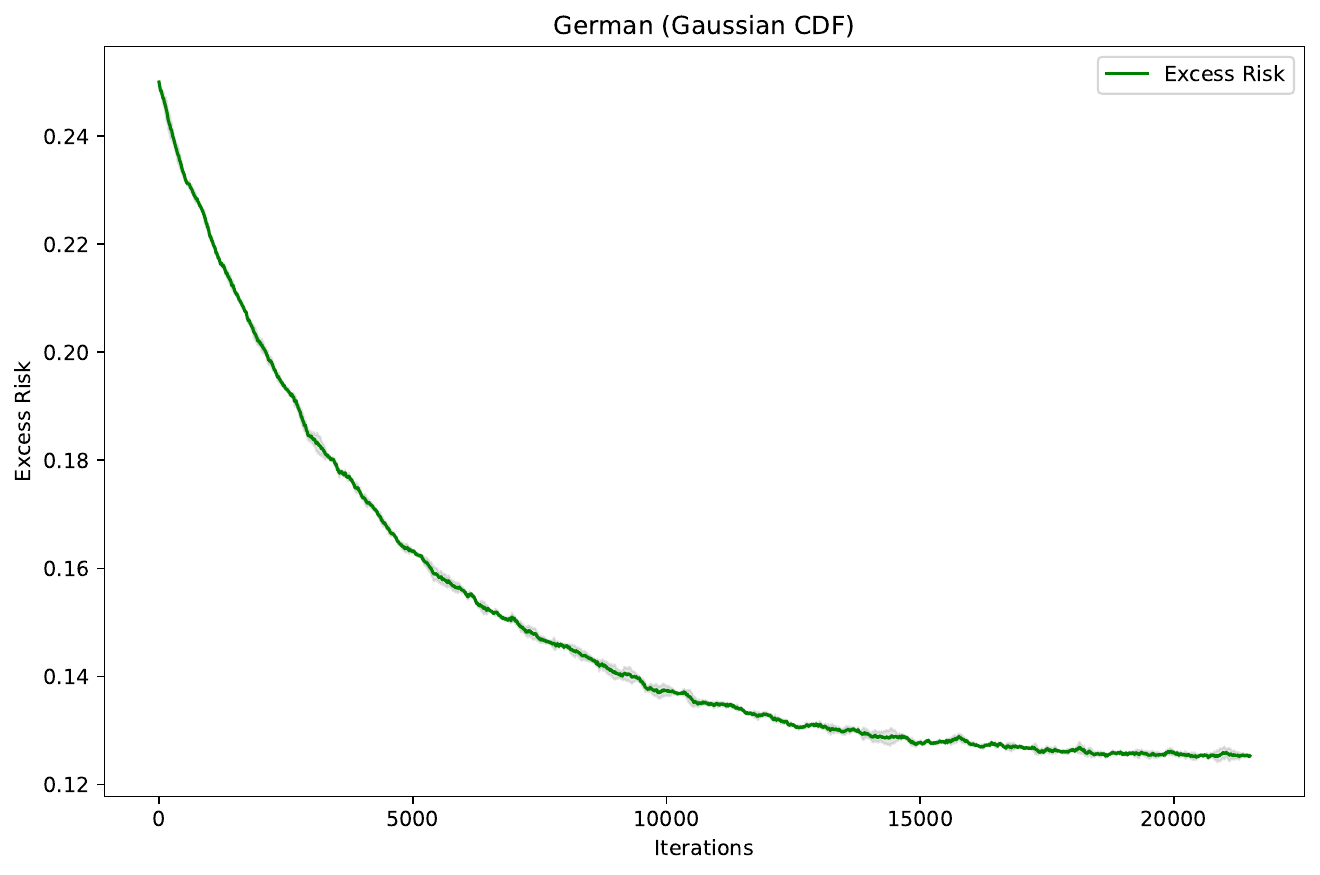}}
  \subfigure[Heart.]{\includegraphics[width=0.24\textwidth,height=0.18\textwidth,trim=3 0 6 3, clip]{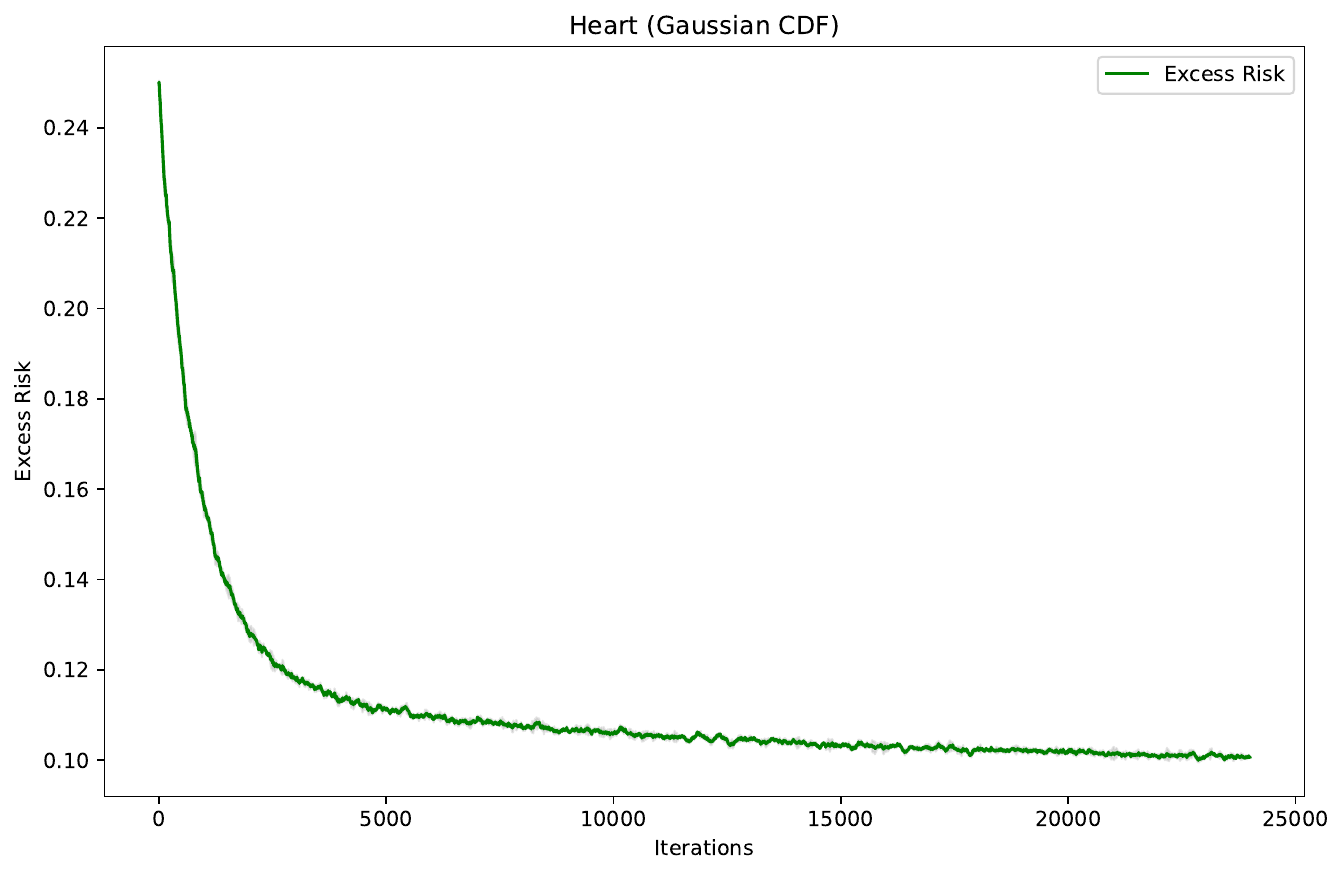}}
    \subfigure[IJCNN.]{\includegraphics[width=0.24\textwidth,height=0.18\textwidth,trim=3 0 6 3, clip]{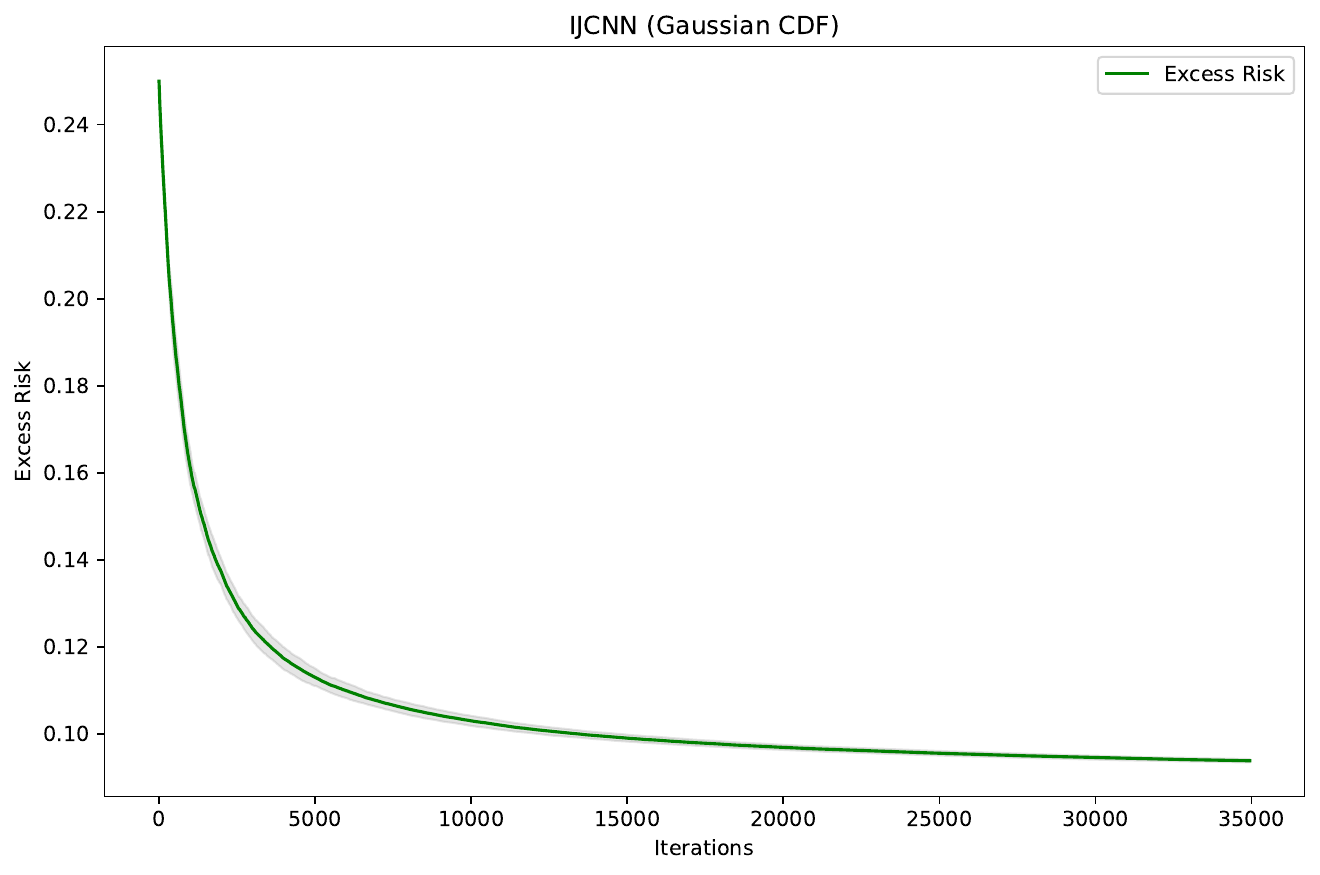}}
  \caption{The excess risk $F(\mathbf{w}) - F^{\ast}$ versus the number of iterations for the probit link function across different datasets: Breast-Cancer, German, Heart, and IJCNN.
  \label{fig:aera-ppjjyyy}}
\end{figure*}

\begin{figure}[htbp]
  \centering
   \hspace*{-0.2cm}
  \subfigure{\includegraphics[width=0.23\textwidth,height=0.17\textwidth,trim=3 0 6 3, clip]{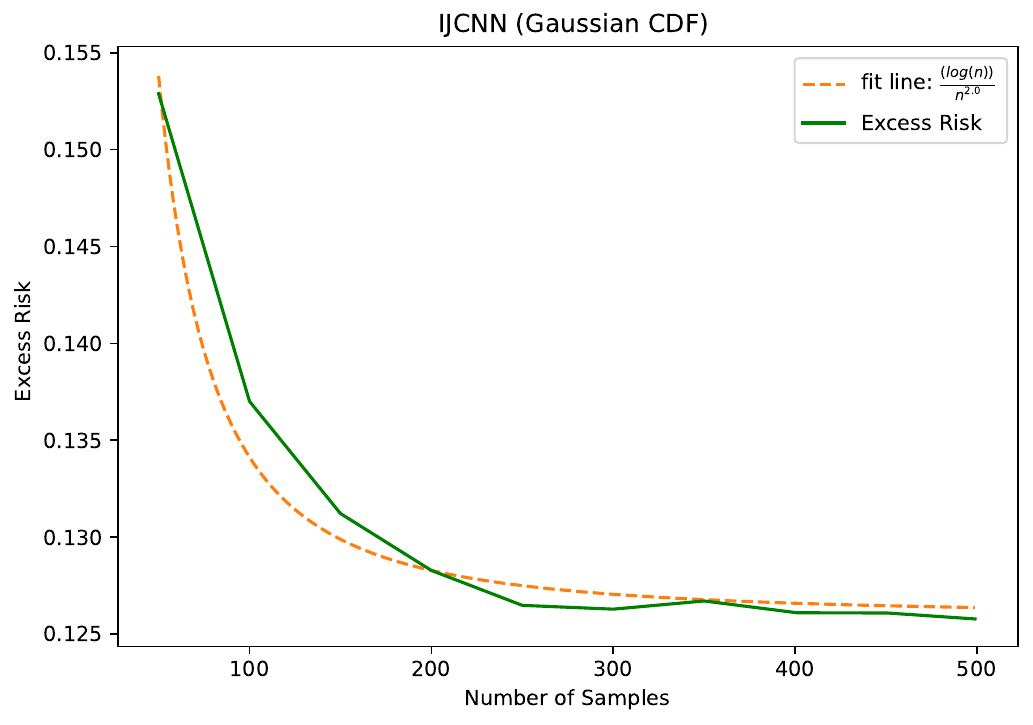}} 
  \hspace{0.1\textwidth}
  \subfigure{\includegraphics[width=0.23\textwidth,height=0.17\textwidth,trim=3 0 6 3, clip]{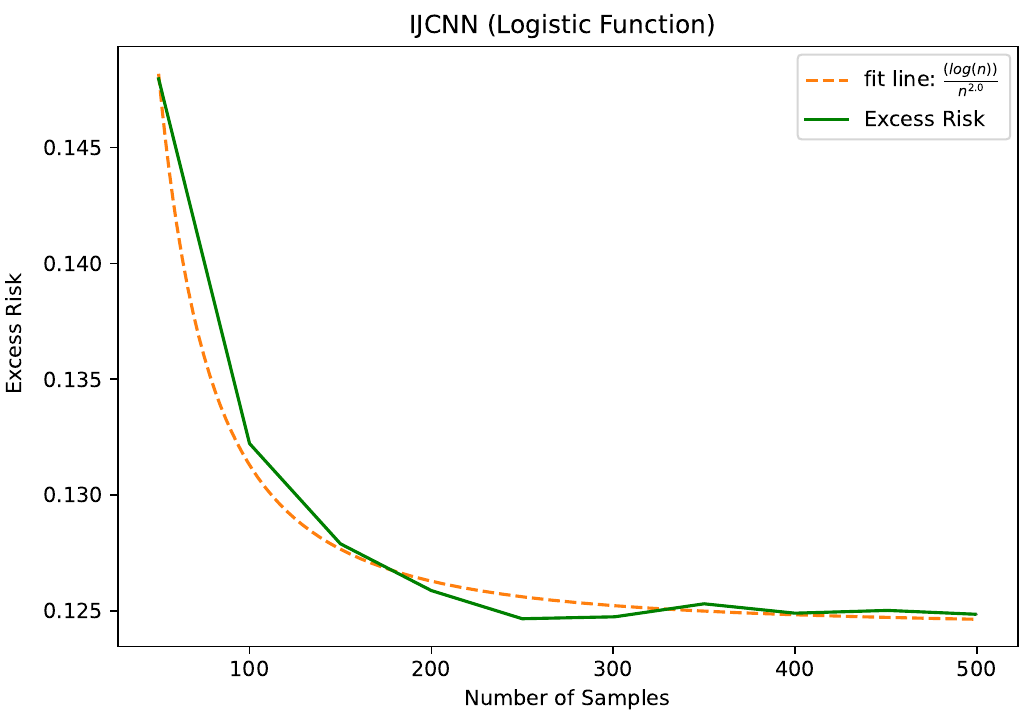}}
  \caption{The excess risk $F(\mathbf{w}) - F^{\ast}$ versus the number of samples for the probit link function (left) and the logistic link function (right) on the IJCNN dataset.
  \label{fig:aera-ppjjjyyy}}
\end{figure}

\begin{figure*}[t]
\centering

\begin{minipage}[t]{0.48\textwidth}
    \centering
    \includegraphics[width=0.48\linewidth,height=0.35\linewidth,trim=3 0 6 3,clip]{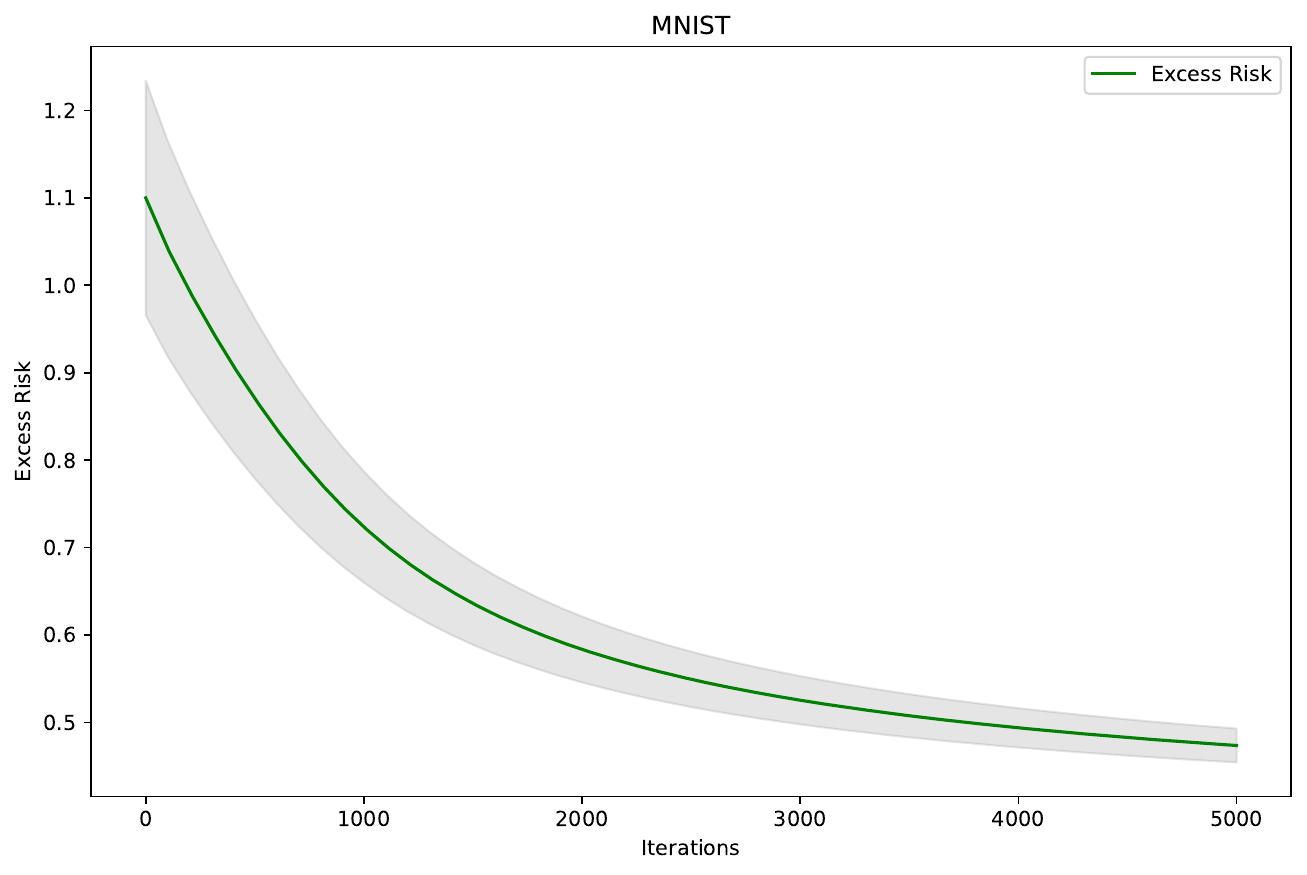}
    \hfill
    \includegraphics[width=0.48\linewidth,height=0.35\linewidth,trim=3 0 6 3,clip]{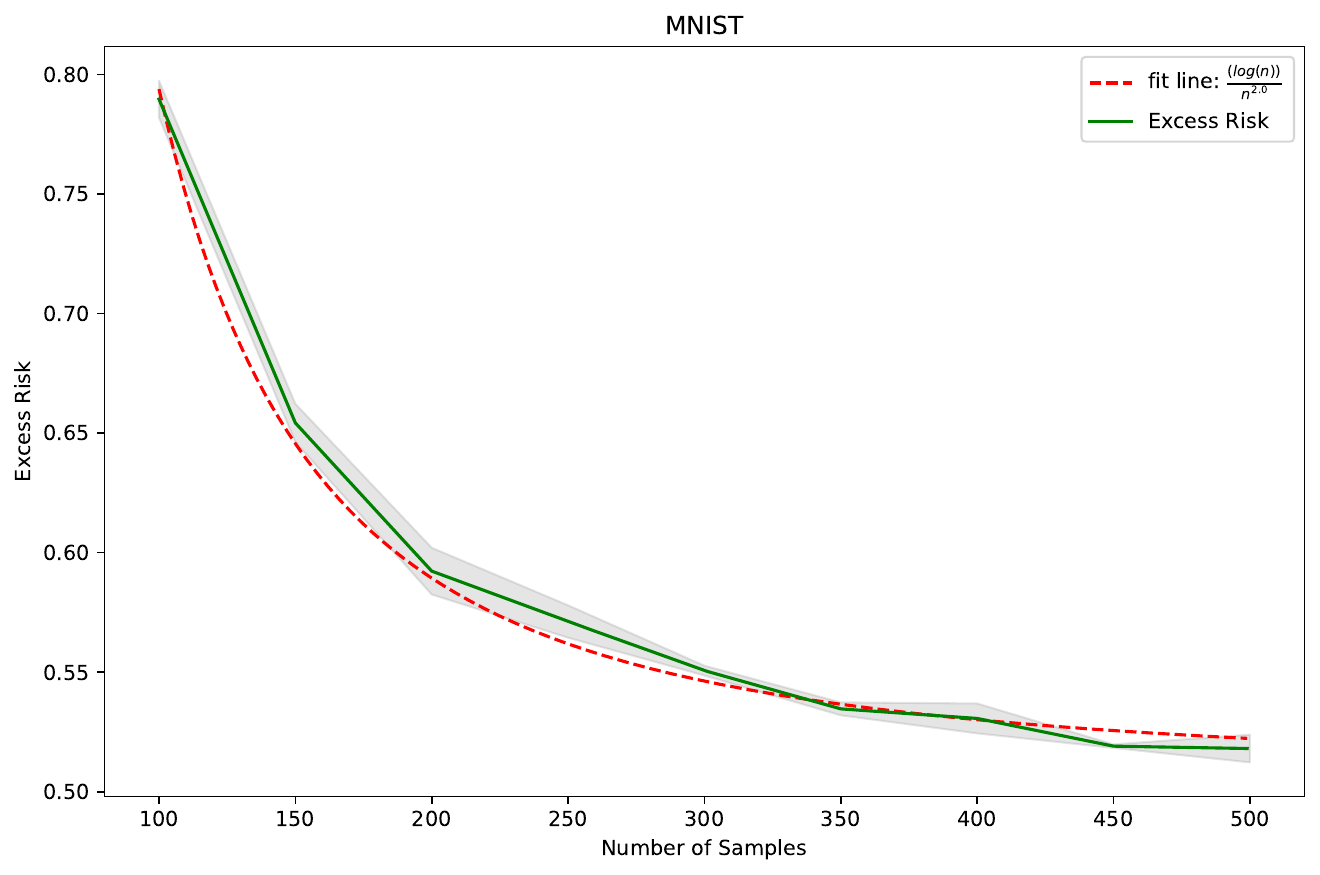}
    
    \captionof{figure}{The excess risk $F(\mathbf{w})-F^\ast$ versus the number of iterations (left) and the number of samples (right) on the MNIST dataset for image classification.}
    \label{fig:mnist-two}
\end{minipage}
\hfill
\begin{minipage}[t]{0.48\textwidth}
    \centering
    \includegraphics[width=0.48\linewidth,height=0.35\linewidth,trim=3 0 6 3,clip]{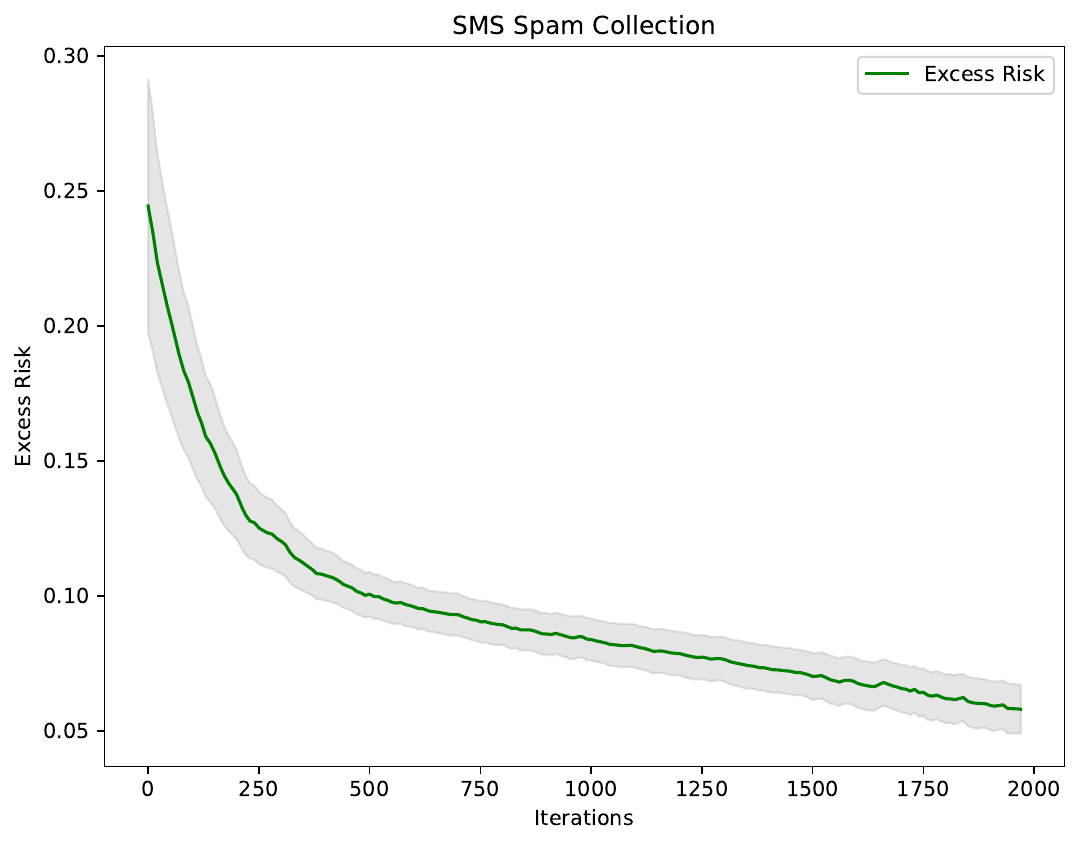}
    \hfill
    \includegraphics[width=0.48\linewidth,height=0.35\linewidth,trim=3 0 6 3,clip]{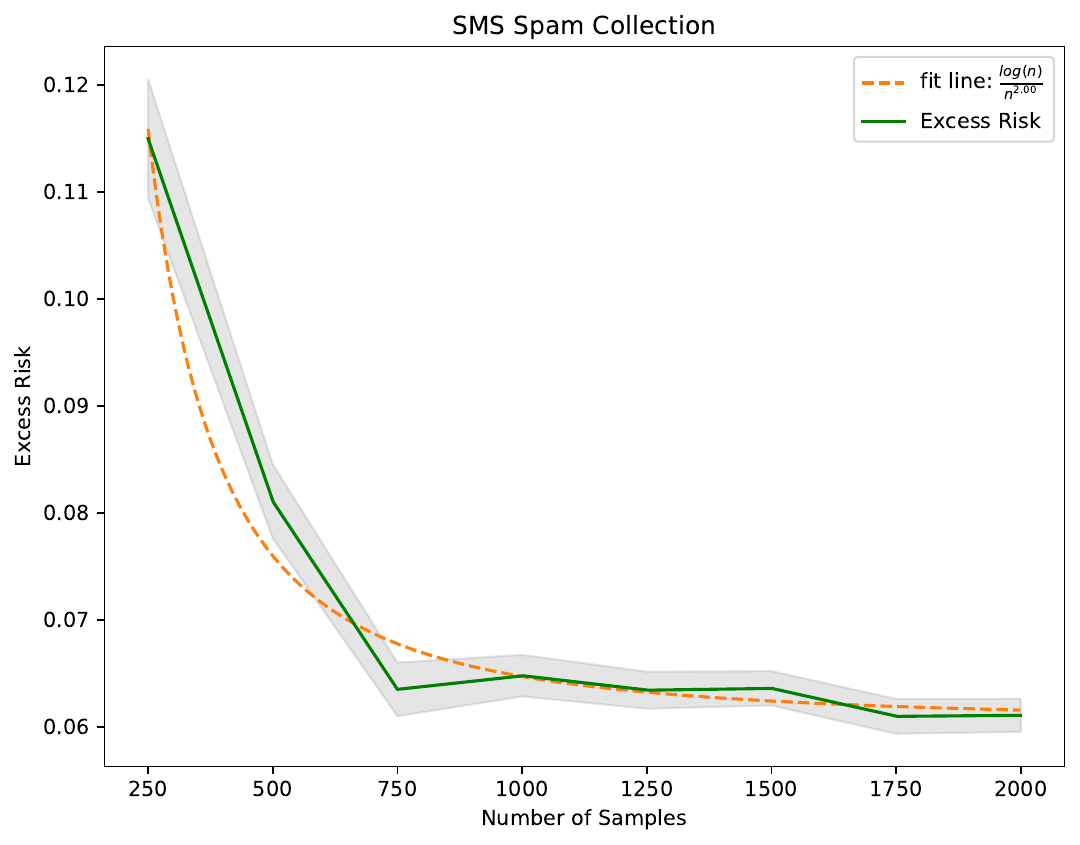}
    
    \captionof{figure}{The excess risk $F(\mathbf{w})-F^\ast$ versus the number of iterations (left) and the number of samples (right) on the SMS Spam Collection dataset for spam detection.}
    \label{fig:sms-two}
\end{minipage}

\end{figure*}

\section{Experiments}\label{sectionexp}
To empirically illustrate our theory, we present experimental results for SGD. Since the main experimental goal is to examine the qualitative behavior predicted by our theory, we use SGD as a representative stochastic optimization algorithm.

\subsection{Numerical Experiments}\label{srctionao}
Section \ref{srctionao} presents numerical experiments to validate our theory. Let $F_S(\mathbf{x})$ and $F_{S'}(\mathbf{x})$ be the risk built on the training dataset $S$ and the test dataset $S'$ respectively. Thus, $F_{S'}(\mathbf{x}) = \frac{1}{|S'|}\sum_{z \in S'} f(\mathbf{x};z)$, where $|S'|$ denotes the cardinality of the set $S'$. We use $F_{S'}(\mathbf{x})$ as a good approximation of the population risk $F$. In this section, we first consider a generalized linear model $\ell(\langle \mathbf{w}, x  \rangle)$ for binary classification where $\ell$ is the logistic link function $\ell(s) = (1+e^{-s})^{-1}$ and a square loss that takes the form $f(\mathbf{w},z) = (\ell(\langle \mathbf{w}, x  \rangle) - y)^2$. Then, we study a generalized linear model $\ell(\langle \mathbf{x}, x  \rangle)$ where $\ell$ is the probit link function $\ell(s) = \Phi(s)$ and where $\Phi$ is the Gaussian cumulative distribution function (CDF). It was shown in Theorem 3 in \cite{foster2018uniform} that the two corresponding objective functions $f$ satisfy the PL condition. We consider four datasets available from the LIBSVM dataset: Breast-Cancer, German, Heart, and IJCNN \cite{chang2011libsvm}. For  Breast-Cancer, German and Heart, we take $80$ percents as the training dataset and leave the remaining $20$ percents as the test dataset.  For IJCNN, we use its provided training dataset and test dataset \cite{chang2011libsvm}.

As discussed below Remark \ref{remark17}, Theorem \ref{theo7} implies that the generalization would improve as we increase the training accuracy and the overfitting phenomenon would not happen under the PL condition.  We aim to verify this theoretical finding. Towards this goal, we apply SGD to the training set and get the model sequence $\{\mathbf{w}_{t} \}$, and then compute the excess risk $F(\mathbf{w}) - F^{\ast}$ of $\{\mathbf{w}_{t} \}$ on the test dataset. Note that $F^{\ast}$ can be seen as a constant, thus we only need to compute $F(\mathbf{w})$. We set the stepsize as $\eta_t = 10 (t+1000)^{-1}$, repeat experiments $100$ times, and report the average of experimental results. The behavior of the excess risk versus the number of iterations is presented in Figure \ref{fig:aera-ppjjyy} and Figure \ref{fig:aera-ppjjyyy} for the logistic link function and the probit link function respectively. It is clear that the excess risk continues to decrease along the learning process. This result is consistent with the conclusion of Theorem \ref{theo7}, verifying our theoretical findings. 

Additionally, we study how the excess risk $F(\mathbf{w}) - F^{\ast}$ would behave along the number of samples. We apply SGD with the above experimental setup to evaluate the excess risk $F(\mathbf{w}) - F^{\ast}$ of $\{\mathbf{w}_{t} \}$. Following Theorem \ref{theo7}, we iterate SGD with $T = n^2$ times. The behavior of the excess risk versus the number of training samples for the probit link function and the logistic link function on the IJCNN dataset is reported in Figure \ref{fig:aera-ppjjjyyy}. From this figure, one can see that the rate of the excess risk matches the predicted rate $\log  n/n^2$ of Theorem \ref{theo7}; see the fitting line in Figure \ref{fig:aera-ppjjjyyy}.

\subsection{Applications}\label{srctionao1} 
Section \ref{srctionao1} evaluates our theory by applications to computer vision and natural language processing: image classification and spam classification respectively. Firstly, we consider Feedforward neural networks for image classification on the dataset MNIST \cite{lecun1998gradient}. We adopt a simple model consisting of one hidden layer with $128$ neurons and one output layer with $10$ neurons. We use Relu as the activate function and select the usual cross entropy loss. Additionally, we avoid the use of regularization such as dropout to make the experiments more interpretable. We apply SGD to the training set and then compute the excess risk $F(\mathbf{w}) - F^{\ast}$ of the produced model sequence $\{\mathbf{w}_{t} \}$ on the test dataset. And we follow the experimental setup in Section \ref{srctionao}.  The behavior of the excess risk versus the number of iterations and the number of training samples is presented in Figure \ref{fig:mnist-two}. Similarly, in this experiment, the excess risk continues to decrease and its rate matches the predicted one $\log  n/n^2$, see fitting line in Figure \ref{fig:mnist-two}. This is also consistent with our conclusions of Theorem \ref{theo7}. 

Secondly, we examine our theory for spam classication on the SMS Spam Collection dataset. We consider the recurrent neural network, which has a considerably different connectivity construction than the Feedforward neural network. Specifically, we use the LSTM model \footnote{ https://www.kaggle.com/code/mehmetlaudatekman/lstm-text-classification-pytorch/notebook}{}. Similarly, we use SGD as the optimizer and follow the experimental setup in Section \ref{srctionao}. In this case, the behavior of the excess risk versus the number of iterations and the number of training samples is displayed in Figure \ref{fig:sms-two}. From this figure, one can see that the excess risk continues to decrease and its rate matches the predicted one $\log  n/n^2$ to a large extent, which is also consistent with our conclusions of Theorem \ref{theo7}.

\section{Conclusion}\label{section6}
This paper studies the generalization performance of two fundamental stochastic optimization methods, SGD and NAG. We establish new learning rates for both algorithms and provide theoretical insight into how optimization dynamics interact with generalization. Several directions remain for future work. First, although fast rates cannot be obtained without additional structure, it would be valuable to relax the assumptions used in this paper while retaining comparable guarantees. Second, it would be interesting to extend our analysis to other stochastic optimization methods, such as stochastic variance-reduced methods \cite{reddi2016stochastic,allen2016variance} and stochastic coordinate descent methods \cite{wright2015coordinate}.

\bibliographystyle{abbrv}
\bibliography{sample}

\appendix

\section{Fundamental Tools}
In this section, we provide some useful tools. 
\subsection{Uniform Convergence and Generalization}
We first introduce a basic lemma.
\begin{definition}
For every $\alpha > 0$, we define the $Orlicz-\alpha$ norm of a random $v$:
\begin{align*}
\| v \|_{Orlicz-\alpha} = \inf \{ K>0 : \mathbb{E} \exp((|v|/K)^{\alpha}) \leq 2\}.
\end{align*}
A random variable (or vector) $X \in \mathbb{R}^d$ is $K$-sub-exponential if $\forall \lambda \in \mathbb{R}^d$, we have
$\| \lambda^T X \|_{Orlicz-1} \leq K \| \lambda \|$.
A random variable (or vector) $X \in \mathbb{R}^d$ is $K$-sub-Gaussian if $\forall \lambda \in \mathbb{R}^d$, we have
$\| \lambda^T X \|_{Orlicz-2} \leq K \| \lambda \|$.
\end{definition}
\begin{lemma}[\cite{xu2025towards}]\label{lemma30}
For all $\mathbf{w}_1, \mathbf{w}_2 \in \mathcal{W}$, we assume that $\frac{\nabla f(\mathbf{w}_1;z) - \nabla f(\mathbf{w}_2;z)}{\| \mathbf{w}_1 - \mathbf{w}_2\|}$ is a $\gamma$-sub-exponential random vector, i.e., there exists $\gamma > 0$ such that for any unit vector $u \in B(0,1)$ and $\mathbf{w}_1, \mathbf{w}_2 \in \mathcal{W}$,
\begin{align*}
\mathbb{E}\left \{ \exp \left(\frac{|u^T (\nabla f (\mathbf{w}_1;z) - \nabla f (\mathbf{w}_2;z))|}{\gamma \|\mathbf{w}_1 - \mathbf{w}_2 \| } \right) \right \} \leq 2.
\end{align*}
Then $\forall \delta \in (0,1)$, with probability $1-\delta$, for all $\mathbf{w} \in \mathcal{W}$, there holds that
\begin{align*}
& \left \| (\nabla F (\mathbf{w} )-\nabla F_S(\mathbf{w}))  -  (\nabla F (\mathbf{w}^{\ast})-\nabla F_S(\mathbf{w}^{\ast})) \right \| \\\leq &c\gamma \max \left \{ \| \mathbf{w} - \mathbf{w}^{\ast} \|, \frac{1}{n} \right \} \Bigg ( \sqrt{\frac{d+ \log \frac{4 \log_2(2n R + 2)}{\delta}}{n}} + \frac{d+ \log\frac{4\log_2(2n R +2)}{\delta}}{n} \Bigg),
\end{align*}
where $c$ is an absolute constant.
\end{lemma}
The next lemma is important for our bounds.
\begin{lemma}\label{erfgefge}
Suppose Assumptions \ref{assu4} and \ref{assu5} hold.
For all $\mathbf{w} \in \mathcal{W}$ and any $\delta >0 $, with probability at least $1- \delta$,
\begin{align}\label{firstpart} 
\left\| \nabla F (\mathbf{w} )-\nabla F_S(\mathbf{w}) \right\| \leq c' \beta \max \left \{ \| \mathbf{w} - \mathbf{w}^{\ast} \|, \frac{1}{n} \right\} 
\eta   + \frac{B_{\ast}\log(4/\delta)}{n} + \sqrt{\frac{2 \mathbb{E} [ \| \nabla f(\mathbf{w}^{\ast};z) \|^2 ] \log(4/\delta)}{n}},
\end{align}
where $c'$ is an absolute constant and $\eta =  \sqrt{\frac{d + \log \frac{8 \log_2(2n R +2)}{\delta}}{n}}  +\frac{d + \log \frac{8 \log_2(2n R+2)}{\delta}}{n} $.

Further, if population risk $F$ satisfies Assumption \ref{assu10} with parameter $\mu$, then for any $\delta >0 $, when $n \geq \frac{c\beta^2(d+ \log(\frac{8 \log(2n R +2)}{\delta}))}{\mu^2}$, with probability at least $1- \delta$,
\begin{align*}
&\hphantom{{}={}}\left\| \nabla F (\mathbf{w} )- \nabla F_S(\mathbf{w}) \right\|\leq \left\| \nabla F_S(\mathbf{w}) \right\| +  \frac{\mu}{n}   + 2\frac{B_{\ast}\log(4/\delta)}{n} + 2\sqrt{\frac{2 \mathbb{E} [ \| \nabla f(\mathbf{w}^{\ast};z) \|^2 ] \log(4/\delta)}{n}},
\end{align*}
and 
\begin{align*}
&\|  \nabla F(\mathbf{w})  \| 
\leq 2\left\| \nabla F_S(\mathbf{w}) \right\| +  \frac{\mu}{n}   + 2\frac{B_{\ast}\log(4/\delta)}{n} + 2\sqrt{\frac{2 \mathbb{E} [ \| \nabla f(\mathbf{w}^{\ast};z) \|^2 ] \log(4/\delta)}{n}},
\end{align*}
where $\mathbf{w}^{\ast}$ is the minimizer of $F$ closest to $\mathbf{w}$, and where $c$ is an absolute constant.
\end{lemma}

\begin{proof}
The proof follows from \cite{xu2025towards}, but requires some changes induced by a set of different assumptions. For brevity, denote by $\eta(\delta) =  \sqrt{\frac{d+ \log \frac{4 \log_2(2n R + 2)}{\delta}}{n}} + \frac{d+ \log\frac{4\log_2(2n R +2)}{\delta}}{n} $.
According to Assumption \ref{assu4}, for any $z \in \mathcal{Z}$ and $\mathbf{w}_1, \mathbf{w}_2 \in \mathcal{W}$, there holds
\begin{align*}
\|\nabla f(\mathbf{w}_1;z) - \nabla f(\mathbf{w}_2;z) \| \leq \beta \| \mathbf{w}_1 - \mathbf{w}_2 \|.
\end{align*}
For any unit vector $u \in B(0,1)$, 
we have 
\begin{align*}
&|u^T (\nabla f(\mathbf{w}_1;z) - \nabla f(\mathbf{w}_2;z) )|  \leq \| u\| \|\nabla f(\mathbf{w}_1;z) - \nabla f(\mathbf{w}_2;z) \| 
\leq \beta \| \mathbf{w}_1 - \mathbf{w}_2 \|. 
\end{align*}
Then we have 
\begin{align*}
\frac{|u^T (\nabla f(\mathbf{w}_1;z) - \nabla f(\mathbf{w}_2;z))|}{\beta \| \mathbf{w}_1 - \mathbf{w}_2 \|} \leq 1, 
\end{align*}
which implies
\begin{align*}
\mathbb{E} \left \{\exp \left (\frac{(\log 2) |u^T (\nabla f(\mathbf{w}_1;z) - \nabla f(\mathbf{w}_2;z))|}{\beta \| \mathbf{w}_1 - \mathbf{w}_2 \|} \right ) \right \}\leq 2, 
\end{align*}
so we obtain that for all $\mathbf{w}_1, \mathbf{w}_2 \in \mathcal{W}$, $\frac{ \nabla f(\mathbf{w}_1;z) - \nabla f(\mathbf{w}_2;z) }{\| \mathbf{w}_1 - \mathbf{w}_2 \|_2} $ is a $\frac{\beta}{\log 2}$-sub-exponential random vector. By Lemma \ref{lemma30}, we know that if Assumption \ref{assu4} holds, for $\forall \delta \in (0,1)$ and all $\mathbf{w} \in \mathcal{W}$,  we have the following inequality with probability at least $1 - \delta$
\begin{align} \label{ineq 4}
&\hphantom{{}={}}\left \| (\nabla F (\mathbf{w} )-\nabla F_S(\mathbf{w}))  -  (\nabla F (\mathbf{w}^{\ast} )-\nabla F_S(\mathbf{w}^{\ast})) \right \| \leq c \frac{\beta}{\log 2} \max \left \{ \| \mathbf{w} - \mathbf{w}^{\ast} \|, \frac{1}{n} \right \} \eta(\delta),
\end{align}
which means that there exists an absolute constant $c'$ such that $\forall \delta > 0$, with probability at least $1 - \delta/2$, there holds 
\begin{align}\label{eq127} 
&\left\| \nabla F (\mathbf{w})-\nabla F_S(\mathbf{w}) \right\| - \left\| \nabla F (\mathbf{w}^{\ast})-\nabla F_S(\mathbf{w}^{\ast}) \right \| \leq c' \beta \max \left \{ \| \mathbf{w} - \mathbf{w}^{\ast} \|, \frac{1}{n} \right\} \eta(\delta/2).
\end{align}
Then, using Lemma \ref{vector} (vector Bernstein inequality) and the fact $\nabla F (\mathbf{w}^{\ast}) = \mathbf{0}$, under Assumption \ref{assu5}, we have the following inequality with probability at least $1- \frac{\delta}{2}$
\begin{align}\label{eq128} 
&\left\| \nabla F (\mathbf{w}^{\ast} )-\nabla F_S(\mathbf{w}^{\ast}) \right \| \leq \frac{B_{\ast}\log(4/\delta)}{n} + \sqrt{\frac{2 \mathbb{E} [ \| \nabla f(\mathbf{w}^{\ast};z) \|^2 ] \log(4/\delta)}{n}}.
\end{align}
Combining (\ref{eq128}) and (\ref{eq127}), we obtain the following inequality with probability at least $1- \delta$
\begin{align}\label{ineq 12} 
\left\| \nabla F (\mathbf{w} )-\nabla F_S(\mathbf{w}) \right\| \leq c' \beta \max \left \{ \| \mathbf{w} - \mathbf{w}^{\ast} \|, \frac{1}{n} \right\}  \eta(\delta/2)  + \frac{B_{\ast}\log(4/\delta)}{n} + \sqrt{\frac{2 \mathbb{E} [ \| \nabla f(\mathbf{w}^{\ast};z) \|^2 ] \log(4/\delta)}{n}}.
\end{align}
This completes the proof of (\ref{firstpart}).

From (\ref{ineq 12}), we get the following inequality with probability at least $1- \delta$,
\begin{align}\label{129} \nonumber
 &\left\| \nabla F (\mathbf{w} )\right\| - \left\| \nabla F_S(\mathbf{w}) \right\| \leq\left\| \nabla F (\mathbf{w} )-\nabla F_S(\mathbf{w}) \right\|  
 \\\leq& c' \beta  \left (\| \mathbf{w} - \mathbf{w}^{\ast} \| + \frac{1}{n} \right ) \eta(\delta/2)  + \frac{B_{\ast}\log(4/\delta)}{n} + \sqrt{\frac{2 \mathbb{E} [ \| \nabla f(\mathbf{w}^{\ast};z) \|^2 ] \log(4/\delta)}{n}}.
\end{align}
According to Lemma \ref{appendixa}, we know that the PL condition of $F$ imply that for all $\mathbf{w} \in \mathcal{W}$, there holds 
\begin{align}\label{130}
\|  \nabla F(\mathbf{w})  \| \geq \mu \| \mathbf{w} - \mathbf{w}^{\ast} \|,
\end{align}
where $\mathbf{w}^{\ast}$ is the minimizer of $F$ closest to $\mathbf{w}$. 
Thus, combining (\ref{129}) and (\ref{130}), 
there holds the following inequality with probability at least $1- \delta$
\begin{align*}
\mu \| \mathbf{w} - \mathbf{w}^{\ast} \| \leq \|  \nabla F(\mathbf{w})  \|  \leq  \left\| \nabla F_S(\mathbf{w}) \right\| + c' \beta  \left (\| \mathbf{w} - \mathbf{w}^{\ast} \| + \frac{1}{n} \right )  \eta(\delta/2)   + \frac{B_{\ast}\log(4/\delta)}{n} + \sqrt{\frac{2 \mathbb{E} [ \| \nabla f(\mathbf{w}^{\ast};z) \|^2 ] \log(4/\delta)}{n}}.
\end{align*}
Let $c = \max\{ 4{c'}^2, 1\}$. When 
\begin{align*}
n \geq \frac{c\beta^2(d+ \log(\frac{8 \log(2n R +2)}{\delta}))}{\mu^2},
\end{align*}
we have $c' \beta \eta(\delta/2)  \leq \mu/2$, followed from the fact $\frac{\mu}{\beta} \leq 1$ \cite{nesterov2003introductory}. 
So we have
\begin{align}\label{132} 
&\| \mathbf{w} - \mathbf{w}^{\ast} \| \leq \frac{2}{\mu} \Big(\left\| \nabla F_S(\mathbf{w}) \right\| + \frac{B_{\ast}\log(4/\delta)}{n}  + \sqrt{\frac{2 \mathbb{E} [ \| \nabla f(\mathbf{w}^{\ast};z) \|^2 ] \log(4/\delta)}{n}} + \frac{\mu}{2n} \Big).
\end{align}
Plugging (\ref{132}) into (\ref{ineq 12}), we obtain that 
when $n \geq \frac{c\beta^2(d+ \log(\frac{8 \log(2n R +2)}{\delta}))}{\mu^2}$, with probability at least $1- \delta$,
\begin{align*}
 \left\| \nabla F (\mathbf{w} )- \nabla F_S(\mathbf{w}) \right\| \leq \left\| \nabla F_S(\mathbf{w}) \right\| +  \frac{\mu}{n}  + 2\frac{B_{\ast}\log(4/\delta)}{n} + 2\sqrt{\frac{2 \mathbb{E} [ \| \nabla f(\mathbf{w}^{\ast};z) \|^2 ] \log(4/\delta)}{n}}.
\end{align*}
Plugging (\ref{132}) into (\ref{129}), we obtain that 
when $n \geq \frac{c\beta^2(d+ \log(\frac{8 \log(2n R +2)}{\delta}))}{\mu^2}$, with probability at least $1- \delta$,
\begin{align*}
&\|  \nabla F(\mathbf{w})  \| 
\leq 2\left\| \nabla F_S(\mathbf{w}) \right\| +  \frac{\mu}{n}  + 2\frac{B_{\ast}\log(4/\delta)}{n} + 2\sqrt{\frac{2 \mathbb{E} [ \| \nabla f(\mathbf{w}^{\ast};z) \|^2 ] \log(4/\delta)}{n}}.
\end{align*}
The proof is complete.
\end{proof}

\subsection{Auxiliary Lemmas}
The following lemma provides relationships between the commonly used curvature conditions in stochastic optimization.
\begin{lemma}[Appendix A in \cite{karimi2016linear}]\label{appendixa}
Let $F(\mathbf{w})$ be differential, and assume that $\mathcal{W}_{\ast}$ be a non-empty solution set of $\arg \min_{\mathbf{w} \in \mathcal{W}} F(\mathbf{w})$. For any $\mathbf{w} \in \mathcal{W}$, let $\mathbf{w}^{\ast} = \arg \min_{\mathbf{u} \in \mathcal{W}_{\ast}} \| \mathbf{u} - \mathbf{w}\|$ denote an optimal solution closest to $\mathbf{w}$.

\noindent (1) Polyak-Lojasiewise (PL.) \cite{polyak1963gradient,lojasiewicz1963topological}: for all $\mathbf{w} \in \mathcal{W}$ we have 
\begin{align*}
F(\mathbf{w}) - F(\mathbf{w}^{\ast}) \leq \frac{1}{2\mu} \| \nabla F(\mathbf{w}) \|^2.
\end{align*}

\noindent (2) Error Bound (EB.) \cite{luo1993error}: for all $\mathbf{w} \in \mathcal{W}$ we have 
\begin{align*}
\| \nabla F(\mathbf{w}) \| \geq \mu \| \mathbf{w} -  \mathbf{w}^{\ast} \|.
\end{align*}

\noindent (3) Quadratic Growth (QG.) \cite{anitescu2000degenerate}: for all $\mathbf{w} \in \mathcal{W}$ we have 
\begin{align*}
F(\mathbf{w}) - F(\mathbf{w}^{\ast}) \geq \frac{\mu}{2} \| \mathbf{w} -  \mathbf{w}^{\ast} \|^2.
\end{align*}

\noindent There holds that:
\begin{align*}
 (PL) \rightarrow (QG), \ \quad  (PL) \rightarrow (EB) . 
\end{align*}
\end{lemma}
\begin{lemma}[\cite{ying2017unregularized}]\label{xinjiade}
Let $f$ be a differentiable function. Let $\alpha \in (0,1]$ and $P>0$. If Assumption \ref{assu4-5} holds, for any $\mathbf{w}_1, \mathbf{w}_2 \in \mathcal{W}$ and $z \in \mathcal{Z}$, then we have 
\begin{align*}
&f(\mathbf{w}_1;z) - f(\mathbf{w}_2;z)  \leq \langle \mathbf{w}_1 - \mathbf{w}_2, \nabla f(\mathbf{w}_2;z) \rangle + \frac{P\|\mathbf{w}_1 - \mathbf{w}_2  \|^{1+\alpha}}{1+\alpha}.
\end{align*}
\end{lemma} 
\begin{lemma}[Geometric reordering; \cite{li2020high}]\label{lem:reorder}
For sequences $\{a_t\}_{t\ge1},\{b_t\}_{t\ge1}$ and any integer $T\ge1$,
\[
\sum_{t=1}^{T} a_t\sum_{i=1}^{t} b_i=\sum_{t=1}^{T} b_t\sum_{i=t}^{T} a_i;
\qquad
\sum_{t=1}^{T} a_t\sum_{i=0}^{t-1} b_i=\sum_{t=1}^{T-1} b_t\sum_{i=t+1}^{T} a_i .
\]
\end{lemma}
\begin{lemma}[\cite{lei2021learning}]\label{lei32}
Let $e$ be the base of the natural logarithm.
There holds the following elementary inequalities.

\noindent (a) If $\theta \in (0,1)$, then $\sum_{k =1}^t k^{-\theta} \leq t^{1-\theta}/(1-\theta)$;

\noindent (b) If $\theta = 1$, then $\sum_{k =1}^t k^{-\theta} \leq \log (et)$;

\noindent (c) If $\theta > 1$, then $\sum_{k =1}^t k^{-\theta} \leq \frac{\theta}{\theta - 1}$.
\end{lemma}

\begin{lemma}[\cite{lei2021learning}]\label{lemma36}
Let $z_1,...,z_n$ be a sequence of randoms variables such that $z_k$ may depend the previous variables $z_1,...,z_{k-1}$ for all $k = 1,...,n$. Consider a sequence of functionals $\xi_k(z_1,...,z_k)$, $k=1,...,n$. Let $\sigma_n^2 = \sum_{k=1}^n \mathbb{E}_{z_k} [ (\xi_k - \mathbb{E}_{z_k}[\xi_k])^2]$ be the conditional variance.

\noindent (a) Assume $|\xi_k - \mathbb{E}_{z_k}[\xi_k] | \leq b_k$ for each $k$. Let $\delta \in (0,1)$. With probability at least $1-\delta$
\begin{align*}
\sum_{k=1}^n \xi_k - \sum_{k=1}^n \mathbb{E}_{z_k}[\xi_k] \leq \left( 2\sum_{k=1}^n b_k^2 \log \frac{1}{\delta} \right)^{\frac{1}{2}}.
\end{align*}

\noindent (b) Assume $|\xi_k - \mathbb{E}_{z_k}[\xi_k] | \leq b$ for each $k$. Let $\rho \in (0,1]$ and $\delta \in (0,1)$. With probability at least $1-\delta$ we have 
\begin{align*}
\sum_{k=1}^n \xi_k - \sum_{k=1}^n \mathbb{E}_{z_k}[\xi_k] \leq \frac{\rho \sigma_n^2}{b} + \frac{b \log \frac{1}{\delta}}{\rho}.
\end{align*}
\end{lemma}

\begin{lemma}[\cite{tarres2014online}]\label{lemma51}
Let $\{\xi_k \}_{k \in \mathbb{N}}$ be a martingale difference sequence in $\mathbb{R}^d$. Suppose that almost surely $\| \xi_k  \| \leq D$ and $ \sum_{k=1}^t \mathbb{E} [ \| \xi_k \|^2 | \xi_1,...,\xi_{k-1}]\leq \sigma_t^2$. Then, for any $0< \delta < 1$, the following inequality holds with probability at least $1-\delta$
\begin{align*}
\max_{1\leq j \leq t} \left \| \sum_{k =1}^j \xi_k \right\| \leq 2\left( \frac{D}{3} + \sigma_t  \right)\log \frac{2}{\delta}.
\end{align*}
\end{lemma}

\begin{lemma}[\cite{pinelis1994optimum,pinelis1999correction}]\label{vector} 
Let $X_1,...,X_n$ be a sequence of i.i.d. random variables taking values in a real separable Hilbert space. Assume that $\mathbb{E} [X_i] = \mu$, $\mathbb{E} [\| X_i - \mu\|^2] = \sigma^2$, $\forall 1\leq i \leq n$. If for all $1 \leq i \leq n$, vector $X_i$ satisfying the following Bernstein condition with parameter $B$
\begin{align*}
\mathbb{E}\left [ \|X_i - \mu \|^k \right] \leq \frac{1}{2} k! \sigma^2 B^{k-2}, \quad \forall  k \ge 2.
\end{align*}
Then for all $\delta \in (0,1)$, with probability $1-\delta$, there holds that
\begin{align*}
\left\| \frac{1}{n}\sum_{i =1}^n X_i - \mu \right \| \leq \frac{B\log(2/\delta)}{n} + \sqrt{\frac{2\sigma^2 \log(2/\delta)}{n}}.
\end{align*}
\end{lemma}

\section{Proofs for Stochastic Gradient Descent}\label{section5}
Section \ref{optimisgd} establishes optimization error bounds of SGD with H\"older smooth functions. The technique is inspired by \cite{lei2021learning} and is organized around three lemmas: bound on gradient norm (Lemma \ref{lemma 43}), bound on iteration norm (Lemma \ref{boundnorm}), and bound on optimization error with an additional PL condition (Lemma \ref{lemmajj}). These Lemmas will serve as fundamental results on the derivation of generalization bound. The following Section \ref{proof-sec-3.4} and \ref{proof-sec-3.5} proves Theorems \ref{theo67} and \ref{theo7}, respectively.
\subsection{Optimization Error of SGD}\label{optimisgd}
\begin{assumption}[H\"{o}lder Smoothness]\label{assu4-5}
Let $P >0$ and $\alpha \in (0,1]$. We say function $f$ is $\alpha$-H\"{o}lder smooth w.r.t. the first argument with parameter $P$ if for all $\mathbf{w}_1, \mathbf{w}_2 \in \mathcal{W}$ and $z \in \mathcal{Z}$,
\begin{align*}
\| \nabla f(\mathbf{w}_1;z) - \nabla f(\mathbf{w}_2;z) \|  \leq P \| \mathbf{w}_1 - \mathbf{w}_2  \|^{\alpha}.
\end{align*}
\end{assumption}
\begin{remark}\rm{}\label{zuihouyinggaibushi}
Assumption~\ref{assu4-5} interpolates between bounded-variation gradients and Lipschitz gradients; the case $\alpha=1$ recovers Assumption~\ref{assu4}.  We establish optimization bounds for SGD under this weaker condition.
\end{remark}
\begin{lemma}\label{lemma 43}
Suppose Assumptions \ref{assu7}, \ref{assu8} and \ref{assu4-5} hold.
Let $\{ \mathbf{w}_t\}_t$ be the sequence produced by SGD, i.e. (\ref{eq1}), with $\eta_t \leq (1/(2P))^{1/\alpha}$ for all $t \in \mathbb{N}$. 
Then, for any $\delta > 0$, with probability $1 - \delta$, we have
\begin{align*}
\sum_{k=1}^t  \eta_k \| \nabla F_S(\mathbf{w}_{k}) \|^2 \leq C \log(2/\delta)  +  C_t  ,
\end{align*}
where $C_t = 4 \sup_{z \in \mathcal{Z}} f(\mathbf{0};z) + 4\max \{PG^2,C_1 \}( \sum_{k=1}^t \eta_k^{2\alpha} +\sum_{k=1}^t \eta_k^{1+\alpha} )$ and $C = 32P G^2+8\max\{ G^2 , 2\sigma^2 (2P)^{-1/\alpha}\}  $, and where $C_1 = \frac{P }{1+\alpha}(\frac{1-\alpha}{2} + (1+\alpha) \sigma^2  )$.
\end{lemma}
\begin{proof}The proof proceeds with three steps.

\medskip\noindent\textbf{Step 1: A decomposition under H\"older smoothness.}
Since  function $f$ satisfies Assumption \ref{assu4-5}, it is easy to verify that $F_S$ also satisfies Assumption \ref{assu4-5}:
\begin{align*}
&\| \nabla F_S(\mathbf{w}_1) - \nabla F_S(\mathbf{w}_2) \|  = \left \|  \frac{1}{n} \sum_{i=1}^n (\nabla f(\mathbf{w}_1;z_i) -   \nabla f(\mathbf{w}_2;z_i) ) \right \|  \leq \frac{1}{n}\sum_{i=1}^n \|\nabla f(\mathbf{w}_1;z_i)  - \nabla f(\mathbf{w}_2;z_i) \|  \leq P\| \mathbf{w}_1 -  \mathbf{w}_2 \|^{\alpha}.
\end{align*}
Denote by $\xi_t=  \eta_t \langle \nabla F_S(\mathbf{w}_{t}) - \nabla f(\mathbf{w}_{t};z_{j_t}), \nabla F_S(\mathbf{w}_{t}) \rangle$ and $\xi'_t =  \| \nabla f(\mathbf{w}_{t}; z_{j_t})  - \nabla F_S(\mathbf{w}_{t}) \|^2 - \mathbb{E}_{j_t}\| \nabla f(\mathbf{w}_{t}; z_{j_t})  - \nabla F_S(\mathbf{w}_{t}) \|^2$. From Lemma \ref{xinjiade} we have
\begin{align}\label{yuanshijiande}\nonumber
 &F_S(\mathbf{w}_{t+1}) \leq F_S(\mathbf{w}_{t}) + \langle \mathbf{w}_{t+1} - \mathbf{w}_{t}, \nabla F_S(\mathbf{w}_{t}) \rangle  + \frac{P\|\mathbf{w}_{t+1} - \mathbf{w}_t  \|^{1+\alpha}}{1+\alpha} \\\nonumber
&= F_S(\mathbf{w}_{t}) + \eta_t \langle \nabla F_S(\mathbf{w}_{t}) - \nabla f(\mathbf{w}_{t}; z_{j_t}), \nabla F_S(\mathbf{w}_{t}) \rangle -\eta_t \| \nabla F_S(\mathbf{w}_{t}) \|^2 + \frac{P \eta_t^{1+\alpha}}{1+\alpha} \| \nabla f(\mathbf{w}_{t}; z_{j_t})  \|^{1+\alpha}\\\nonumber
& = F_S(\mathbf{w}_{t}) + \xi_t -\eta_t \| \nabla F_S(\mathbf{w}_{t}) \|^2 + \frac{P \eta_t^{1+\alpha}}{1+\alpha} \| \nabla f(\mathbf{w}_{t}; z_{j_t})  \|^{1+\alpha} \\\nonumber
& \leq F_S(\mathbf{w}_{t}) + \xi_t -\eta_t \| \nabla F_S(\mathbf{w}_{t}) \|^2  + \frac{P \eta_t^{1+\alpha}}{1+\alpha} \left[\frac{1-\alpha}{2}+  \frac{1+\alpha}{2}(\| \nabla f(\mathbf{w}_{t}; z_{j_t})  \|^{1+\alpha} )^{\frac{2}{1+\alpha}}   \right]\\\nonumber
& = F_S(\mathbf{w}_{t}) + \xi_t -\eta_t \| \nabla F_S(\mathbf{w}_{t}) \|^2+ \frac{P \eta_t^{1+\alpha}}{1+\alpha} \Big[\frac{1-\alpha}{2}  +  \frac{1+\alpha}{2}\| \nabla f(\mathbf{w}_{t}; z_{j_t}) - \nabla F_S(\mathbf{w}_{t}) + \nabla F_S(\mathbf{w}_{t}) \|^2   \Big]\\\nonumber
& \leq F_S(\mathbf{w}_{t}) + \xi_t -\eta_t \| \nabla F_S(\mathbf{w}_{t}) \|^2 + \frac{P \eta_t^{1+\alpha}}{1+\alpha} \Big[\frac{1-\alpha}{2} +  \frac{1+\alpha}{2}  \left[ 2 \xi'_t + 2 \mathbb{E}_{j_t} [\| \nabla f(\mathbf{w}_{t}; z_{j_t}) - \nabla F_S(\mathbf{w}_{t}) \|^2 ]+ 2 \|\nabla F_S(\mathbf{w}_{t}) \|^2 \right]   \Big]\\\nonumber
& \leq F_S(\mathbf{w}_{t}) + \xi_t -\eta_t \| \nabla F_S(\mathbf{w}_{t}) \|^2  + \frac{P \eta_t^{1+\alpha}}{1+\alpha} \Big[\frac{1-\alpha}{2}+  \frac{1+\alpha}{2}\left[ 2 \xi'_t + 2 \sigma^2 + 2 \|\nabla F_S(\mathbf{w}_{t}) \|^2 \right]   \Big]\\ 
& \leq F_S(\mathbf{w}_{t}) + \xi_t -2^{-1}\eta_t \| \nabla F_S(\mathbf{w}_{t}) \|^2  + \frac{P \eta_t^{1+\alpha}}{1+\alpha} \left[\frac{1-\alpha}{2}+  (1+\alpha)\left[  \xi'_t +  \sigma^2 \right]   \right],
\end{align}
 where the second inequality follows from the Young's inequality: for all $\mu,v \in \mathbb{R}, p^{-1}+q^{-1}=1, p \geq 0$
\begin{align*}
\mu v \leq p^{-1}|\mu|^p + q^{-1}|v|^q,
\end{align*}
where the fourth inequality follows from Assumption \ref{assu8}, and where the last follows from the fact that $P \eta_t^{1+\alpha} \leq P [(1/(2P))^{1/\alpha}]^{\alpha}  \eta_t \leq 2^{-1} \eta_t$.
Denote by $C_1 = \frac{P }{1+\alpha}(\frac{1-\alpha}{2} + (1+\alpha) \sigma^2  )$.
Taking a summation of the above inequality gives 
\begin{align} \label{yaodairude}\nonumber
&F_S(\mathbf{w}_{t+1}) = F_S(\mathbf{w}_{1}) + \sum_{k=1}^t(F_S(\mathbf{w}_{k+1}) -  F_S(\mathbf{w}_{k}))   \\\leq& F_S(\mathbf{w}_{1}) + \sum_{k=1}^t \xi_k - \frac{1}{2}\sum_{k=1}^t  \eta_k \| \nabla F_S(\mathbf{w}_{k}) \|^2   + \sum_{k=1}^t P \eta_k^{1+\alpha}  \xi'_k + \sum_{k=1}^t C_1 \eta_k^{1+\alpha}.
\end{align}

\medskip\noindent\textbf{Step 2: High-probability control of noise terms.} This step bounds $\sum_{k=1}^t \xi_k$ and $\sum_{k=1}^t P \eta_k^{1+\alpha}  \xi'_k$.

(I) Since $\mathbb{E}_{j_k} \xi_k = 0$, thus $\{ \xi_k \}$ is a martingale difference sequence. There holds
\begin{align}\label{inequ75}\nonumber
|\xi_k| &\leq \eta_k \left (  \| \nabla F_S(\mathbf{w}_{k})\| + \| \nabla f(\mathbf{w}_{k};z_{j_k}) \| \right) \| \nabla F_S(\mathbf{w}_{k}) \|\\& = \sqrt{ \eta_k} \left (  \| \nabla F_S(\mathbf{w}_{k})\| + \| \nabla f(\mathbf{w}_{k};z_{j_k}) \| \right)\sqrt{ \eta_k} \| \nabla F_S(\mathbf{w}_{k}) \| \leq 2G^2,
\end{align}
followed from Assumption \ref{assu7}. 
Moreover, we have
\begin{align}\nonumber \label{inequ76}
 &\sum_{k=1}^t \mathbb{E}_{j_k} \left[ ( \xi_k  - \mathbb{E}_{j_k} \xi_k )^2 \right ]  = \sum_{k=1}^t \mathbb{E}_{j_k}  ( \xi_k)^2 \leq \sum_{k=1}^t \eta_k^2 \mathbb{E}_{j_k} \left (  \| \nabla F_S(\mathbf{w}_{k}) - \nabla f(\mathbf{w}_{k}; z_{j_k}) \|^2 \right) \| \nabla F_S(\mathbf{w}_{k}) \|^2 \\&\leq \sigma^2 \sum_{k=1}^t \eta_k^2  \| \nabla F_S(\mathbf{w}_{k}) \|^2 \leq (2P)^{-\frac{1}{\alpha}}\sigma^2 \sum_{k=1}^t \eta_k  \| \nabla F_S(\mathbf{w}_{k}) \|^2,
\end{align}
where the second inequality follows from Assumption \ref{assu8} and the last inequality follows from the fact that $\eta_t \leq (1/(2P))^{1/\alpha}$ for all $t \in \mathbb{N}$. Substituting (\ref{inequ75}) and (\ref{inequ76}) into part (b) of Lemma \ref{lemma36} with $\rho  = \min \{ 1, G^2 (2P)^{1/\alpha} (2\sigma^2 )^{-1} \}$, we have the following inequality with probability at least $1  - \delta/2$ 
\begin{align} \label{yihaishi}\nonumber
 \sum_{k=1}^t \xi_k  &\leq \frac{\rho (2P)^{-1/\alpha} \sigma^2  \sum_{k = 1}^t \eta_k \| \nabla F_S(\mathbf{w}_{k}) \|^2}{2G^2} + \frac{2G^2 \log(2/\delta)}{\rho} \\&\leq \frac{1}{4} \sum_{k = 1}^t \eta_k \| \nabla F_S(\mathbf{w}_{k}) \|^2 + 2 \log(\frac{2}{\delta}) \max\{ G^2 , 2\sigma^2 (2P)^{-\frac{1}{\alpha}}\}.
\end{align}

(II) Similarly, since $\mathbb{E}_{j_k} \xi_k' = 0$, thus $\{ \xi_k' \}$ is a martingale difference sequence. According to Assumption \ref{assu7}, we have
\begin{align}\label{inequ77}
|\xi'_k| \leq 2 \left (  \| \nabla F_S(\mathbf{w}_{k})\|^2 + \| \nabla f(\mathbf{w}_{k}; z_{j_k}) \|^2 \right) = 2 \eta_k^{-1} \eta_k\left (  \| \nabla F_S(\mathbf{w}_{k})\|^2 + \| \nabla f(\mathbf{w}_{k}; z_{j_k}) \|^2 \right)  \leq 4\eta_k^{-1}G^2.
\end{align}
Substituting (\ref{inequ77}) into part (a) of Lemma \ref{lemma36}, we have the following inequality with probability at least $1-\delta/2$
\begin{align} \label{erhaozhenengchongshi} 
\sum_{k=1}^t \eta_k^{1+\alpha} \xi'_k &\leq 4G^2 \left( 2\sum_{k=1}^t \eta_k^{2\alpha} \log\left(\frac{2}{\delta} \right)  \right)^{\frac{1}{2}}  \leq 8G^2 \log(2/\delta) + G^2 \sum_{k=1}^t \eta_k^{2\alpha},
\end{align}
where the last inequality follows from the Schwarz's inequality.

\medskip\noindent\textbf{Step 3: Putting together.}
Substituting (\ref{yihaishi}) and (\ref{erhaozhenengchongshi}) into (\ref{yaodairude}), we have the following inequality with probability at least $1-\delta$
\begin{align*}
F_S(\mathbf{w}_{t+1}) 
  &\leq  F_S(\mathbf{0}) + 2 \log(2/\delta) \max\{ G^2 , 2\sigma^2 (2P)^{-1/\alpha}\}   \\&- \frac{1}{4}\sum_{k=1}^t  \eta_k \| \nabla F_S(\mathbf{w}_{k}) \|^2  +  P \left[ 8G^2 \log(2/\delta) + G^2 \sum_{k=1}^t \eta_k^{2\alpha} \right]  + \sum_{k=1}^t C_1 \eta_k^{1+\alpha},
\end{align*}
which implies that 
\begin{align*}
& \frac{1}{4}\sum_{k=1}^t  \eta_k \| \nabla F_S(\mathbf{w}_{k}) \|^2   \leq  F_S(\mathbf{0}) +  \log(2/\delta) \left[2\max\{ G^2 , 2\sigma^2 (2P)^{-1/\alpha}\}  +  8P G^2  \right]  + PG^2 \sum_{k=1}^t \eta_k^{2\alpha}  + \sum_{k=1}^t C_1 \eta_k^{1+\alpha}.
\end{align*}
Therefore, we have the following inequality with probability at least $1-\delta$
\begin{align*}
\sum_{k=1}^t  \eta_k \| \nabla F_S(\mathbf{w}_{k}) \|^2 \leq C \log(2/\delta)  +  C_t  ,
\end{align*}
where $C_t = 4 \sup_{z \in \mathcal{Z}} f(\mathbf{0};z) + 4\max \{PG^2,C_1 \}( \sum_{k=1}^t \eta_k^{2\alpha} +\sum_{k=1}^t \eta_k^{1+\alpha} )$ and $C = 32P G^2+8\max\{ G^2 , 2\sigma^2 (2P)^{-1/\alpha}\}  $. The proof is complete.
\end{proof}
\begin{lemma}\label{boundnorm}
Suppose Assumptions \ref{assu7}, \ref{assu8} and \ref{assu4-5} hold.
Let $\{ \mathbf{w}_t\}_t$ be the sequence produced by SGD, i.e. (\ref{eq1}), with $\eta_t \leq (1/(2P))^{1/\alpha}$ for all $t \in \mathbb{N}$. 
Then, for any $\delta > 0$, with probability $1 - \delta$, we have the following inequality uniformly for all $t = 1,...T$
\begin{align*}
&\| \mathbf{w}_{t+1} \| \leq C_2 \left( \Big( \sum_{k=1}^T \eta_k^2 \Big)^{\frac{1}{2}} + 1 + \Big( \sum_{k=1}^t \eta_k \Big)^{\frac{1}{2}}  + \Big( \sum_{k=1}^t \eta_k \Big)^{\frac{1}{2}} \Big(\sum_{k=1}^t \eta_k^{2\alpha} +\sum_{k=1}^t \eta_k^{1+\alpha} \Big)^{\frac{1}{2}} \right) \log \Big(\frac{4}{\delta} \Big),
\end{align*}
where $C_2 = \max \Big\{ \frac{4G (1/2P)^{1/2\alpha}}{3},2\sigma, 4\sqrt{C+4 \sup_{z \in \mathcal{Z}} f(\mathbf{0};z)} ,  4\sqrt{4\max \{PG^2,C_1 \}}   \Big\}$.
\end{lemma}
\begin{proof}The proof proceeds with four steps.

\medskip\noindent\textbf{Step 1: A recursive expression of iteration.}
Denote by $\xi_t=  \eta_t ( \nabla F_S(\mathbf{w}_{t}) - \nabla f(\mathbf{w}_{t};z_{j_t})) $. 
According to the iteration update of SGD, we have
\begin{align*}
\mathbf{w}_{t+1} = \mathbf{w}_t - \eta_t \left( \nabla f (\mathbf{w}_t;z_{j_t}) - \nabla F_S(\mathbf{w}_t)  \right) -\eta_t \nabla F_S(\mathbf{w}_t)).
\end{align*}
Taking a summation and using $\mathbf{w}_1 = \mathbf{0}$, we get 
\begin{align*}
\mathbf{w}_{t+1} = \sum_{k=1}^t \xi_k - \sum_{k=1}^t\eta_k \nabla F_S(\mathbf{w}_k).
\end{align*}
By the triangle inequality of the norm,
\begin{align}\label{yuanshi}
\| \mathbf{w}_{t+1} \| \leq \left \|\sum_{k=1}^t \xi_k \right \| + \left \| \sum_{k=1}^t\eta_k \nabla F_S(\mathbf{w}_k) \right \|.
\end{align}

\medskip\noindent\textbf{Step 2: High-probability control of noise terms.}
Since $\mathbb{E}_{j_k} \xi_k = 0$, thus $\{ \xi_k \}$ is a martingale difference sequence. Firstly,
\begin{align}\label{yihao} 
\| \xi_k \|  = \eta_k \| \nabla F_S(\mathbf{w}_{k}) - \nabla f(\mathbf{w}_{k}; z_{j_k}) \|  \leq \sqrt{\eta_k} \left(2 \sqrt{\eta_k} \sup_{z\in \mathcal{Z}}\|\nabla f(\mathbf{w}_{k};z) \| \right )  \leq 2G \sqrt{\eta_k} \leq 2G (2P)^{-1/(2\alpha)},
\end{align}
where the second inequality follows from Assumption \ref{assu7} and the last inequality follows from the fact that $\eta_t \leq (1/2P)^{1/\alpha}$ for all $t \in \mathbb{N}$. Secondly, according to Assumption \ref{assu8}, we have
\begin{align}\label{erhao}
\sum_{k=1}^T \mathbb{E}_{j_k} [\| \xi_k \|^2]  \leq \sum_{k=1}^T \eta_k^2 \sigma^2.
\end{align}
Substituting (\ref{yihao}) and (\ref{erhao}) into Lemma \ref{lemma51}, we have the following inequality with probability at least $1-\delta/2$
\begin{align}\label{youerhao}
\max_{1\leq t \leq T} \left \| \sum_{k =1}^t \xi_k \right\| \leq 2\left( \frac{2G  (2P)^{-1/(2\alpha)}}{3} + \sigma \left(\sum_{k=1}^T \eta_k^2   \right)^{\frac{1}{2}}  \right)\log \frac{4}{\delta}.
\end{align}

\medskip\noindent\textbf{Step 3: Bounding the $  \| \sum_{k=1}^t\eta_k \nabla F_S(\mathbf{w}_k) \|$ term.}
For the term $\left \| \sum_{k=1}^t\eta_k \nabla F_S(\mathbf{w}_k) \right \|$, according to Lemma \ref{lemma 43} and the Schwarz's inequality, we have the following inequality with probability at least $1-\delta/2$,
\begin{align}\label{youyihao} \left \| \sum_{k=1}^t\eta_k \nabla F_S(\mathbf{w}_k) \right \|^2  \leq  \left ( \sum_{k=1}^t\eta_k \| \nabla F_S(\mathbf{w}_k) \| \right)^2 \leq \left( \sum_{k=1}^t\eta_k \right) \left( \sum_{k=1}^t\eta_k \| \nabla F_S(\mathbf{w}_k) \|^2 \right)  \leq \left( \sum_{k=1}^t\eta_k \right) \left( C \log(4/\delta)  +  C_t \right)  .
\end{align}

\medskip\noindent\textbf{Step 4: Putting together.}
Substituting (\ref{youyihao}) and (\ref{youerhao}) into (\ref{yuanshi}), we have the following inequality with probability at least $1-\delta$ uniformly for all $t = 1,...T$
\begin{align*}
\| \mathbf{w}_{t+1} \| & \leq 2\left( \frac{2G  (2P)^{-1/(2\alpha)}}{3} + \sigma \Big(\sum_{k=1}^T \eta_k^2   \Big)^{\frac{1}{2}}  \right)\log \frac{4}{\delta}  + \left( \sum_{k=1}^t\eta_k  \Big( C \log(4/\delta)  +  C_t \Big) \right)^{\frac{1}{2}}\\
& \leq C_2 \left( \Big( \sum_{k=1}^T \eta_k^2 \Big)^{1/2} + 1 + \Big( \sum_{k=1}^t \eta_k \Big)^{1/2}  + \Big( \sum_{k=1}^t \eta_k \Big)^{1/2} \Big(\sum_{k=1}^t \eta_k^{2\alpha} +\sum_{k=1}^t \eta_k^{1+\alpha} \Big)^{1/2} \right) \log(4/\delta),
\end{align*}
where $C_2 = \max \Big\{ \frac{4G  (2P)^{-1/(2\alpha)}}{3},2\sigma, 4\sqrt{C+4 \sup_{z \in \mathcal{Z}} f(\mathbf{0};z)} ,   4\sqrt{4\max \{PG^2,C_1 \}}   \Big\}$.
The proof is complete.
\end{proof}
\begin{lemma}\label{lemmajj}
Suppose Assumptions \ref{assu7}, \ref{assu8} and \ref{assu4-5} hold, and suppose $F_S$ satisfies Assumption \ref{assu10} with parameter $2\mu_S$.
Let $\{ \mathbf{w}_t\}_t$ be the sequence produced by SGD, i.e. (\ref{eq1}), with $\eta_t = \frac{2}{\mu_S (t+t_0)}$ such that $t_0 \geq \max \left\{\frac{2(2P)^{1/\alpha}}{\mu_S},1 \right\}$ for all $t \in \mathbb{N}$. 
Then, for any $\delta > 0$, with probability at least $1 - \delta$, we have
\begin{align*}
F_S(\mathbf{w}_{T+1}) - F_S^{\ast} = 
\begin{cases}
             \mathcal{O}\left( \frac{1}{T^{\alpha}}  \right) &\quad \text{if   } \alpha  \in (0,1 ),  \\ 
             \mathcal{O}\left (\frac{\log(T) \log^3(1/\delta)}{T} \right)  &\quad \text{if   }\alpha = 1.
\end{cases}
\end{align*}
\end{lemma}
\begin{proof}
The proof proceeds with five steps.

\medskip\noindent\textbf{Step 1: A new decomposition under H\"older smoothness.}
When $\eta_t = \frac{2}{\mu_S (t+t_0)}$ with $t_0 \geq \frac{2(2P)^{1/\alpha}}{\mu_S}$, we have $\eta_t \leq (2P)^{-1/\alpha}$.
Thus, from (\ref{yuanshijiande}), we know  
\begin{align*}
F_S(\mathbf{w}_{t+1})  \leq F_S(\mathbf{w}_{t}) + \xi_t -2^{-1}\eta_t \| \nabla F_S(\mathbf{w}_{t}) \|^2  + \frac{P \eta_t^{1+\alpha}}{1+\alpha} \left[\frac{1-\alpha}{2}+  (1+\alpha)\left(  \xi'_t +  \sigma^2 \right)   \right],
\end{align*}
where $\xi_t=  \eta_t \langle \nabla F_S(\mathbf{w}_{t}) - \nabla f(\mathbf{w}_{t}; z_{j_t}), \nabla F_S(\mathbf{w}_{t}) \rangle$ and $\xi'_t =  \| \nabla f(\mathbf{w}_{t}; z_{j_t})  - \nabla F_S(\mathbf{w}_{t}) \|^2 - \mathbb{E}_{j_t}\| \nabla f(\mathbf{w}_{t}; z_{j_t})  - \nabla F_S(\mathbf{w}_{t}) \|^2$.
Since $F_S$ satisfies the PL condition with parameter $2\mu_S$, which means $F_S (\mathbf{w}) - F_S^{\ast} \leq \frac{1}{4 \mu_S} \| \nabla F_S(\mathbf{w}) \|^2$ by Assumption \ref{assu10}, 
we have 
\begin{align}\label{eq-recursion-for-sgd}
F_S(\mathbf{w}_{t+1}) &\leq
 F_S(\mathbf{w}_{t}) + \xi_t -4^{-1}\eta_t \| \nabla F_S(\mathbf{w}_{t}) \|^2 + P \eta_t^{1+\alpha}  \xi'_t  +  C_1 \eta_t^{1+\alpha}  + \eta_t \mu_S (F_S^{\ast} - F_S(\mathbf{w}_{t})),
\end{align}
where $C_1 = \frac{P }{1+\alpha}(\frac{1-\alpha}{2} + (1+\alpha) \sigma^2  )$.

The inequality (\ref{eq-recursion-for-sgd}) implies that  
\begin{align*}
&F_S(\mathbf{w}_{t+1}) - F_S^{\ast}  + \frac{1}{ 2\mu_S (t+t_0)} \| \nabla F_S(\mathbf{w}_{t}) \|^2   \leq  \xi_t  +  P \eta_t^{1+\alpha}  \xi'_t +  C_1 \eta_t^{1+\alpha} + \frac{t+t_0-2}{t+t_0} ( F_S(\mathbf{w}_{t}) - F_S^{\ast} ).
\end{align*}
Multiply both side by $(t+t_0)(t+t_0-1)$ gives
\begin{align*}
&(t+t_0)(t+t_0-1)(F_S(\mathbf{w}_{t+1}) - F_S^{\ast}  ) + \frac{t+t_0-1}{2 \mu_S } \| \nabla F_S(\mathbf{w}_{t}) \|^2  \\
 \leq &(t+t_0)(t+t_0-1)\xi_t  + P\left(\frac{2}{\mu_S}\right)^{1+\alpha} \xi'_t (t+t_0)^{-\alpha}(t+t_0-1)\\
 &+C_1 \left(\frac{2}{\mu_S}\right)^{1+\alpha} (t+t_0)^{-\alpha}(t+t_0-1)  + (t+t_0-1)(t+t_0-2) ( F_S(\mathbf{w}_{t}) - F_S^{\ast} ).
\end{align*}
Take a summation from $t =1$ to $t=T$, we obtain
\begin{align}\label{zehkojcuui}\nonumber
&\sum_{t = 1}^T \frac{t+t_0-1}{2 \mu_S } \| \nabla F_S(\mathbf{w}_{t}) \|^2   + (T+t_0)(T+t_0-1)(F_S(\mathbf{w}_{T+1}) - F_S^{\ast}  )  \\\nonumber
 \leq &\sum_{t = 1}^T  (t+t_0)(t+t_0-1)\xi_t  +  \sum_{t = 1}^T P \left(\frac{2}{\mu_S}\right)^{1+\alpha} \xi'_t (t+t_0)^{-\alpha}(t+t_0-1)  \\ 
 &+\sum_{t = 1}^T C_1 \left(\frac{2}{\mu_S}\right)^{1+\alpha} (t+t_0)^{-\alpha}(t+t_0-1)  +  (t_0-1)t_0( F_S(\mathbf{w}_{1}) - F_S^{\ast}).
\end{align}

\medskip\noindent\textbf{Step 2: A bound on the iteration.}
After obtaining the decomposition in (\ref{zehkojcuui}), we begin to bound the term $\| \mathbf{w}_{t+1} \|$. This bound will later allow us to dispense with the bounded-gradient assumption—see (\ref{vgrgrgtt}) for where it is invoked.

When  $\eta_t = \frac{2}{\mu_S (t+t_0)}$ and $t_0 \geq 1$, we have 
\begin{align*}
\sum_{t=1}^T \eta_t = \frac{2}{\mu_S}\sum_{t=1}^T \frac{1}{t+t_0} \leq \frac{2}{\mu_S} \log(T+1). 
\end{align*}
According to Lemma \ref{boundnorm}, we have the following inequality with probability at least $1 - \delta/2$ uniformly for all $t = 1,...,T$
\begin{align*}
&\| \mathbf{w}_{t+1}  \| \leq C_2 \left( \Big( \sum_{k=1}^T \eta_k^2 \Big)^{\frac{1}{2}} + 1 + \Big( \sum_{k=1}^T \eta_k \Big)^{\frac{1}{2}}  + \Big( \sum_{k=1}^T \eta_k \Big)^{\frac{1}{2}} \Big(\sum_{k=1}^T \eta_k^{2\alpha} +\sum_{k=1}^T \eta_k^{1+\alpha} \Big)^{\frac{1}{2}} \right) \log \Big(\frac{8}{\delta} \Big)\\
&\leq C_2 \Big( \Big(\frac{1}{2P}\Big)^\frac{1}{2\alpha}\Big( \sum_{k=1}^T \eta_k \Big)^{\frac{1}{2}} + 1 + \Big( \sum_{k=1}^T \eta_k \Big)^{\frac{1}{2}}  + \Big( \sum_{k=1}^T \eta_k \Big)^{\frac{1}{2}} \Big(\sum_{k=1}^T \eta_k^{2\alpha} +(\frac{2}{\mu_S})^{1+\alpha}\frac{1+\alpha}{\alpha} \Big)^{\frac{1}{2}} \Big) \log \Big(\frac{8}{\delta} \Big)\\
&\leq 2C_2\max \Big\{ \Big(\frac{1}{2P}\Big)^\frac{1}{2\alpha},1, \Big((\frac{2}{\mu_S})^{1+\alpha}\frac{(1+\alpha)}{\alpha}\Big)^{\frac{1}{2}} \Big\}  \Big(\Big( \sum_{k=1}^T \eta_k \Big)^{\frac{1}{2}}  + 1 +  \Big( \sum_{k=1}^T \eta_k \Big)^{\frac{1}{2}} \Big(\sum_{k=1}^T \eta_k^{2\alpha}\Big)^{\frac{1}{2}} \Big) \log \Big(\frac{8}{\delta} \Big)\\
&\leq C_3  \Big(\sqrt{\frac{2}{\mu_S} } \log^{\frac{1}{2}}(T+1)+ 1  +  \sqrt{\frac{2}{\mu_S} } \log^{\frac{1}{2}}(T+1)  \Big(\sum_{k=1}^T \eta_k^{2\alpha} \Big)^{\frac{1}{2}} \Big) \log \Big(\frac{8}{\delta} \Big)\\
&\leq C_3 \max \Big\{ \sqrt{\frac{2}{\mu_S} },1 \Big\} \Big( \log^{\frac{1}{2}}(T+1)  \Big(\sum_{k=1}^T \eta_k^{2\alpha} \Big)^{\frac{1}{2}} \Big) \log \Big(\frac{8}{\delta} \Big) \leq C_{T, \delta},
\end{align*}
where the second inequality follows from $\eta_t \leq (2P)^{-1/\alpha}$ and Lemma \ref{lei32} together with $1+\alpha >1$, and where in the fourth inequality we denote $C_3 := 2C_2\max \left\{ \left(\frac{1}{2P}\right)^\frac{1}{2\alpha},1, \left((\frac{2}{\mu_S})^{1+\alpha}\frac{(1+\alpha)}{\alpha}\right)^{\frac{1}{2}} \right\} $.

The dominated term in the above inequality is $\log^{\frac{1}{2}}(T+1)  \left(\sum_{k=1}^T \eta_k^{2\alpha} \right)^{\frac{1}{2}} \log \left(\frac{8}{\delta} \right)$. According to Lemma \ref{lei32},
if $\alpha \in (0,\frac{1}{2})$, $(\sum_{k=1}^T \eta_k^{2\alpha})^{\frac{1}{2}} = \mathcal{O}(T^{(1-2\alpha)/2})$; if $\alpha = \frac{1}{2}$, $(\sum_{k=1}^T \eta_k^{2\alpha})^{\frac{1}{2}} = \mathcal{O}(\log^{1/2}(T))$; if $\alpha \in (\frac{1}{2},1]$, $(\sum_{k=1}^T \eta_k^{2\alpha})^{\frac{1}{2}} = \mathcal{O}(1)$. 
Hence, we obtain the following result with probability at least $1 - \delta/2$ uniformly for all $t = 1,...,T$
 \begin{align*}
\|  \mathbf{w}_{t+1} \| = 
\begin{cases}
             \mathcal{O} \left (\log^{\frac{1}{2}}(T) T^{(1-2\alpha)/2} \log \left(\frac{1}{\delta} \right) \right) &\quad \text{if   } \alpha  \in (0,\frac{1}{2} ),  \\ 
             \mathcal{O}\left (\log(T) \log \left(\frac{1}{\delta} \right) \right)  &\quad \text{if   }\alpha = 1/2,\\ 
             \mathcal{O}\left (\log^{\frac{1}{2}}(T) \log \left(\frac{1}{\delta} \right) \right)  &\quad \text{if   }\alpha \in (\frac{1}{2},1].
\end{cases}
 \end{align*}
For brevity, we denote $C_{T, \delta}$ as the upper bound of $\| \mathbf{w}_{t+1} \|$ for all $t = 1,...,T$. For example, if $\alpha = 1/2$, there holds $C_{T, \delta} = \mathcal{O}(\log T \log (1/\delta))$.

\medskip\noindent\textbf{Step 3: High-probability control of noise terms.}
This step bounds $\sum_{t = 1}^T  (t+t_0)(t+t_0-1)\xi_t$ and $P \left(\frac{2}{\mu_S}\right)^{1+\alpha} \sum_{t = 1}^T  \xi'_t (t+t_0)^{-\alpha}(t+t_0-1)$ in (\ref{zehkojcuui}) conditioned on the event that $\|  \mathbf{w}_{t}\| \leq C_{T,\delta}$ with probability at least $1-\delta/2$.

(I) We first bound the term $\sum_{t = 1}^T  (t+t_0)(t+t_0-1)\xi_t$ in (\ref{zehkojcuui}). Since $\mathbb{E}_{j_t} [\xi_t] = 0$, so $\{ \xi_t \} $ is a martingale difference sequence.
According  to the H\"older smoothness, we have the following inequality for all $\mathbf{w} \in \mathcal{W}$ and any $z \in \mathcal{Z}$
\begin{align}\label{vgrgrgtt}
\| \nabla f(\mathbf{w};z) \| &\leq \| \nabla f(\mathbf{0};z)  \|+ P \| \mathbf{w} \|^{\alpha} \leq\sup_{z \in \mathcal{Z}} \| \nabla f(\mathbf{0};z)  \| +  P \| \mathbf{w} \|^{\alpha}.
\end{align}
Thus the following inequality holds uniformly for all $t = 1,...,T$
\begin{align*}
&(t+t_0)(t+t_0-1)|\xi_t|  \leq 2\mu_S^{-1} (t+t_0-1) \|  \nabla F_S(\mathbf{w}_{t}) - \nabla f(\mathbf{w}_{t}; z_{j_t})  \|  \|  \nabla F_S(\mathbf{w}_{t}) \| \\
&\leq 4 \mu_S^{-1} (T+t_0-1) \sup_{z \in \mathcal{Z}}\| \nabla f(\mathbf{w}_{t}; z) \|^2  \leq 4 \mu_S^{-1} (T+t_0-1) \left(\sup_{z \in \mathcal{Z}}\| \nabla f(\mathbf{0};z)  \|+ P \| \mathbf{w}_t \|^{\alpha} \right)^2 \\
&\leq 4 \mu_S^{-1} (T+t_0-1) \left(P C_{T,\delta}^{\alpha} + \sup_{z \in \mathcal{Z}} \| \nabla f(\mathbf{0};z)  \| \right)^2,
\end{align*}
where the first inequality follows from Schwarz's inequality and $\eta_t = \frac{2}{\mu_S (t+t_0)}$, the third follows from (\ref{vgrgrgtt}), and the last inequality follows from the fact that $\|  \mathbf{w}_{t}\| \leq C_{T,\delta}$.
And by Assumption \ref{assu8}, we have 
\begin{align*}
&\mathbb{E}_{j_t} (t+t_0)^2(t+t_0-1)^2\xi_t^2\\\leq &4 \mu_S^{-2} (t+t_0-1)^2 \|  \nabla F_S(\mathbf{w}_{t}) \|^2 \mathbb{E}_{j_t} \|  \nabla F_S(\mathbf{w}_{t}) - \nabla f(\mathbf{w}_{t}; z_{j_t})  \|^2  \leq 4 \mu_S^{-2} (t+t_0-1)^2 \sigma^2 \|  \nabla F_S(\mathbf{w}_{t}) \|^2.
\end{align*}
Denote by $b:= \sup_{z \in \mathcal{Z}} \| \nabla f(\mathbf{0};z)  \|$. Applying part (b) of Lemma \ref{lemma36} with $\rho = \min \left \{ 1, (4 \sigma^2)^{-1}\left(P C_{T,\delta}^{\alpha} + b \right)^2 \right \}$, we have the following inequality with probability at least $1-\delta/4$, 
\begin{align*}
&\hphantom{{}={}}\sum_{t = 1}^T (t+t_0)(t+t_0-1) \xi_t\leq \frac{4 \rho \sum_{t = 1}^{T} (t+t_0-1)^2 \sigma^2 \|  \nabla F_S(\mathbf{w}_{t}) \|^2}{\mu_S^{2} 4 \mu_S^{-1} (T+t_0-1) \left(P C_{T,\delta}^{\alpha} + b \right)^2} + \frac{4  (T+t_0-1) \left(P C_{T,\delta}^{\alpha} + b \right)^2 \log(\frac{4}{\delta})}{\rho \mu_S} \\
&\leq \frac{ \rho \sum_{t = 1}^{T} (t+t_0-1) \sigma^2 \|  \nabla F_S(\mathbf{w}_{t}) \|^2}{\mu_S  \left(P C_{T,\delta}^{\alpha} + b \right)^2}  + \frac{4  (T+t_0-1) \left(P C_{T,\delta}^{\alpha} + b \right)^2 \log(\frac{4}{\delta})}{\rho \mu_S} \\
&\leq (4 \mu_S)^{-1}  \sum_{t = 1}^{T} (t+t_0-1) \|  \nabla F_S(\mathbf{w}_{t}) \|^2+ 4 \mu_S^{-1} (T+t_0-1) \log(\frac{4}{\delta}) \max \left \{ 4 \sigma^2 , \left(P C_{T,\delta}^{\alpha} + b\right)^2 \right \}.
\end{align*}

(II) We then focus on the term $P \left(\frac{2}{\mu_S}\right)^{1+\alpha} \sum_{t = 1}^T  \xi'_t (t+t_0)^{-\alpha}(t+t_0-1)$ in (\ref{zehkojcuui}). Since $\mathbb{E}_{j_t} [\xi_t'] = 0$, so $\{ \xi_t' \} $ is a martingale difference sequence. Firstly, we have 
\begin{align*}
&|\xi'_t| \leq  \| \nabla f(\mathbf{w}_{t}; z_{j_t})  - \nabla F_S(\mathbf{w}_{t}) \|^2 \leq 2 \| \nabla f(\mathbf{w}_{t};z_{j_t})\|^2  +2 \| \nabla F_S(\mathbf{w}_{t}) \|^2 \leq4 \left(P C_{T,\delta}^{\alpha} + b \right)^2,
\end{align*}
where the last inequality follows from (\ref{vgrgrgtt}).
Applying part (a) of Lemma \ref{lemma36}, we have the following inequality with probability at least $1- \delta/4$
\begin{align*}
&\sum_{t =1}^T (t+t_0)^{-\alpha}(t+t_0-1) \xi'_t \leq 4\left(P C_{T,\delta}^{\alpha} + b \right)^2 \left (2\sum_{t =1}^T (t+t_0)^{-2\alpha}(t+t_0-1)^2 \log \frac{4}{\delta} \right)^{1/2}.
\end{align*}
Moreover,
 it is clear that 
\begin{align*}
\sum_{t = 1}^T (t+t_0)^{-2\alpha}(t+t_0-1)^2 &\leq \sum_{t = 1}^T (t+t_0)^{2-2\alpha} \leq  \int_{1}^{T} (t+t_0)^{2-2\alpha}dt + (1+t_0)^{2-2\alpha} \\&\leq \frac{(T+t_0)^{3-2\alpha}}{3-2\alpha}- \frac{(1+t_0)^{3-2\alpha}}{3-2\alpha} + (1+t_0)^{2-2\alpha} = \mathcal{O}(T^{3-2\alpha}).
\end{align*}
Therefore, we have the following result with probability at least $1- \delta/4$
\begin{align*}
\sum_{t =1}^T (t+t_0)^{-\alpha}(t+t_0-1) \xi'_t = \mathcal{O}\left( C_{T,\delta}^{2\alpha} T^{(\frac{3}{2}-\alpha)} \log^{\frac{1}{2}} \frac{1}{\delta} \right).
\end{align*}

\medskip\noindent\textbf{Step 4: Control of remaining terms.}
For the remaining term $\sum_{t = 1}^T C_1 \left(\frac{2}{\mu_S}\right)^{1+\alpha}  (t+t_0)^{-\alpha}(t+t_0-1)$ in (\ref{zehkojcuui}), we have
\begin{align*}
\sum_{t = 1}^T (t+t_0)^{-\alpha}(t+t_0-1) &\leq \sum_{t = 1}^T (t+t_0)^{1-\alpha}  \leq  \int_{1}^{T} (t+t_0)^{1-\alpha}dt+ (1+t_0)^{1-\alpha} \\&\leq \frac{(T+t_0)^{2-\alpha}}{2-\alpha}-\frac{(1+t_0)^{2-\alpha}}{2-\alpha} + (1+t_0)^{1-\alpha}  = \mathcal{O}(T^{2-\alpha}).
\end{align*}

\medskip\noindent\textbf{Step 5: Final bound.}
Substituting these bounds in Step 3-4 into (\ref{zehkojcuui}), we finally have the following inequality with probability at least $1-\delta$
\begin{align*}
& (T+t_0)(T+t_0-1)[F_S(\mathbf{w}_{T+1}) - F_S^{\ast} ]   \leq -\sum_{t = 1}^T \frac{t+t_0-1}{4 \mu_S } \| \nabla F_S(\mathbf{w}_{t}) \|^2 \\
&+ 4 \mu_S^{-1} (T+t_0-1) \log \left(\frac{4}{\delta} \right) \max \left \{ 4 \sigma^2 , \left(P C_{T,\delta}^{\alpha} + b\right)^2 \right \}  + \mathcal{O}\left( C_{T,\delta}^{2\alpha} T^{3/2-\alpha} \log^{\frac{1}{2}} \frac{1}{\delta} \right)  \\&+ \mathcal{O}(T^{2-\alpha})  + (t_0-1)t_0( F_S(\mathbf{w}_{1}) - F_S^{\ast}),
\end{align*}
which implies that 
\begin{align*}
F_S(\mathbf{w}_{T+1}) - F_S^{\ast} = 
\begin{cases}
             \mathcal{O}\left( \frac{1}{T^{\alpha}}  \right) &\quad \text{if   } \alpha  \in (0,1 ),  \\ 
             \mathcal{O}\left (\frac{\log(T) \log^3(\frac{1}{\delta})}{T} \right)  &\quad \text{if   }\alpha = 1.
\end{cases}
\end{align*}
The proof is complete.
\end{proof}

\subsection{Proof of Theorem \ref{theo67}}\label{proof-sec-3.4}
\begin{proof}
By Lemma \ref{erfgefge}, if Assumptions \ref{assu4} and \ref{assu5} hold and $F$ satisfies Assumption \ref{assu10} with parameter $\mu$, when $n \geq \frac{c\beta^2(d+ \log(\frac{8 \log(2n R +2)}{\delta}))}{\mu^2}$, we have the following inequality with probability at least $1- \delta$
\begin{align*}
&\|  \nabla F(\mathbf{w})  \| 
\leq 2\left\| \nabla F_S(\mathbf{w}) \right\| +  \frac{\mu}{n} + 2\frac{B_{\ast}\log(4/\delta)}{n} + 2\sqrt{\frac{2 \mathbb{E} [ \| \nabla f(\mathbf{w}^{\ast};z) \|^2 ] \log(4/\delta)}{n}},
\end{align*}
which implies that with probability at least $1- \delta/2$ 
\begin{align}\label{theo111111} \left(\sum_{t = 1}^T \eta_t \right)^{-1} \sum_{t = 1}^T \eta_t \| \nabla F(\mathbf{w}_t) \|^2  \leq 16\left(\sum_{t = 1}^T \eta_t \right)^{-1} \sum_{t = 1}^T \eta_t\left\| \nabla F_S(\mathbf{w}_t) \right\|^2 +  \frac{4\mu^2}{n^2}+ \frac{16B_{\ast}^2\log^2(\frac{8}{\delta})}{n^2} + \frac{32 \mathbb{E} [ \| \nabla f(\mathbf{w}^{\ast};z) \|^2 ] \log(\frac{8}{\delta})}{n}.
\end{align}
Then by Lemma \ref{lemma 43}, if Assumptions \ref{assu4} (set $\alpha = 1$ in Assumption \ref{assu4-5}), \ref{assu7} and \ref{assu8} hold and when $\eta_t = \eta_1 t^{- 1/2}$ with $\eta_1 \leq \frac{1}{2\beta}$,  we obtain the following inequality with probability at least $1-\delta/2$,
\begin{align}\label{theo1112111}
&\left(\sum_{t = 1}^T \eta_t \right)^{-1} \sum_{t = 1}^T \eta_t\left\| \nabla F_S(\mathbf{w}_t) \right\|^2  \leq \left(\sum_{t = 1}^T \eta_t \right)^{-1}\mathcal{O} \left(\sum_{t=1}^T \eta_t^2 + \log \left(\frac{1}{\delta} \right) \right).
\end{align}
Combining (\ref{theo111111}) and (\ref{theo1112111}), we derive that with probability at least $1-\delta$ 
\begin{align*}
\left(\sum_{t = 1}^T \eta_t \right)^{-1} \sum_{t = 1}^T \eta_t \| \nabla F(\mathbf{w}_t) \|^2 \leq \left(\sum_{t = 1}^T \eta_t \right)^{-1}\mathcal{O} \left(\sum_{t=1}^T \eta_t^2 + \log \left(\frac{1}{\delta} \right) \right)  + \mathcal{O} \left( \frac{\log^2(1/\delta)}{n^2} + \frac{ \mathbb{E} [ \| \nabla f(\mathbf{w}^{\ast};z) \|^2 ] \log(1/\delta)}{n} \right).
\end{align*}
Using Lemma \ref{lei32} with $\eta_t = \eta_1 t^{- 1/2}$,
we finally obtain the following inequality with probability at least $1-\delta$,
\begin{align}\label{eq-withgradient}
\left(\sum_{t = 1}^T \eta_t \right)^{-1} \sum_{t = 1}^T \eta_t \| \nabla F(\mathbf{w}_t) \|^2= 
             \mathcal{O} \left( \frac{\log^2(\frac{1}{\delta})}{n^2} + \frac{ \mathbb{E} [ \| \nabla f(\mathbf{w}^{\ast};z) \|^2 ]\log(\frac{1}{\delta})}{n} \right)  + \mathcal{O}\left(\log(\frac{T}{\delta}) T^{-\frac{1}{2}}\right).
\end{align}

According to \cite{srebro2010optimistic}, there holds the following property for smooth functions
\begin{align}\label{smooth_property}
\frac{1}{2\beta} \|\nabla f(\mathbf{w})\|^2 \le f(\mathbf{w}) - \inf_{\mathbf{w}} f(\mathbf{w}).
\end{align}
When $f$ is nonnegative and $\beta$-smooth, from (\ref{smooth_property}), we have 
\begin{align*}
\| \nabla f(\mathbf{w}^{\ast};z) \|^2 \leq 2 \beta f(\mathbf{w}^{\ast};z),
\end{align*}
thus  
\begin{align}\label{smoothco}
\mathbb{E} [ \| \nabla f(\mathbf{w}^{\ast};z) \|^2 ] \leq 2 \beta \mathbb{E}  f(\mathbf{w}^{\ast};z) = 2 \beta  F(\mathbf{w}^{\ast}).
\end{align}
Inequality \eqref{smoothco} implies that (\ref{eq-withgradient}) becomes the following inequality 
\begin{align*}
\left(\sum_{t = 1}^T \eta_t \right)^{-1} \sum_{t = 1}^T \eta_t \| \nabla F(\mathbf{w}_t) \|^2= 
             \mathcal{O} \left( \frac{\log^2(\frac{1}{\delta})}{n^2} + \frac{ F(\mathbf{w}^{\ast})\log(\frac{1}{\delta})}{n} \right)  + \mathcal{O}\left(\log(\frac{T}{\delta}) T^{-\frac{1}{2}}\right).
\end{align*}

When $F$ satisfies the PL condition with parameter $\mu$, we have
\begin{align*}
F(\mathbf{w}) - F^{\ast}  \leq \frac{\left\| \nabla F(\mathbf{w}) \right\|^2 }{2\mu}, \quad \forall  \mathbf{w} \in \mathcal{W}.
\end{align*} 
Thus, selecting $T \asymp n^4$, we obtain the following result with probability at least $1-\delta$
\begin{align*}
&\left(\sum_{t = 1}^T \eta_t \right)^{-1} \sum_{t = 1}^T \eta_t F(\mathbf{w}_{t}) - F^{\ast}=
\mathcal{O} \left(\frac{\log^2(\frac{1}{\delta})}{n^2} + \frac{ F(\mathbf{w}^{\ast}) \log(\frac{1}{\delta})}{n}\right).
\end{align*}
The proof is complete.
\end{proof}
\subsection{Proof of Theorem \ref{theo7}}\label{proof-sec-3.5}
\begin{proof}
Since $F$ satisfies the PL condition with parameter $2\mu$, we have
\begin{align}\label{pll1344}
F(\mathbf{w}) - F^{\ast}  \leq \frac{\left\| \nabla F(\mathbf{w}) \right\|^2 }{4 \mu}, \quad \forall  \mathbf{w} \in \mathcal{W}.
\end{align} 
To bound $F(\mathbf{w}_{T+1}) - F^{\ast}$, it suffices to bound the term $\left\| \nabla F(\mathbf{w}_{T+1}) \right\|^2$.
By the triangle inequality of the norm
\begin{align}\label{pll1111}
&\left\| \nabla F(\mathbf{w}_{T+1}) \right\|^2 \leq  2 \left\| \nabla F(\mathbf{w}_{T+1})- \nabla F_S(\mathbf{w}_{T+1}) \right\|^2 + 2 \| \nabla F_S(\mathbf{w}_{T+1}) \|^2.
\end{align}
By \eqref{smooth_property}, when $f$ is smooth, we have
\begin{align}\label{pltogradientwithfs}
    \| \nabla F_S(\mathbf{w}_{T+1})  \|^2 \le 2\beta (F_S(\mathbf{w}_{T+1}) - F_S^\ast).
\end{align}
 By Lemma \ref{lemmajj}, if Assumptions \ref{assu4} (set $\alpha = 1$ in Assumption \ref{assu4-5}), \ref{assu7} and \ref{assu8} hold and $F_S$ satisfies the PL condition, we know that with probability at least $1-\delta/2$, the following inequality holds
\begin{align}\label{last008111}
\| \nabla F_S(\mathbf{w}_{T+1})  \|^2 \le 2\beta (F_S(\mathbf{w}_{T+1}) - F_S^\ast)= \mathcal{O} \left(\frac{\log T \log^3(1/\delta)}{T} \right).
\end{align}
Further by Lemma \ref{erfgefge}, if Assumptions \ref{assu4} and \ref{assu5} hold and $F$ satisfies the PL condition,
 when $n \geq \frac{c\beta^2(d+ \log(\frac{16 \log(2n R +2)}{\delta}))}{\mu^2}$, with probability at least $1- \delta/2$, the following inequality holds
\begin{align*}
&\left\| \nabla F (\mathbf{w}_{T+1} )- \nabla F_S(\mathbf{w}_{T+1}) \right\| \leq \left\| \nabla F_S(\mathbf{w}_{T+1}) \right\| +  \frac{2\mu}{n} + 2\frac{B_{\ast}\log(8/\delta)}{n} + 2\sqrt{\frac{2 \mathbb{E} [ \| \nabla f(\mathbf{w}^{\ast};z) \|^2 ] \log(8/\delta)}{n}}\\
&\leq \left\| \nabla F_S(\mathbf{w}_{T+1}) \right\| +  \frac{2\mu}{n} + 2\frac{B_{\ast}\log(8/\delta)}{n} + 2\sqrt{\frac{8 \beta F(\mathbf{w}^{\ast}) \log(8/\delta)}{n}}, 
\end{align*}
where the last inequality follows from  \eqref{smoothco}.
Together with (\ref{last008111}), this inequality implies that with probability at least $1- \delta$, there holds 
\begin{align}\label{last007111} 
&\left\| \nabla F(\mathbf{w}_{T+1})- \nabla F_S(\mathbf{w}_{T+1}) \right\|^2 
 = \mathcal{O} \left(\frac{\log T \log^3(\frac{1}{\delta})}{T} \right) + \mathcal{O} \left(\frac{\log^2(\frac{1}{\delta})}{n^2} + \frac{ F(\mathbf{w}^{\ast}) \log(\frac{1}{\delta})}{n}\right).
\end{align}
Substituting (\ref{last007111}) and (\ref{last008111}) into (\ref{pll1111}),
we have the following inequality with probability at least $1-\delta$
\begin{align}\label{lzj1111} 
&\left\| \nabla F(\mathbf{w}_{T+1}) \right\|^2  =  \mathcal{O} \left(\frac{\log T \log^3(\frac{1}{\delta})}{T}\right) + \mathcal{O} \left(\frac{\log^2(\frac{1}{\delta})}{n^2} + \frac{ F(\mathbf{w}^{\ast}) \log(\frac{1}{\delta})}{n}\right).
\end{align}
Then substituting (\ref{lzj1111}) into (\ref{pll1344}) and selecting $T \asymp n^2$,
we obtain the following inequality with probability at least $1-\delta$
\begin{align*}
F(\mathbf{w}_{T+1}) - F^{\ast}  = \mathcal{O} \left(\frac{\log n \log^3(1/\delta)}{n^2} + \frac{ F(\mathbf{w}^{\ast}) \log(\frac{1}{\delta})}{n}\right).
\end{align*}
The proof is complete.
\end{proof}

\section{Proofs for Nesterov Accelerated Gradient}\label{section56}
Section \ref{proof-of-nag} establishes optimization error bounds for NAG with smooth functions, organized with three Lemmas: bound on gradient norm (Lemma \ref{thm:nag-noncvx}), bound on iteration norm (Lemma \ref{thm:nag-generalization}), and bound on optimization error with an additional PL condition (Lemma \ref{thm:nag-corrected-final}). These Lemmas will serve as fundamental results on the derivation of generalization bound. The following Section \ref{section-proof-nag}, \ref{setion-proof-4.4}, and \ref{section-proof-4.5bag} prove Theorems \ref{theo6nag}, \ref{theo67nag}, and  \ref{theo7nag}, respectively.  

\subsection{Optimization Error of NAG}\label{proof-of-nag}
\begin{lemma}\label{thm:nag-noncvx}
Suppose Assumptions \ref{assu4}, \ref{assu7} and \ref{assu8} hold.
Let $\{ \mathbf{w}_t\}_t$ be the sequence produced by NAG, i.e. (\ref{eq:nag}), with $\eta_t$ such that for all $t \in \mathbb{N}$ 
\begin{equation}\label{eq:c-small}
\eta_t \le c := \min\!\left\{\frac{1-\gamma}{2\sqrt{2}\,\gamma \beta},\ 
\frac{(1-\gamma)^2}{32\,C_m(\gamma,\beta)}\right\}.
\end{equation}
Then, for any $\delta > 0$, with probability $1 - \delta$, we have
\begin{align*}
\sum_{k=1}^t  \eta_k \| \nabla F_S(\mathbf{w}_{k}) \|^2 \leq 8\Delta_1
+ \frac{32C_m(\gamma,\beta)}{(1-\gamma)^2} \left(\sigma^2\sum_{k=1}^t \eta_k^{2}  + 8G^2 \log(2/\delta) + G^2 \sum_{k=1}^t \eta_k^{2} \right)
+ \frac{16\frac{1 }{1-\gamma}G^2 \log(2/\delta)}{\min\{1,\ \frac{(1-\gamma)G^2}{4c \sigma^2 }\}},
\end{align*}
where
\[
\Delta_1:=F_S(\mathbf{w}_1)-F_S(\mathbf{w}_{t+1}),\qquad
C_m(\gamma,\beta):=\frac{1}{1-\gamma}(\beta\gamma+ \frac{\beta\gamma(1-\gamma)}{4\sqrt{2}})+\frac{\beta}{2}.
\]
\end{lemma}

\begin{proof} The proof proceeds with four steps.

\medskip\noindent\textbf{Step 1: A decomposition under smoothness.}
Since function $f$ satisfies Assumption \ref{assu4}, it is easy to verify that $F_S$ also satisfies Assumption \ref{assu4}:
\begin{align*}
\| \nabla F_S(\mathbf{w}_1) - \nabla F_S(\mathbf{w}_2) \|  = \left \|  \frac{1}{n} \sum_{i=1}^n (\nabla f(\mathbf{w}_1;z_i) -   \nabla f(\mathbf{w}_2;z_i) ) \right \| \leq \frac{1}{n}\sum_{i=1}^n \|\nabla f(\mathbf{w}_1;z_i)  - \nabla f(\mathbf{w}_2;z_i) \|  \leq \beta\| \mathbf{w}_1 -  \mathbf{w}_2 \|.
\end{align*}
With the $\beta$-smoothness of $F_S$ and the update of NAG: $\mathbf{w}_{t+1} \;=\; \mathbf{w}_t + \mathbf{m}_{t+1}$,
\begin{equation}\label{eq:smooth-descent}
F_S(\mathbf{w}_{t+1})-F_S(\mathbf{w}_t)
\le
\langle \nabla F_S(\mathbf{w}_t) , \mathbf{m}_{t+1}\rangle+\frac{\beta}{2}\|\mathbf{m}_{t+1}\|^2.
\end{equation}
Using NAG's iteration $\mathbf{m}_{t+1} \;=\; \gamma \mathbf{m}_t - \eta_t\, \mathbf{g}_t$ and $\xi_t=\nabla f(\mathbf{y}_t;z_{j_t})-\nabla F_S(\mathbf{y}_t)$, expand
\begin{equation}\label{eq:inner-expand}
\langle\nabla F_S(\mathbf{w}_t), \mathbf{m}_{t+1}\rangle
=\gamma \langle \nabla F_S(\mathbf{w}_t), \mathbf{m}_t\rangle
-\eta_t \langle \nabla F_S(\mathbf{w}_t), \nabla F_S(\mathbf{y}_t)\rangle
-\eta_t \langle \nabla F_S(\mathbf{w}_t), \xi_t \rangle.
\end{equation}
For the middle product, by polarization and smoothness,
\begin{align}
- \langle \nabla F_S(\mathbf{w}_t),\nabla F_S(\mathbf{y}_t)\rangle
&= -\tfrac12\|\nabla F_S(\mathbf{w}_t)\|^2 - \tfrac12\|\nabla F_S(\mathbf{y}_t)\|^2
+\tfrac12\|\nabla F_S(\mathbf{y}_t)-\nabla F_S(\mathbf{w}_t)\|^2 \nonumber\\
&\le -\tfrac12\|\nabla F_S(\mathbf{w}_t)\|^2 + \tfrac12 \beta^2 \|\mathbf{y}_t-\mathbf{w}_t\|^2
= -\tfrac12\|\nabla F_S(\mathbf{w}_t)\|^2 + \tfrac12 \beta^2 \gamma^2 \|\mathbf{m}_t\|^2,
\label{eq:polar1111}
\end{align}
where in the inequality we have used $\mathbf{y}_t \;=\; \mathbf{w}_t + \gamma \mathbf{m}_t$.
Next, relate $\langle \nabla F_S(\mathbf{w}_t), \mathbf{m}_t\rangle$ to time $t-1$:
\begin{align}
\langle \nabla F_S(\mathbf{w}_t), \mathbf{m}_t\rangle
&= \langle \nabla F_S(\mathbf{w}_{t-1}), \mathbf{m}_t\rangle
+\langle \nabla F_S(\mathbf{w}_t)-\nabla F_S(\mathbf{w}_{t-1}), \mathbf{m}_t\rangle \le \langle \nabla F_S(\mathbf{w}_{t-1}), \mathbf{m}_t\rangle + \beta\|\mathbf{m}_t\|^2.
\label{eq:grad-shift}
\end{align}
Plugging \eqref{eq:polar1111}–\eqref{eq:grad-shift} into \eqref{eq:inner-expand} yields
\begin{equation}\label{eq:inner-rec}
\langle\nabla F_S(\mathbf{w}_t), \mathbf{m}_{t+1}\rangle
\le
\gamma \langle \nabla F_S(\mathbf{w}_{t-1}), \mathbf{m}_t\rangle
+ \Bigl(\beta\gamma + \tfrac12 \beta^2\gamma^2 \eta_t\Bigr)\|\mathbf{m}_t\|^2
-\tfrac12 \eta_t \|\nabla F_S(\mathbf{w}_t)\|^2
-\eta_t \langle \nabla F_S(\mathbf{w}_t), \xi_t \rangle.
\end{equation}
Unroll \eqref{eq:inner-rec} backwards in $t$ and use $\mathbf{m}_{1} = \mathbf{0}$:
\[
\langle\nabla F_S(\mathbf{w}_t), \mathbf{m}_{t+1}\rangle
\le
-\tfrac12 \sum_{i=1}^t \gamma^{t-i}\eta_i \|\nabla F_S(\mathbf{w}_i)\|^2
+\sum_{i=1}^t \gamma^{t-i}\Bigl(\beta\gamma + \tfrac12 \beta^2\gamma^2 \eta_i\Bigr)\|\mathbf{m}_i\|^2
-\sum_{i=1}^t \gamma^{t-i}\eta_i \langle \nabla F_S(\mathbf{w}_i), \xi_i \rangle.
\]
Plugging this inequality into \eqref{eq:smooth-descent} and taking a summation of \eqref{eq:smooth-descent} gives 
\begin{align}\label{eq-deco-weights}\nonumber
&F_S(\mathbf{w}_{t+1}) = F_S(\mathbf{w}_{1}) + \sum_{k=1}^t(F_S(\mathbf{w}_{k+1}) -  F_S(\mathbf{w}_{k}))  \leq F_S(\mathbf{w}_{1})  - \tfrac12\sum_{k=1}^t  \sum_{i=1}^k \gamma^{k-i}\eta_i \|\nabla F_S(\mathbf{w}_i)\|^2
\\&+ \sum_{k=1}^t \sum_{i=1}^k \gamma^{k-i}\Bigl(\beta\gamma + \tfrac12 \beta^2\gamma^2 \eta_i\Bigr)\|\mathbf{m}_i\|^2
-\sum_{k=1}^t\sum_{i=1}^k \gamma^{k-i}\eta_i \langle \nabla F_S(\mathbf{w}_i), \xi_i \rangle + \frac{\beta}{2}\sum_{k=1}^t \|\mathbf{m}_{k+1}\|^2.
\end{align}
Reordering the geometric weights in (\ref{eq-deco-weights}) by Lemma \ref{lem:reorder},
\begin{align}
F_S(\mathbf{w}_{t+1})-F_S(\mathbf{w}_1)
&\le -\tfrac12 \sum_{k=1}^t \Bigl(\sum_{i=k}^t \gamma^{i-k}\Bigr)\eta_k \|\nabla F_S(\mathbf{w}_k)\|^2
+ \sum_{k=1}^t \Bigl(\sum_{i=k}^t \gamma^{i-k}\Bigr)\Bigl(\beta\gamma + \tfrac12 \beta^2\gamma^2 \eta_k\Bigr)\|\mathbf{m}_k\|^2 \nonumber\\
&\quad - \sum_{k=1}^t \Bigl(\sum_{i=k}^t \gamma^{i-k}\Bigr)\eta_k \langle \nabla F_S(\mathbf{w}_k), \xi_k \rangle
+ \frac{\beta}{2}\sum_{k=1}^t \|\mathbf{m}_{k+1}\|^2. \label{eq:sum-main}
\end{align}
Define  $w_k :=  \sum_{i=k}^t \gamma^{i-k}=  (1-\gamma^{t-k+1})/(1-\gamma)\in (1,(1-\gamma)^{-1})$, the inequality (\ref{eq:sum-main}) gives
\begin{align}\label{eq:53}\nonumber
&F_S(\mathbf{w}_{t+1})-F_S(\mathbf{w}_1)
\\\le& -\frac12\sum_{k=1}^t w_k\eta_k\|\nabla F_S(\mathbf{w}_k)\|^2
+\underbrace{\sum_{k=1}^t w_k\!\left(\beta\gamma+\tfrac12 \beta^2\gamma^2\eta_k\right)\!\|\mathbf{m}_k\|^2}_{\text{(A)}}
+\underbrace{\frac \beta 2\sum_{k=1}^t \|\mathbf{m}_{k+1}\|^2}_{\text{(B)}}
-\sum_{k=1}^t w_k\eta_k\langle \nabla F_S(\mathbf{w}_k),\xi_k\rangle .
\end{align}

\medskip\noindent\textbf{Step 2: Bounding the momentum terms (A) and (B).}
Since $w_k\le (1-\gamma)^{-1}$ and $\eta_k \le \frac{1-\gamma}{2\sqrt{2}\,\gamma \beta}$, we have for every $t$,
\[
w_k\!\left(\beta\gamma+\tfrac12 \beta^2\gamma^2\eta_t\right)\le \frac{1}{1-\gamma}(\beta\gamma+ \frac{\beta\gamma(1-\gamma)}{4\sqrt{2}}) .
\]
Therefore
\[
\text{(A)}\ \le\ \frac{1}{1-\gamma}(\beta\gamma+ \frac{\beta\gamma(1-\gamma)}{4\sqrt{2}}) \sum_{k=1}^t \|\mathbf{m}_k\|^2 \le \frac{1}{1-\gamma}(\beta\gamma+ \frac{\beta\gamma(1-\gamma)}{4\sqrt{2}})\sum_{k=1}^t \|\mathbf{m}_{k+1}\|^2.
\]
Keep (B) as is and we get
\begin{align}\label{eq:absorption}
\text{(A)} + \text{(B)} \le \left(\frac{1}{1-\gamma}(\beta\gamma+ \frac{\beta\gamma(1-\gamma)}{4\sqrt{2}}) + \frac{\beta}{2}\right) \sum_{k=1}^t \|\mathbf{m}_{k+1}\|^2 = C_m(\gamma,\beta) \sum_{k=1}^t \|\mathbf{m}_{k+1}\|^2.
\end{align}
The next step is to bound $\sum_{k=1}^t \|\mathbf{m}_{k+1}\|^2$. From the update $\mathbf{m}_{k+1}=\gamma \mathbf{m}_k - \eta_k \nabla f(\mathbf{y}_k;z_{j_k})$, write it as a convex combination:
\[
\mathbf{m}_{k+1}=\gamma \mathbf{m}_k + (1-\gamma)\Bigl(-\frac{\eta_k}{1-\gamma} \nabla f(\mathbf{y}_k;z_{j_k})\Bigr).
\]
By convexity of $\|\cdot\|^2$ and $\mathbf{m}_1= \mathbf{0}$, a standard telescoping gives
\begin{equation}\label{eq:mtplus1}
\|\mathbf{m}_{k+1}\|^2 \le \gamma \|\mathbf{m}_k\|^2 + \frac{\eta_k^2}{1-\gamma}\,\|\nabla f(\mathbf{y}_k;z_{j_k})\|^2 = \frac{1}{1-\gamma}\sum_{i=1}^k \gamma^{k-i} \|  \eta_i\nabla f(\mathbf{w}_i;z_{j_i}) \|^2.
\end{equation}
Summing \eqref{eq:mtplus1} over $k=1,\dots,t$, and reordering the geometric weights by Lemma \ref{lem:reorder} yields
\begin{align}\label{eq:sum-mt}\nonumber
\sum_{k=1}^t \|\mathbf{m}_{k+1}\|^2 \;&\le\; \frac{1}{1-\gamma} \sum_{k=1}^t\sum_{i=1}^k \gamma^{k-i} \|  \eta_i\nabla f(\mathbf{w}_i;z_{j_i}) \|^2 \\&\le \frac{1}{1-\gamma}\sum_{k=1}^t w_k\eta_k^2 \|\nabla f(\mathbf{y}_k;z_{j_k})\|^2 \le \frac{1}{(1-\gamma)^2}\sum_{k=1}^t \eta_k^2 \|\nabla f(\mathbf{y}_k;z_{j_k})\|^2,
\end{align}
where we have also used $w_k\le (1-\gamma)^{-1}$.
Decompose 
\begin{align}
\|\nabla f(\mathbf{y}_k;z_{j_k})\|^2 \le 2\|\nabla F_S(\mathbf{y}_k)\|^2 + 2\|\xi_k\|^2 \le 4\|\nabla F_S(\mathbf{w}_k)\|^2 + 4\beta^2\gamma^2 \|\mathbf{m}_k\|^2 + 2\|\xi_k\|^2,
\label{eq:gt2}
\end{align}
where we have used $\|\nabla F_S(\mathbf{y}_k)-\nabla F_S(\mathbf{w}_k)\|\le \beta\|\mathbf{y}_k-\mathbf{w}_k\|=\beta\gamma\|\mathbf{m}_k\|$ and
$\|a+b\|^2\le 2\|a\|^2+2\|b\|^2$.
Plug \eqref{eq:gt2} into \eqref{eq:sum-mt} 
and absorb the $\|\mathbf{m}_k\|^2$ term to the left to obtain
\begin{align*}
\sum_{k=1}^t \|\mathbf{m}_{k+1}\|^2 -  \frac{ 4\beta^2\gamma^2 }{(1-\gamma)^2}\sum_{k=1}^t \eta_k^2  \|\mathbf{m}_k\|^2   \;\le\; \frac{1}{(1-\gamma)^2}\sum_{k=1}^t \eta_k^2 (4\|\nabla F_S(\mathbf{w}_k)\|^2  + 2\|\xi_k\|^2)
\end{align*}
To ensure the absorption, using $\eta_k \le \frac{1-\gamma}{2\sqrt{2}\,\gamma \beta}$ gives
\begin{equation}\label{eq:sum-mt-final}
 \sum_{k=1}^t \|\mathbf{m}_{k+1}\|^2 \le\; \frac{4}{(1-\gamma)^2}\sum_{k=1}^t \eta_k^2 \|\xi_k\|^2
\;+\; \frac{8}{(1-\gamma)^2}\sum_{k=1}^t \eta_k^2 \|\nabla F_S(\mathbf{w}_k)\|^2.
\end{equation}
Plugging (\ref{eq:sum-mt-final}) into (\ref{eq:absorption})  gives
\begin{equation}\label{eq:sum-mt-finall}
(A) + (B)\;\le\; \frac{4C_m(\gamma,\beta)}{(1-\gamma)^2}\sum_{k=1}^t \eta_k^2 \|\xi_k\|^2
\;+\; \frac{8C_m(\gamma,\beta)}{(1-\gamma)^2}\sum_{k=1}^t \eta_k^2 \|\nabla F_S(\mathbf{w}_k)\|^2.
\end{equation}
Combining \eqref{eq:sum-mt-finall} and \eqref{eq:53}, 
we get
\begin{align}\label{eq:key-ineq}
\sum_{k=1}^t w_k\,\eta_k \|\nabla F_S(\mathbf{w}_k)\|^2
\le 2\Delta_1
+ \frac{8C_m(\gamma,\beta)}{(1-\gamma)^2}\sum_{k=1}^t \eta_k^2 \|\xi_k\|^2
+ \frac{16C_m(\gamma,\beta)}{(1-\gamma)^2}\sum_{k=1}^t \eta_k^2 \|\nabla F_S(\mathbf{w}_k)\|^2-2\sum_{k=1}^t w_k\,\eta_k \langle \nabla F_S(\mathbf{w}_k), \xi_k \rangle.
\end{align}
By the second condition in \eqref{eq:c-small}, $\eta_k\le \frac{(1-\gamma)^2}{32\,C_m(\gamma,\beta)}$ and $w_t\ge 1$, the third term on the RHS can be absorbed into the LHS:
\begin{equation}\label{eq:absorb}
\frac12\sum_{k=1}^t  \,\eta_k \|\nabla F_S(\mathbf{w}_k)\|^2
\ \le\
2\Delta_1
+ \frac{8C_m(\gamma,\beta)}{(1-\gamma)^2}\sum_{k=1}^t \eta_k^2 \|\xi_k\|^2
-2\sum_{k=1}^t w_k\,\eta_k \langle \nabla F_S(\mathbf{w}_k), \xi_k \rangle.
\end{equation}

\medskip\noindent\textbf{Step 3: High-probability control of noise terms.} After obtaining the inequality (\ref{eq:absorb}), the next step is to bound $\sum_{k=1}^t \eta_k^2 \|\xi_k\|^2$ and $-\sum_{k=1}^t w_k\,\eta_k \langle \nabla F_S(\mathbf{w}_k), \xi_k \rangle$.

(I) 
Denote by \(
M_k= -\,w_k\,\eta_k \langle \nabla F_S(\mathbf{w}_k), \xi_k \rangle
\). Since $\mathbb{E}_{j_k}M_k = 0$, thus $\{ M_k\}$ is a martingale difference sequence. By Assumption \ref{assu7},
\begin{align} 
 |M_k|\ \le\ w_k\,\eta_k\,\|\nabla F_S(\mathbf{w}_k)\|\,\|\xi_k\|
\  \le w_k\,\sqrt{\eta_k}\,\|\nabla F_S(\mathbf{w}_k)\|\,(\sqrt{\eta_k}\|\nabla f(\mathbf{y}_k;z_{j_k})\| +\sqrt{\eta_k}\|\nabla F_S(\mathbf{y}_k)\| ) \le \ 2 (1-\gamma)^{-1} \,G^2 ,\label{eq:martinagl1}
\end{align}
where we have also used $w_k\le (1-\gamma)^{-1}$.
Moreover, we have
\begin{align}\nonumber \label{inequ767676}
&\hphantom{{}={}}\sum_{k=1}^t \mathbb{E}_{j_k} \left[ ( M_k  - \mathbb{E}_{j_k} M_k )^2 \right ] = \sum_{k=1}^t \mathbb{E}_{j_k}  ( M_k)^2 \leq \sum_{k=1}^t (\frac{1 }{1-\gamma})^2 \eta_k^2 \mathbb{E}_{j_k} \left (  \| \nabla F_S(\mathbf{y}_{k}) - \nabla f(\mathbf{y}_{k}; z_{j_k}) \|^2 \right) \| \nabla F_S(\mathbf{w}_{k}) \|^2 \\ 
&\leq \sigma^2 (\frac{1 }{1-\gamma})^2 \sum_{k=1}^t \eta_k^2  \| \nabla F_S(\mathbf{w}_{k}) \|^2\leq c\sigma^2 (\frac{1 }{1-\gamma})^2 \sum_{k=1}^t \eta_k  \| \nabla F_S(\mathbf{w}_{k}) \|^2,
\end{align}
where the second inequality follows from Assumption \ref{assu8} and the last inequality follows from the fact that $\eta_k \leq c$ for all $k \in \mathbb{N}$. Substituting \eqref{eq:martinagl1} and (\ref{inequ767676}) into part (b) of Lemma \ref{lemma36}, we have the following inequality with probability at least $1  - \delta/2$ 
\begin{align}  \label{yihaishihufgrgrg}
&\sum_{k=1}^t M_k \leq \frac{\rho c \sigma^2 \frac{1 }{1-\gamma} \sum_{k = 1}^t \eta_k \| \nabla F_S(\mathbf{w}_{k}) \|^2}{2G^2} + \frac{2\frac{1 }{1-\gamma}G^2 \log(2/\delta)}{\rho}.
\end{align}

(II) Next, for the quadratic variation term, a triangle inequality gives
\[
\sum_{k=1}^t \eta_k^2 \|\xi_k\|^2
= \sum_{k=1}^t \eta_k^2\,\mathbb{E}_{j_k}[\|\xi_k\|^2 ]\ +\ \sum_{k=1}^t \eta_k^2\bigl(\|\xi_k\|^2-\mathbb{E}_{j_k}[\|\xi_k\|^2] \bigr).
\]
By Assumption \ref{assu8}, the first sum is bounded by $\sigma^2\sum_{k=1}^t \eta_k^2$. 
Denote by $M_k' := \eta_k^2\bigl(\| \nabla f(\mathbf{y}_k;z_{j_k}) - \nabla F_S(\mathbf{y}_{k})\|^2- \mathbb{E}_{j_k}\| \nabla f(\mathbf{y}_k;z_{j_k}) - \nabla F_S(\mathbf{y}_{k})\|^2\bigr)$. Since  $\mathbb{E}_{j_k}M'_k = 0$, thus $\{ M_k'\}$ is a martingale difference sequence. Similarly, by Assumption \ref{assu7}, 
\begin{align}\label{inequ777754}
|M_k'| \leq 2 \eta_k^2\left (  \| \nabla F_S(\mathbf{y}_{k})\|^2 + \| \nabla f(\mathbf{y}_{k}; z_{j_k}) \|^2 \right) = 2 \eta_k \left (  \eta_k\| \nabla F_S(\mathbf{y}_{k})\|^2 + \eta_k\| \nabla f(\mathbf{y}_{k}; z_{j_k}) \|^2 \right)  \leq 4\eta_kG^2.
\end{align}
Substituting (\ref{inequ777754}) into part (a) of Lemma \ref{lemma36}, we have the following inequality with probability at least $1-\delta/2$
\begin{align}\label{eq-sum-marti-bound}
\sum_{k=1}^t   M_k'&\leq 4G^2 \left( 2\sum_{k=1}^t \eta_k^{2} \log\left(\frac{2}{\delta} \right)  \right)^{\frac{1}{2}}  \leq 8G^2 \log(2/\delta) + G^2 \sum_{k=1}^t \eta_k^{2},
\end{align}
where the last inequality follows from the Schwarz's inequality. The inequality (\ref{eq-sum-marti-bound}) implies that we have the following inequality with probability at least $1-\delta/2$
\begin{align}\label{eq:marrtingale}
\sum_{k=1}^{t} \eta_k^2\| \xi_k\|^2 \le \sum_{k=1}^t \eta_k^{2} \sigma^2 + 8G^2 \log(2/\delta) + G^2 \sum_{k=1}^t \eta_k^{2}.
\end{align}

\medskip\noindent\textbf{Step 4: Putting together and averaging.}
Insert \eqref{eq:marrtingale} and \eqref{yihaishihufgrgrg} into \eqref{eq:absorb}. Choosing
\(
\rho := \min\{1,\ \frac{(1-\gamma)G^2}{4c \sigma^2 }\}
\), we obtain the following inequality with probability at least $1-\delta$
\begin{align*}
&\frac{1}{2}\sum_{k=1}^t \,\eta_k \|\nabla F_S(\mathbf{w}_k)\|^2
\ \\\le&\
2\Delta_1
+ \frac{8C_m(\gamma,\beta)}{(1-\gamma)^2} (\sum_{k=1}^t \eta_k^{2} \sigma^2 + 8G^2 \log(2/\delta) + G^2 \sum_{k=1}^t \eta_k^{2})
+ \frac{1}{4}\sum_{t = 1}^t \eta_t \| \nabla F_S(\mathbf{w}_{t}) \|^2 + \frac{ 4\frac{1 }{1-\gamma}G^2 \log(2/\delta)}{\min\{1,\ \frac{(1-\gamma)G^2}{4c \sigma^2 }\}},
\end{align*}
which means, with probability $1-\delta$ we have the following inequality
\[
\sum_{k=1}^t\,\eta_k \|\nabla F_S(\mathbf{w}_k)\|^2
\ \le\
8\Delta_1
+ \frac{32C_m(\gamma,\beta)}{(1-\gamma)^2} (\sum_{k=1}^t \eta_k^{2} \sigma^2 + 8G^2 \log(2/\delta) + G^2 \sum_{k=1}^t \eta_k^{2})
+ \frac{16\frac{1 }{1-\gamma}G^2 \log(2/\delta)}{\min\{1,\ \frac{(1-\gamma)G^2}{4c \sigma^2 }\}},
\]
which is the claimed bound.
\end{proof}

\begin{lemma}
\label{thm:nag-generalization}
Suppose Assumptions  \ref{assu4}, \ref{assu7} and \ref{assu8} hold.
Let $\{ \mathbf{w}_t\}_t$ be the sequence produced by NAG, i.e. (\ref{eq:nag}), with $\eta_t$ such that for all $t \in \mathbb{N}$ 
\begin{equation}\label{eq:c-smallaa}
\eta_t \le c := \min\!\left\{\frac{1-\gamma}{2\sqrt{2}\,\gamma \beta},\ 
\frac{(1-\gamma)^2}{32\,C_m(\gamma,\beta)}\right\}.
\end{equation}
Then for any $\delta\in(0,1)$, we have the following inequality with probability $1-\delta$ uniformly for all $t = 1,...,T$ 
\begin{align*}
&\|\mathbf{w}_{t+1}\| \le \frac{2}{1-\gamma} \left( \frac{2G \sqrt{c}}{3} + \sigma \left(\sum_{i=1}^T\eta_i^2   \right)^{\frac{1}{2}}  \right)\log \frac{6}{\delta}  + \frac{( \sum_{i=1}^t \eta_i )^{1/2}}{1-\gamma}   \left (  2 \beta^2\gamma^2 \frac{4c}{(1-\gamma)^2}\left(\sum_{k=1}^t \eta_k^{2} \sigma^2 + 8G^2 \log(\frac{6}{\delta}) + G^2 \sum_{k=1}^t \eta_k^{2} \right) \right.\\& \left. +(2+2 \beta^2\gamma^2 c^2\frac{8}{(1-\gamma)^2})\Big( 8\Delta_1
+ \frac{32C_m(\gamma,\beta)}{(1-\gamma)^2} \left(\sigma^2\sum_{k=1}^t \eta_k^{2}  + 8G^2 \log(6/\delta) + G^2 \sum_{k=1}^t \eta_k^{2} \right)
+ \frac{16\frac{1 }{1-\gamma}G^2 \log(6/\delta)}{\min\{1,\ \frac{(1-\gamma)G^2}{4c \sigma^2 }\}}\Big) \right)^{1/2},
\end{align*}
where
\[
\Delta_1:=F_S(\mathbf{w}_1)-F_S(\mathbf{w}_{t+1}),\qquad
C_m(\gamma,\beta):=\frac{1}{1-\gamma}(\beta\gamma+ \frac{\beta\gamma(1-\gamma)}{4\sqrt{2}})+\frac{\beta}{2}.
\]
\end{lemma}

\begin{proof}
The proof proceeds with four steps.

\medskip\noindent\textbf{Step 1: A recursive expression of iteration with momentum.}
  From the update
\[
\mathbf{m}_{t+1}=\gamma \mathbf{m}_t - \eta_t \nabla f(\mathbf{y}_t;z_{j_t}),
\]
by induction and $\mathbf{m}_1 = \mathbf{0}$ we obtain
\[
\mathbf{m}_{t+1} \;=\; -\sum_{i=1}^t \gamma^{\,t-i}\,\eta_i\,\nabla f(\mathbf{y}_i;z_{j_i}).
\]
Further from the update  \[\mathbf{w}_{t+1} \;=\; \mathbf{w}_t + \mathbf{m}_{t+1} 
\] by induction and $\mathbf{w}_1 = \mathbf{0}$ we obtain
\begin{align*} 
\mathbf{w}_{t+1}
 = \mathbf{w}_1+\sum_{k=1}^t \mathbf{m}_{k+1}
=   -\sum_{k=1}^t \sum_{i=1}^k \gamma^{\,k-i}\,\eta_i\,\nabla f(\mathbf{y}_i;z_{j_i}) =  -\sum_{k=1}^t \sum_{i=k}^t \gamma^{\,i-k}\,\eta_k\,\nabla f(\mathbf{y}_k;z_{j_k})=   - \sum_{i=1}^t w_{t,i}\,\eta_i\,\nabla f(\mathbf{y}_i;z_{j_i}),\end{align*}
where $w_{t,i}:=\frac{1-\gamma^{\,t-i+1}}{1-\gamma}\in (1,(1-\gamma)^{-1})$, and where we have used Geometric reordering (Lemma \ref{lem:reorder}) in the third identity.
Thus, by the triangle inequality, we have the following norm bound
\begin{align}\label{bound-on-w}
&\|\mathbf{w}_{t+1}\|  \le \frac{1}{1-\gamma} \Big\| \sum_{i=1}^t   \eta_i\nabla f(\mathbf{y}_i;z_{j_i}) \Big\|   
\leq  \frac{1}{1-\gamma} \Big\| \sum_{i=1}^t  \eta_i (\nabla f(\mathbf{y}_i;z_{j_i}) -  \nabla F_S(\mathbf{y}_{i}))\Big\| + \frac{1}{1-\gamma}\Big\| \sum_{i=1}^t  \eta_i  \nabla F_S(\mathbf{y}_{i})\Big\|  .
\end{align}

\medskip\noindent\textbf{Step 2: High-probability control of noise terms.}
Let's consider the term $\|\sum_{i=1}^t  \eta_i(\nabla f(\mathbf{y}_i;z_{j_i}) -  \nabla F_S(\mathbf{y}_{i})) \|$.
We know $\{  \eta_i(\nabla f(\mathbf{y}_i;z_{j_i}) -  \nabla F_S(\mathbf{y}_{i})) \}$ is a martingale difference sequence, because $\mathbb{E}_{j_i}  [ \eta_i(\nabla f(\mathbf{y}_i;z_{j_i}) -  \nabla F_S(\mathbf{y}_{i}))] = 0$. Firstly,
\begin{align}\label{yihadfffo}
 & \|  \eta_i( \nabla f(\mathbf{y}_{i}; z_{j_i}) - \nabla F_S(\mathbf{y}_{i}) ) \| \leq  \sqrt{\eta_i} \left(2 \sqrt{\eta_i} \sup_{z\in \mathcal{Z}}\|\nabla f(\mathbf{y}_{i};z) \| \right ) \leq 2G \sqrt{\eta_i}\leq 2G \sqrt{c},
\end{align}
where the second inequality follows from Assumption \ref{assu7} and the last inequality follows from the fact that $\eta_t \leq c$ for all $t \in \mathbb{N}$. 
Secondly, according to Assumption \ref{assu8},  
\begin{align}\label{erffggghao}
\sum_{i=1}^t \mathbb{E}_{j_i}   \|   \eta_i(\nabla f(\mathbf{y}_i;z_{j_i}) -  \nabla F_S(\mathbf{y}_{i}))  \|^2  \leq\sum_{i=1}^t \eta_i^2 \sigma^2.
\end{align}
Substituting (\ref{yihadfffo}) and (\ref{erffggghao}) into Lemma \ref{lemma51}, with probability $1- \delta$ we have  the following inequality  uniformly for all $t = 1,...,T$
\begin{align}\label{bound-on-noise}
\max_{1\leq t \leq T} \Big\| \sum_{i=1}^t  \eta_i (\nabla f(\mathbf{y}_i;z_{j_i}) -  \nabla F_S(\mathbf{y}_{i}))\Big\| \leq 2\left( \frac{2G \sqrt{c}}{3} + \sigma \Big(\sum_{i=1}^T \eta_i^2   \Big)^{\frac{1}{2}}  \right)\log \frac{2}{\delta}.
\end{align}

\medskip\noindent\textbf{Step 3: Bounding the look-ahead first-order stationary point $\sum_{i=1}^t \eta_i \|\nabla F_S(\mathbf{y}_{i}) \|^2$.}
From Schwarz's inequality
\begin{align}\label{rxnbxm}
\Big\| \sum_{i=1}^t  \eta_i \nabla F_S(\mathbf{y}_{i}) \Big\|^2  \leq  \Big( \sum_{i=1}^t \eta_i \|\nabla F_S(\mathbf{y}_{i}) \|\Big)^2 \leq \Big( \sum_{i=1}^t \eta_i \Big) \Big( \sum_{i=1}^t \eta_i \|\nabla F_S(\mathbf{y}_{i}) \|^2 \Big),
\end{align}
which implies 
\begin{align}\label{rxnbxm1}
&\Big\| \sum_{i=1}^t   \eta_i \nabla F_S(\mathbf{y}_{i}) \Big\|  \leq \Big( \sum_{i=1}^t \eta_i \Big)^{1/2} \Big( \sum_{i=1}^t \eta_i \|\nabla F_S(\mathbf{y}_{i}) \|^2 \Big)^{1/2}.
\end{align}
Next we transform $\mathbf{y}_t$ to $\mathbf{w}_t$ via smoothness. The
$\beta$-smoothness (Assumption \ref{assu4}) implies
\[
\|\nabla F(\mathbf{y}_t)\| \;\le\; \|\nabla F(\mathbf{w}_t)\| + \beta\|\mathbf{y}_t-\mathbf{w}_t\|
\;=\; \|\nabla F(\mathbf{w}_t)\| + \beta\gamma \|\mathbf{m}_t\|.
\]
Using $(a+b)^2\le 2a^2+2b^2$ gives 
\begin{align}\label{eq-all}
     \sum_{i=1}^t \eta_i\|\nabla F(\mathbf{y}_i)\|^2
\;\le\; 2\Big( \sum_{i=1}^t \eta_i\|\nabla F(\mathbf{w}_i)\|^2\Big)
\;+\; 2 \beta^2\gamma^2 \Big( \sum_{i=1}^t \eta_i \|\mathbf{m}_i\|^2\Big).
\end{align}
Involving the bound on $\sum_{i=1}^t \eta_i\|\nabla F(\mathbf{w}_i)\|^2$ in Lemma \ref{thm:nag-noncvx}, we know with probability $1-\delta$
\begin{align}\label{eq-1}\sum_{i=1}^t \eta_i\|\nabla F(\mathbf{w}_i)\|^2 \le 8\Delta_1
+ \frac{32C_m(\gamma,\beta)}{(1-\gamma)^2} \left(\sigma^2\sum_{k=1}^t \eta_k^{2}  + 8G^2 \log(2/\delta) + G^2 \sum_{k=1}^t \eta_k^{2} \right)
+ \frac{16\frac{1 }{1-\gamma}G^2 \log(2/\delta)}{\min\{1,\ \frac{(1-\gamma)G^2}{4c \sigma^2 }\}}.\end{align}
Involving the bound on $\sum_{i=1}^t \|\mathbf{m}_i\|^2$ in (\ref{eq:sum-mt-final}), we know
\begin{align}\label{eq-2}
\sum_{i=1}^t \eta_i \|\mathbf{m}_i\|^2 \le\; \frac{4c}{(1-\gamma)^2}\sum_{i=1}^t \eta_i^2 \|\xi_i\|^2
\;+\; \frac{8c}{(1-\gamma)^2}\sum_{i=1}^t \eta_i^2 \|\nabla F_S(\mathbf{w}_i)\|^2,
\end{align}
where we used the fact that $\eta_t \leq c$ for all $t \in \mathbb{N}$. 
Combining the bound on $\sum_{i=1}^t \eta_i^2 \|\xi_i\|^2$ in (\ref{eq:marrtingale}), the bound on $\sum_{i=1}^t \eta_i\|\nabla F(\mathbf{w}_i)\|^2$ in (\ref{eq-1}) and the bound on $\sum_{i=1}^t \eta_i\|\mathbf{m}_i\|^2 $ in (\ref{eq-2}), with probability $1-2\delta$, we have the following inequality
\begin{align}\label{bound-on-m}\nonumber
&\sum_{i=1}^t \eta_i \|\mathbf{m}_i\|^2 \le  \frac{4c}{(1-\gamma)^2}\left(\sum_{k=1}^t \eta_k^{2} \sigma^2 + 8G^2 \log(2/\delta) + G^2 \sum_{k=1}^t \eta_k^{2} \right)
\;\\&+\; \frac{8c^2}{(1-\gamma)^2} \left( 8\Delta_1
+ \frac{32C_m(\gamma,\beta)}{(1-\gamma)^2} \left(\sigma^2\sum_{k=1}^t \eta_k^{2}  + 8G^2 \log(2/\delta) + G^2 \sum_{k=1}^t \eta_k^{2} \right)
+ \frac{16\frac{1 }{1-\gamma}G^2 \log(2/\delta)}{\min\{1,\ \frac{(1-\gamma)G^2}{4c \sigma^2 }\}} \right).
\end{align}
Thus, plugging the bound on $\sum_{i=1}^t \eta_i \|\mathbf{m}_i\|^2$ in (\ref{bound-on-m}) and the bound on $\sum_{i=1}^t \eta_i\|\nabla F(\mathbf{w}_i)\|^2$ in (\ref{eq-1}) into (\ref{eq-all}), with probability $1-2\delta$, we have the following inequality
\begin{align*}
 &\sum_{i=1}^t \eta_i\|\nabla F(\mathbf{y}_i)\|^2
\\&\;\le\; 2\Big( 8\Delta_1
+ \frac{32C_m(\gamma,\beta)}{(1-\gamma)^2} \left(\sigma^2\sum_{k=1}^t \eta_k^{2}  + 8G^2 \log(2/\delta) + G^2 \sum_{k=1}^t \eta_k^{2} \right)
+ \frac{16\frac{1 }{1-\gamma}G^2 \log(2/\delta)}{\min\{1,\ \frac{(1-\gamma)G^2}{4c \sigma^2 }\}}\Big)\\&+ 2 \beta^2\gamma^2 \frac{4c}{(1-\gamma)^2}\left(\sum_{k=1}^t \eta_k^{2} \sigma^2 + 8G^2 \log(2/\delta) + G^2 \sum_{k=1}^t \eta_k^{2} \right)
\;\\&+\; 2 \beta^2\gamma^2\frac{8c^2}{(1-\gamma)^2} \left( 8\Delta_1
+ \frac{32C_m(\gamma,\beta)}{(1-\gamma)^2} \left(\sigma^2\sum_{k=1}^t \eta_k^{2}  + 8G^2 \log(2/\delta) + G^2 \sum_{k=1}^t \eta_k^{2} \right)
+ \frac{16\frac{1 }{1-\gamma}G^2 \log(2/\delta)}{\min\{1,\ \frac{(1-\gamma)G^2}{4c \sigma^2 }\}} \right) \\&= (2+2 \beta^2\gamma^2 c^2\frac{8}{(1-\gamma)^2})\Big( 8\Delta_1
+ \frac{32C_m(\gamma,\beta)}{(1-\gamma)^2} \left(\sigma^2\sum_{k=1}^t \eta_k^{2}  + 8G^2 \log(2/\delta) + G^2 \sum_{k=1}^t \eta_k^{2} \right)
+ \frac{16\frac{1 }{1-\gamma}G^2 \log(2/\delta)}{\min\{1,\ \frac{(1-\gamma)G^2}{4c \sigma^2 }\}}\Big)\\&+2 \beta^2\gamma^2 \frac{4c}{(1-\gamma)^2}\left(\sum_{k=1}^t \eta_k^{2} \sigma^2 + 8G^2 \log(2/\delta) + G^2 \sum_{k=1}^t \eta_k^{2} \right).
\end{align*}
Further plugging this inequality into (\ref{rxnbxm1}) implies that with probability $1-2\delta$
\begin{align}\label{eq-boun-on-shift}\nonumber
    &\Big\| \sum_{i=1}^t   \eta_i \nabla F_S(\mathbf{y}_{i}) \Big\|  \leq \Big( \sum_{i=1}^t \eta_i \Big)^{1/2} \times \left (  2 \beta^2\gamma^2 \frac{4c}{(1-\gamma)^2}\left(\sum_{k=1}^t \eta_k^{2} \sigma^2 + 8G^2 \log(2/\delta) + G^2 \sum_{k=1}^t \eta_k^{2} \right) \right.\\& \left. +(2+2 \beta^2\gamma^2 c^2\frac{8}{(1-\gamma)^2})\Big( 8\Delta_1
+ \frac{32C_m(\gamma,\beta)}{(1-\gamma)^2} \left(\sigma^2\sum_{k=1}^t \eta_k^{2}  + 8G^2 \log(2/\delta) + G^2 \sum_{k=1}^t \eta_k^{2} \right)
+ \frac{16\frac{1 }{1-\gamma}G^2 \log(2/\delta)}{\min\{1,\ \frac{(1-\gamma)G^2}{4c \sigma^2 }\}}\Big) \right)^{1/2}.
\end{align}
\medskip\noindent\textbf{Step 4: Putting together.}
Plugging the bound on $\| \sum_{i=1}^t   \eta_i \nabla F_S(\mathbf{y}_{i}) \|$ in (\ref{eq-boun-on-shift}) and the bound on $\max_{1\leq t \leq T} \Big\| \sum_{i=1}^t  \eta_i (\nabla f(\mathbf{y}_i;z_{j_i}) -  \nabla F_S(\mathbf{y}_{i}))\Big\|$ in (\ref{bound-on-noise})  into (\ref{bound-on-w}), with probability at least $1-3\delta$ we have  the following inequality  uniformly for all $t = 1,...,T$
\begin{align*}
&\|\mathbf{w}_{t+1}\| \le  \frac{2}{1-\gamma} \left( \frac{2G \sqrt{c}}{3} + \sigma \left(\sum_{i=1}^T\eta_i^2   \right)^{\frac{1}{2}}  \right)\log \frac{2}{\delta} \\&+ \frac{1}{1-\gamma} \Big( \sum_{i=1}^t \eta_i \Big)^{1/2} \times \left (  2 \beta^2\gamma^2 \frac{4c}{(1-\gamma)^2}\left(\sum_{k=1}^t \eta_k^{2} \sigma^2 + 8G^2 \log(2/\delta) + G^2 \sum_{k=1}^t \eta_k^{2} \right) \right.\\& \left. + (2+2 \beta^2\gamma^2 c^2\frac{8}{(1-\gamma)^2})\Big( 8\Delta_1
+ \frac{32C_m(\gamma,\beta)}{(1-\gamma)^2} \left(\sigma^2\sum_{k=1}^t \eta_k^{2}  + 8G^2 \log(2/\delta) + G^2 \sum_{k=1}^t \eta_k^{2} \right)
+ \frac{16\frac{1 }{1-\gamma}G^2 \log(2/\delta)}{\min\{1,\ \frac{(1-\gamma)G^2}{4c \sigma^2 }\}}\Big) \right)^{1/2},
\end{align*}
which implies the claimed bound. The proof is complete.
\end{proof}
\newcommand{\ip}[2]{\langle #1,#2\rangle}
\newcommand{\norm}[1]{\left\|#1\right\|}
\newcommand{\R}{\mathbb{R}}

\begin{lemma}
\label{thm:nag-corrected-final}
Suppose Assumptions \ref{assum666}, \ref{assu4} and \ref{assu8} hold, and suppose $F_S$ satisfies Assumption \ref{assu10} with parameter $2\mu_S$.
Let $\{ \mathbf{w}_t\}_t$ be the sequence produced by NAG, i.e. (\ref{eq:nag}), with $\eta_t = \frac{2}{\mu_S (t+t_0)}$ such that $t_0 >0$.
Then for any $\delta >0$, with probability at least $1-\delta$ we have the following inequality
\begin{align*}
F_S(\mathbf{w}_{T+1})-F_S^{\ast}=\mathcal{O}\left(\frac{\log 1/\delta}{T}\right).
\end{align*}
\end{lemma}

\begin{proof} The proof proceeds with six precise steps.

\medskip\noindent\textbf{Step 1: A new decomposition under smoothness.}
For a concise presentation, define weights 
\[W_t:=(t+t_0)(t+t_0-1),\qquad \Delta_1:=F_S(\mathbf{w}_1)-F_S^{\ast},
\]and set
\[
\kappa_\gamma:=\frac{1}{1-\gamma}+\frac{2}{(1-\gamma)^2}+\frac{2}{(1-\gamma)^3},
\qquad
\kappa_\gamma^{(1)}:=\frac{1}{1-\gamma}+\frac{1}{(1-\gamma)^2}.
\]
According to (\ref{eq:smooth-descent}), by $\beta$-smoothness of Assumption \ref{assu4},
\begin{equation}\label{eq:smooth-one}
F_S(\mathbf{w}_{t+1})-F_S(\mathbf{w}_t)\ \le\ \langle \nabla F_S(\mathbf{w}_t),\mathbf{m}_{t+1}\rangle\ +\ \frac{\beta}{2}\|\mathbf{m}_{t+1}\|^2.
\end{equation}
Using NAG's update $\mathbf{m}_{t+1} \;=\; \gamma \mathbf{m}_t - \eta_t\, \mathbf{g}_t$ and $\xi_t=\nabla f(\mathbf{y}_t;z_{j_t})-\nabla F_S(\mathbf{y}_t)$, expand
\begin{equation}\label{eq:inner-expand1}
\langle \nabla F_S(\mathbf{w}_t),\mathbf{m}_{t+1}\rangle
=\gamma\langle \nabla F_S(\mathbf{w}_t),\mathbf{m}_t\rangle
-\eta_t\langle \nabla F_S(\mathbf{w}_t),\nabla F_S(\mathbf{y}_t)\rangle
-\eta_t\langle \nabla F_S(\mathbf{w}_t),\xi_t\rangle.
\end{equation}
By polarization and smoothness,
\begin{align}\label{eq:polar1}\nonumber
&-\langle \nabla F_S(\mathbf{w}_t),\nabla F_S(\mathbf{y}_t)\rangle
= -\tfrac12\|\nabla F_S(\mathbf{w}_t)\|^2 -\tfrac12\|\nabla F_S(\mathbf{y}_t)\|^2
+\tfrac12\|\nabla F_S(\mathbf{y}_t)-\nabla F_S(\mathbf{w}_t)\|^2
\\&\le -\tfrac12\|\nabla F_S(\mathbf{w}_t)\|^2 + \tfrac12 \beta^2\gamma^2\|\mathbf{m}_t\|^2 .
\end{align}
From $\mathbf{m}_{t+1}=\gamma \mathbf{m}_t-\eta_t \nabla f(\mathbf{y}_t;z_{j_t})$, a triangle inequality of $\|\mathbf{m}_{t+1}\|^2$ gives
\begin{equation}\label{eq:m-split}
\|\mathbf{m}_{t+1}\|^2
\le 2\gamma^2\|\mathbf{m}_t\|^2\   +\ 2\eta_t^2\|\nabla f(\mathbf{y}_t;z_{j_t})\|^2 .
\end{equation}
Plug \eqref{eq:inner-expand1}, \eqref{eq:polar1}, \eqref{eq:m-split} into \eqref{eq:smooth-one} gives
\begin{align}\nonumber
&F_S(\mathbf{w}_{t+1})-F_S(\mathbf{w}_t)
 \le \gamma\langle \nabla F_S(\mathbf{w}_t),\mathbf{m}_t\rangle
\ -\ \frac{\eta_t}{2}\|\nabla F_S(\mathbf{w}_t)\|^2
\ \\&+\ \frac{\eta_t \beta^2\gamma^2}{2}\|\mathbf{m}_t\|^2
\ -\ \eta_t\langle \nabla F_S(\mathbf{w}_t),\xi_t\rangle  +\ \beta\gamma^2\|\mathbf{m}_t\|^2
\ +\ \beta\eta_t^2\|\nabla f(\mathbf{y}_t;z_{j_t})\|^2.
\label{eq:oneshot}
\end{align}
Split $\frac{\eta_t}{2}\|\nabla F_S(\mathbf{w}_t)\|^2=\frac{\eta_t}{4}\|\nabla F_S(\mathbf{w}_t)\|^2+\frac{\eta_t}{4}\|\nabla F_S(\mathbf{w}_t)\|^2$
and move the first half to LHS of (\ref{eq:oneshot}):
\begin{align}
\frac{\eta_t}{4}\|\nabla F_S(\mathbf{w}_t)\|^2 + \big(F_S(\mathbf{w}_{t+1})-F_S(\mathbf{w}_t)\big)
&\le \gamma\langle \nabla F_S(\mathbf{w}_t),\mathbf{m}_t\rangle
\ -\ \frac{\eta_t}{4}\|\nabla F_S(\mathbf{w}_t)\|^2
\ +\ \frac{\eta_t \beta^2\gamma^2}{2}\|\mathbf{m}_t\|^2 \nonumber\\
&\quad -\ \eta_t\langle \nabla F_S(\mathbf{w}_t),\xi_t\rangle
\ +\ \beta\gamma^2\|\mathbf{m}_t\|^2
\ +\ \beta\eta_t^2\|\nabla f(\mathbf{y}_t;z_{j_t})\|^2.
\label{eq:balanced}
\end{align}
Apply PL on $F_S$ with parameter $2\mu_S$ and $\eta_t=\frac{2}{\mu_S(t+t_0)}$:
\[
\frac{\eta_t}{4}\|\nabla F_S(\mathbf{w}_t)\|^2
\ge  \mu_S\eta_t \big(F_S(\mathbf{w}_t)-F_S^{\ast}\big)
= \frac{2}{t+t_0}\big(F_S(\mathbf{w}_t)-F_S^{\ast}\big).
\]
Plug this inequality into (\ref{eq:balanced}):
\begin{align*}
&\frac{\eta_t}{4}\|\nabla F_S(\mathbf{w}_t)\|^2
+\ \big(F_S(\mathbf{w}_{t+1})-F_S(\mathbf{w}_t)\big) \le\ \gamma\langle \nabla F_S(\mathbf{w}_t),\mathbf{m}_t\rangle
\nonumber\\
& -  \ \frac{2}{t+t_0}\big(F_S(\mathbf{w}_t)-F_S^{\ast}\big) 
\ +\ \frac{\eta_t \beta^2\gamma^2}{2}\|\mathbf{m}_t\|^2
\ -\ \eta_t\langle \nabla F_S(\mathbf{w}_t),\xi_t\rangle
\ +\ \beta\gamma^2\|\mathbf{m}_t\|^2
\ +\ \beta\eta_t^2\|\nabla f(\mathbf{y}_t;z_{j_t})\|^2,
\end{align*}
which implies
\begin{align}
&\frac{\eta_t}{4}\|\nabla F_S(\mathbf{w}_t)\|^2
+\ \big(F_S(\mathbf{w}_{t+1})-F_S^{\ast}\big) \le\ \gamma\langle \nabla F_S(\mathbf{w}_t),\mathbf{m}_t\rangle
\nonumber\\
&
+\ \frac{t+t_0-2}{t+t_0}\big(F_S(\mathbf{w}_t)-F_S^{\ast}\big) \ +\ \frac{\eta_t \beta^2\gamma^2}{2}\|\mathbf{m}_t\|^2
\ -\ \eta_t\langle \nabla F_S(\mathbf{w}_t),\xi_t\rangle
\ + \beta\gamma^2\|\mathbf{m}_t\|^2
\ +\ \beta\eta_t^2\|\nabla f(\mathbf{y}_t;z_{j_t})\|^2 .
\label{eq:balanced-PL}
\end{align}
Multiply \eqref{eq:balanced-PL} by $W_t$ and sum $t=1$ to $T$.
By the identity
$W_t\cdot\tfrac{1}{t+t_0}=t+t_0-1$ and the following polynomial telescoping identity
\begin{align}\label{eq:poly-telescope}\nonumber
&\sum_{t=1}^T (t+t_0)(t+t_0-1)\big(F_S(\mathbf{w}_{t+1})-F_S^{\ast}\big)
-  (t+t_0-1)(t+t_0-2)\big(F_S(\mathbf{w}_t)-F_S^{\ast}\big)
\\&= (T+t_0)(T+t_0-1)\big(F_S(\mathbf{w}_{T+1})-F_S^{\ast}\big) -  t_0 (t_0-1)\Delta_1,
\end{align}
we obtain
\begin{align}
&\sum_{t=1}^T \underbrace{\Big(\frac{\eta_t}{4}W_t\Big)}_{=\frac{t+t_0-1}{2\mu_S}}
  \|\nabla F_S(\mathbf{w}_t)\|^2
\ +\ W_T\big(F_S(\mathbf{w}_{T+1})-F_S^{\ast}\big)\nonumber\\
\le&
\underbrace{\sum_{t=1}^T W_t\,\gamma\langle \nabla F_S(\mathbf{w}_t),\mathbf{m}_t\rangle}_{\Sigma_1}
\ +\
\underbrace{\sum_{t=1}^T W_t \cdot \frac{\beta^2\gamma^2}{2}\eta_t\|\mathbf{m}_t\|^2}_{\Sigma_{2a}}
\ +\
\underbrace{\sum_{t=1}^T -W_t\,\eta_t\langle \nabla F_S(\mathbf{w}_t),\xi_t\rangle}_{\Sigma_3}\nonumber\\[-1mm]
&\hspace{1.2cm}
+\ \underbrace{\sum_{t=1}^T W_t \cdot  \beta\gamma^{2}\|\mathbf{m}_t\|^2}_{\Sigma_{2b}}
\ +\ \underbrace{\sum_{t=1}^T W_t \cdot \beta\eta_t^2\|\nabla f(\mathbf{y}_t;z_{j_t})\|^2}_{\Sigma_{4}}
\ +\  t_0 (t_0-1)\Delta_1 .
\label{eq:master}
\end{align}
After obtaining this decomposition, the next step is to bound the terms on the RHS of (\ref{eq:master}).

\medskip\noindent\textbf{Step 2: Control of $\Sigma_{2a}$ and $\Sigma_{2b}$.}
We first prove the following two inequalities
\begin{align*}
&\sum_{t=1}^{T}(t+t_0-1)\|\mathbf{m}_{t}\|^2
\ \le\ \frac{\kappa_\gamma^{(1)}}{1-\gamma}\sum_{i=1}^{T}(i+t_0)\eta_i^2\|\nabla f(\mathbf{y}_i;z_{j_i})\|^2,
\qquad
\\&\sum_{t=1}^{T}W_t\|\mathbf{m}_t\|^2
\ \le\ \frac{4\kappa_\gamma}{1-\gamma}\sum_{i=1}^{T} \,\, (i+t_0)^2\eta_i^2\|\nabla f(\mathbf{y}_i;z_{j_i})\|^2,
\end{align*}
 and then the bound of $\Sigma_{2a}$ and $\Sigma_{2b}$. From NAG's update $\mathbf{m}_{t+1}=\gamma \mathbf{m}_t-\eta_t \nabla f(\mathbf{y}_t;z_{j_t})$ with $\mathbf{m}_1=\mathbf{0}$ we have, by induction,
\begin{equation}\label{eq:m-unroll}
\mathbf{m}_{t+1} \;=\; -\sum_{i=1}^{t}\gamma^{\,t-i}\,\eta_i\,\nabla f(\mathbf{y}_i;z_{j_i}).
\end{equation}
Then by Schwarz’s inequality and the fact
$\sum_{i=1}^{t}\gamma^{t-i} \le \frac{1}{1-\gamma}$, we get
\begin{align}\label{eq:m-square-one-step} \|\mathbf{m}_{t+1}\|^2
\;=\;\Big\| \sum_{i=1}^{t}\gamma^{\,t-i}\eta_i \nabla f(\mathbf{y}_i;z_{j_i})\Big\|^2
\; \le\; \Big(\sum_{i=1}^{t}\gamma^{\,t-i}\Big)\,\sum_{i=1}^{t}\gamma^{\,t-i}\|\eta_i \nabla f(\mathbf{y}_i;z_{j_i})\|^2
\;\le\; \sum_{i=1}^{t}\frac{\gamma^{\,t-i}}{1-\gamma}\,\eta_i^2\,\|\nabla f(\mathbf{y}_i;z_{j_i})\|^2.
\end{align}
Let $\{a_t\}_{t\ge1}$ be a nonnegative weight sequence. Summing \eqref{eq:m-square-one-step} with
weights $a_t$ and swapping the order of summation (by Lemma \ref{lem:reorder}) yields the following geometric reordering inequality
\begin{align}\label{eq:geom-reorder}\nonumber
&\sum_{t=1}^{T} a_t\,\|\mathbf{m}_{t}\|^2=\sum_{t=1}^{T} a_t\, \sum_{i=1}^{t-1}\frac{\gamma^{\,t-1-i}}{1-\gamma}\,\eta_i^2\,\|\nabla f(\mathbf{y}_i;z_{j_i})\|^2
\;=\; \frac{1}{1-\gamma}\sum_{t=1}^{T} \eta_t^2\|\nabla f(\mathbf{y}_t;z_{j_t})\|^2\, \sum_{i= t+1 }^{T} a_i\,\gamma^{\,i-1-t} \\&= \frac{1}{1-\gamma}\sum_{i=1}^{T} \eta_i^2\|\nabla f(\mathbf{y}_i;z_{j_i})\|^2\,\underbrace{\sum_{t= i+1 }^{T} a_t\,\gamma^{\,t-1-i}}_{=:K_i(a)},
\end{align}
where Lemma \ref{lem:reorder} is used in the second identity.
Thus everything reduces to bounding $K_i(a)$ for the particular weights we use below.

\medskip\noindent\textbf{Step 2.1: linear weights $a_t=t+t_0-1$.}
Write $t=i+k+1$ with $k\ge 0$. Then
\[
K_i(a) \;=\; \sum_{t=i+1}^{T} (t+t_0-1)\,\gamma^{\,t-1-i}
\;=\; \sum_{k=0}^{T-i-1} (i+t_0+k)\,\gamma^{\,k}
\;= \sum_{k=0}^{T-i-1} (i+t_0 )\,\gamma^{\,k} + \sum_{k=0}^{T-i-1} k\,\gamma^{\,k}.
\]
Since $i+t_0\ge 1$ and $\gamma < 1$, dropping the truncation (monotone in $T$) and using the infinite-sum bounds
$\sum_{k\ge0}\gamma^k=\frac{1}{1-\gamma}$ and $\sum_{k\ge0}k\gamma^k=\frac{\gamma}{(1-\gamma)^2}$ gives
\[
K_i(a) \;\le\;  (i+t_0)\frac{1}{1-\gamma} + \frac{\gamma}{(1-\gamma)^2} \;\le\; \Big(\frac{1}{1-\gamma}+\frac{1}{(1-\gamma)^2}\Big)(i+t_0)
\;=:\; \kappa_\gamma^{(1)}\,(i+t_0). 
\]
Plugging this into \eqref{eq:geom-reorder} yields
\begin{equation}\label{eq:linear-weighted-m}
\sum_{t=1}^{T}(t+t_0-1)\|\mathbf{m}_t\|^2
\;\le\; \frac{\kappa_\gamma^{(1)}}{1-\gamma}\sum_{i=1}^{T}(i+t_0)\,\eta_i^2\,\|\nabla f(\mathbf{y}_i;z_{j_i})\|^2 \le \frac{\kappa_\gamma^{(1)} L^2}{1-\gamma}   \frac{4}{\mu_S^2 } \sum_{i=1}^{T} \frac{1}{(i+t_0)} \le \frac{\kappa_\gamma^{(1)} L^2}{1-\gamma}   \frac{4}{\mu_S^2  } \log\!\frac{T+t_0}{t_0},
\end{equation}
where we have used Assumption \ref{assum666} in the second inequality.
Thus
\begin{equation}\label{eq:S2a-final}
\Sigma_{2a}\ \le\ \frac{4\beta^2\gamma^2\kappa_\gamma^{(1)}}{(1-\gamma)\mu_S^3}\,L^2\,
\log\!\frac{T+t_0}{t_0}.
\end{equation}

\medskip\noindent\textbf{Step 2.2: quadratic weights $a_t=(t+t_0)(t+t_0-1)$.}
Let $x:=i+t_0+1$ and again write $t=i+k+1$ with $k\ge 0$. Then
\[
(t+t_0)(t+t_0-1)=(x+k)(x+k-1)=(x^2-x)+(2x-1)k+k^2,
\]
hence
\[
K_i(a)
=\sum_{t=i+1}^{T} (t+t_0)(t+t_0-1)\,\gamma^{\,t-1-i}
= \sum_{k=0}^{T-i-1}\big[(x^2-x)+(2x-1)k+k^2\big]\gamma^{k}.
\]
Since $i+t_0\ge 1$ and $\gamma < 1$, using the infinite-sum bounds $\sum_{k\ge0}\gamma^k=\frac{1}{1-\gamma}$, $\sum_{k\ge0}k\gamma^k=\frac{\gamma}{(1-\gamma)^2}$ and $\sum_{k\ge0}k^2\gamma^k=\frac{\gamma(1+\gamma)}{(1-\gamma)^3}$, we obtain
\begin{align}\label{eq-bound-on-genmetric}\nonumber
&K_i(a)
\;\le\; \frac{ (x^2-x)}{1-\gamma} + \frac{(2x-1)\gamma}{(1-\gamma)^2}
+ \frac{\gamma(1+\gamma)}{(1-\gamma)^3} \\&\le\; \left(\frac{1}{1-\gamma}+\frac{2}{(1-\gamma)^2}+\frac{2}{(1-\gamma)^3}\right)x^2
\;=:\; \kappa_\gamma\, (i+t_0+1)^2 \le4 \kappa_\gamma\, (i+t_0)^2. 
\end{align}
Plugging (\ref{eq-bound-on-genmetric}) into \eqref{eq:geom-reorder} gives
\begin{equation}\label{eq:quadratic-weighted-m}
\sum_{t=1}^{T} a_t\,\|\mathbf{m}_t\|^2
\;\le\;  \frac{4}{1-\gamma}\sum_{i=1}^{T} \eta_i^2\|\nabla f(\mathbf{y}_i;z_{j_i})\|^2\,\kappa_\gamma\, (i+t_0)^2\le \frac{ 4\kappa_\gamma}{1-\gamma}\cdot \frac{4L^2}{\mu_S^2}\,T,
\end{equation}
where we have used Assumption \ref{assum666} in the second inequality.
Thus
\begin{equation}\label{eq:Sigma2ap-final}
\Sigma_{2b}
\ \le\ \frac{16\beta\gamma^2\kappa_\gamma}{1-\gamma}\cdot \frac{L^2}{\mu_S^2}\,T.
\end{equation}

\medskip\noindent\textbf{Step 3: High-probability control of $\Sigma_3$.} Denote by $M_t:=-W_t\eta_t\langle \nabla F_S(\mathbf{w}_t),\xi_t\rangle$. Since $\mathbb{E}_{j_t} M_t = 0$, thus $\{M_t\}$ is a martingale difference sequence. By Assumption \ref{assum666}, we have $\|\nabla F_S(\mathbf{w}_t)\|\le  L$ and $\|\xi_t\|\le2L$. Together with $W_t\eta_t=\frac{2}{\mu_S}(t+t_0-1)\le \frac{2}{\mu_S}(T+t_0-1)$, these increments gives 
\begin{align}\label{eq-bound-free}
|M_t|\le W_t\eta_t \|\nabla F_S(\mathbf{w}_t) \| \|\xi_t\| \le \frac{4}{\mu_S}(T+t_0-1)\,L^2.
\end{align}
By Assumption \ref{assu8}, the conditional variance satisfies
\begin{align}\label{eq-bound-freedman}
\mathbb{E}_{j_t}[M_t^2 ] \le \mathbb{E}_{j_t} W_t^2\eta_t^2 \|\nabla F_S(\mathbf{w}_t) \|^2 \|\xi_t\|^2
\le \frac{4}{\mu_S^2}(t+t_0-1)^2\,\|\nabla F_S(\mathbf{w}_t)\|^2\,\sigma^2.
\end{align}
Apply Part (b) of Lemma \ref{lemma36} with (\ref{eq-bound-free}) and (\ref{eq-bound-freedman}) yields, with probability at least $1-\delta$,
\begin{align*}
\Sigma_3=\sum_{t=1}^T M_t
\ \le\ \frac{\rho \sigma^2}{ \mu_S L^2 (T+t_0-1)}\sum_{t=1}^T (t+t_0-1)^2\|\nabla F_S(\mathbf{w}_t)\|^2
\ +\ \frac{4}{\mu_S}(T+t_0-1)\frac{L^2\log(1/\delta)}{\rho}.
\end{align*}
Setting $\rho = \min\{1,\frac{L^2}{4\sigma^2}\}$, the above inequality implies
\begin{align}\label{eq:Sigma3-bound}
\Sigma_3 \le\ \frac{1}{ 4\mu_S}\sum_{t=1}^T (t+t_0-1)\|\nabla F_S(\mathbf{w}_t)\|^2
\ +\ \frac{4}{\mu_S}(T+t_0-1)\frac{L^2\log(1/\delta)}{\min\{1,\frac{L^2}{4\sigma^2}\}}.
\end{align}

\medskip\noindent\textbf{Step 4: Control of $\Sigma_4$.}
By Assumption \ref{assum666},
\begin{align*}
\sum_{t=1}^T W_t \cdot \beta\eta_t^2\|\nabla f(\mathbf{y}_t;z_{j_t})\|^2 \le \sum_{t=1}^T   \beta L^2 \frac{4}{\mu_S^2} = \frac{4\beta L^2T}{\mu_S^2}.
\end{align*}

\medskip\noindent\textbf{Step 5: High-probability control of $\Sigma_1$.}
By smoothness (Assumption \ref{assu4}) and NAG's update $\mathbf{w}_{t+1} \;=\; \mathbf{w}_t + \mathbf{m}_{t+1}$,
\[
\langle \nabla F_S(\mathbf{w}_t),\mathbf{m}_t\rangle
\ \le\ \langle \nabla F_S(\mathbf{w}_{t-1}),\mathbf{m}_t\rangle + \beta\|\mathbf{m}_t\|^2.
\]
Multiply by $\gamma W_t$ and sum over $t=1,\dots,T$:
\begin{equation}\label{eq:s1-first}
\Sigma_1
=\sum_{t=1}^{T}\gamma W_t\langle \nabla F_S(\mathbf{w}_t),\mathbf{m}_t\rangle = \sum_{t=2}^{T}\gamma W_t\langle \nabla F_S(\mathbf{w}_t),\mathbf{m}_t\rangle
\ \le\ \sum_{t=2}^{T}\gamma W_t\langle \nabla F_S(\mathbf{w}_{t-1}),\mathbf{m}_t\rangle
\ +\ \sum_{t=2}^{T}\gamma W_t\,\beta\|\mathbf{m}_t\|^2,
\end{equation}
where the second identity holds due to $\mathbf{m}_1=\mathbf{0}$.
Since $\mathbf{m}_t=\gamma \mathbf{m}_{t-1}-\eta_{t-1}\nabla f(\mathbf{y}_{t-1};z_{j_{t-1}})$,
\[
\langle \nabla F_S(\mathbf{w}_{t-1}),\mathbf{m}_t\rangle
= \gamma\langle \nabla F_S(\mathbf{w}_{t-1}),\mathbf{m}_{t-1}\rangle
-\eta_{t-1}\langle \nabla F_S(\mathbf{w}_{t-1}),\nabla f(\mathbf{y}_{t-1};z_{j_{t-1}})\rangle .
\]
Plugging this into \eqref{eq:s1-first},
\begin{align}\nonumber
&\Sigma_1=\sum_{t=1}^{T}\gamma W_t\langle \nabla F_S(\mathbf{w}_t),\mathbf{m}_t\rangle \\
&\le \sum_{t=2}^{T}\gamma^2 W_t\langle \nabla F_S(\mathbf{w}_{t-1}),\mathbf{m}_{t-1}\rangle
\ -\ \sum_{t=2}^{T}\gamma W_t\,\eta_{t-1}\langle \nabla F_S(\mathbf{w}_{t-1}),\nabla f(\mathbf{y}_{t-1};z_{j_{t-1}})\rangle
\ +\ \sum_{t=2}^{T}\gamma W_t\,\beta\|\mathbf{m}_t\|^2 \nonumber\\
&= \sum_{t=1}^{T-1}\gamma^2 W_{t+1}\langle \nabla F_S(\mathbf{w}_{t}),\mathbf{m}_{t}\rangle
\ -\ \sum_{t=1}^{T-1}\gamma W_{t+1}\,\eta_{t}\langle \nabla F_S(\mathbf{w}_{t}),\nabla f(\mathbf{y}_t;z_{j_t})\rangle
\ +\ \sum_{t=2}^{T}\gamma W_{t}\,\beta\|\mathbf{m}_{t}\|^2. \label{eq:s1-step2}
\end{align}
Apply \eqref{eq:s1-step2} recursively to the first sum on the right and stop at $\mathbf{m}_1=\mathbf{0}$; a standard induction gives the following inequality
\begin{equation}\label{eq:w-unroll}
\Sigma_1 
\;\le\;
-\sum_{i=1}^{T-1}\eta_i
\Big(\sum_{t=i+1}^{T}\gamma^{\,t-i} W_t\Big)\,
\langle \nabla F_S(\mathbf{w}_i),\nabla f(\mathbf{y}_i;z_{j_i})\rangle
\;+\;
\sum_{t=2}^{T}
\Big(\sum_{s=t}^{T}\gamma^{\,s-t} W_s\Big)\,\beta\|\mathbf{m}_t\|^2 .
\end{equation}
With the notation
\[
H_i(W):=\sum_{t=i+1}^{T}\gamma^{\,t-i} W_t,\qquad
K_t(W):=\sum_{s=t}^{T}\gamma^{\,s-t} W_s ,
\]
then
\begin{equation}\label{eq:sigma1-B123}
\Sigma_1
\;\le\; (B1)\;-\;(B2)\;+\;(B3),
\end{equation}
where
\[
\begin{aligned}
(B1) &:= \sum_{t=2}^{T}\underbrace{\Big(\sum_{s=t}^{T}\gamma^{\,s-t}\,W_s\Big)}_{:=\,K_t(W)}
\,\beta\,\|\mathbf{m}_t\|^2,\\[1mm]
(B2) &:= \sum_{t=1}^{T-1}\underbrace{\Big(\sum_{i=t+1}^{T}\gamma^{\,i-t}\,W_i\Big)}_{:=\,H_t(W)}
\,\eta_t\,\ip{\nabla F_S(\mathbf{w}_t)}{\nabla F_S(\mathbf{y}_t)},\\[1mm]
(B3) &:= -\sum_{t=1}^{T-1}H_t(W)\,\eta_t\,\ip{\nabla F_S(\mathbf{w}_t)}{\xi_t}.
\end{aligned}
\]

\medskip\noindent\textbf{Step 5.1: Bounding $H_t(W)$ and $K_t(W)$.}
Let $x:=t+t_0$ and again write $s=t+k$ with $k\ge 0$. Then
\[
(s+t_0)(s+t_0-1)=(x+k)(x+k-1)=(x^2-x)+(2x-1)k+k^2,
\]
hence
\[
K_t(W)
=\sum_{s=t}^{T} (s+t_0)(s+t_0-1)\,\gamma^{\,s-t}
= \sum_{k=0}^{T-t}\big[(x^2-x)+(2x-1)k+k^2\big]\gamma^{k}.
\]
Since $t+t_0\ge 1$ and $\gamma < 1$, using the infinite-sum bounds $\sum_{k\ge0}\gamma^k=\frac{1}{1-\gamma}$, $\sum_{k\ge0}k\gamma^k=\frac{\gamma}{(1-\gamma)^2}$ and $\sum_{k\ge0}k^2\gamma^k=\frac{\gamma(1+\gamma)}{(1-\gamma)^3}$, we obtain
\[
K_t(W)
\;\le\; \frac{ (x^2-x)}{1-\gamma} + \frac{(2x-1)\gamma}{(1-\gamma)^2}
+ \frac{\gamma(1+\gamma)}{(1-\gamma)^3} \le\; \left(\frac{1}{1-\gamma}+\frac{2}{(1-\gamma)^2}+\frac{2}{(1-\gamma)^3}\right)x^2
\;=:\; \kappa_\gamma\, (t+t_0 )^2 . 
\]
Thus we obtain
\begin{equation}\label{eq:kernel-bds}
K_t(W)\ \le\ \kappa_\gamma\,(t+t_0)^2 \qquad H_t(W)\ \le\ \kappa_\gamma\,(t+t_0)^2.
\end{equation}

\medskip\noindent\textbf{Step 5.2: Bounding $(B1)$.}
By \eqref{eq:kernel-bds},
\[
(B1)\ \le\ \beta\kappa_\gamma\sum_{t=1}^{T} (t+t_0)^2\|\mathbf{m}_t\|^2.
\]
From (\ref{eq:m-square-one-step}) we know
\[
\|\mathbf{m}_t\|^2\ \le\ \sum_{i=1}^{t-1}\frac{\gamma^{\,t-1-i}}{1-\gamma}\,\eta_i^2\,\|\nabla f(\mathbf{y}_i;z_{j_i})\|^2
\ \le\ \frac{L^2}{1-\gamma}\sum_{i=1}^{t-1}\gamma^{\,t-1-i}\eta_i^2,
\]
where the second inequality holds due to Assumption \ref{assum666}.
By reordering with Lemma \ref{lem:reorder},
\begin{align*}
\sum_{t=1}^{T} (t+t_0)^2\|\mathbf{m}_t\|^2 &\le \frac{L^2}{1-\gamma}\sum_{t=1}^{T} (t+t_0)^2 \sum_{i=1}^{t-1}\gamma^{\,t-1-i}\eta_i^2
\ \\&\le\ \frac{L^2}{1-\gamma}\sum_{i=1}^{T-1}\eta_i^2
\sum_{t=i+1}^{T}\gamma^{\,t-1-i} (t+t_0)^2
\ \le\ \frac{4\kappa_\gamma L^2}{1-\gamma}\sum_{i=1}^{T-1}\eta_i^2\,(i+t_0)^2,
\end{align*}
where the last inequality holds because (\ref{eq-bound-on-genmetric}).
Since $\eta_i^2=\frac{4}{\mu_S^2(i+t_0)^2}$, we get
\begin{equation}\label{eq:B1-final}
(B1)\ \le\ \frac{16\beta\kappa_\gamma^2}{1-\gamma}\cdot \frac{L^2}{\mu_S^2}\,T.
\end{equation}

\medskip\noindent\textbf{Step 5.3: Bounding $(B2)$.}
Using the polarization lower bound:
\[
\langle \nabla F_S(\mathbf{y}_t),\nabla F_S(\mathbf{w}_t)\rangle
\ \ge\ \tfrac12\|\nabla F_S(\mathbf{w}_t)\|^2\ -\ \tfrac12 \beta^2\gamma^2\|\mathbf{m}_t\|^2,
\]
we have 
\[
-(B2)
\ \le\ -\sum_{t=1}^{T-1} H_t(W) 
\,\eta_t\, \tfrac12\|\nabla F_S(\mathbf{w}_t)\|^2\ +\  \sum_{t=1}^{T-1} H_t(W) 
\,\eta_t\,\tfrac12 \beta^2\gamma^2\|\mathbf{m}_t\|^2.
\]
Together with $H_t(W)\ge \gamma W_{t+1}$ and $\gamma \eta_tW_{t+1}=\gamma \frac{2}{\mu_S}(t+t_0+1)$,
\begin{equation}\label{eq:B2-final}
-(B2)\ \le\ -\frac{\gamma}{\mu_S}\sum_{t=1}^{T-1}(t+t_0+1)\,\|\nabla F_S(\mathbf{w}_t)\|^2
\ +\ \frac{\kappa_\gamma \beta^2\gamma^2}{\mu_S}\sum_{t=1}^{T-1}(t+t_0)\,\|\mathbf{m}_t\|^2,
\end{equation}
where we have also used
$H_t(W)\ \le\ \kappa_\gamma\,(t+t_0)^2$ in (\ref{eq:kernel-bds}).
Following the proof of (\ref{eq:geom-reorder}) and then the proof in Step 2.1, we can prove that
\begin{align*}
\sum_{t=1}^{T-1}(t+t_0)\|\mathbf{m}_t\|^2 \;\le\;  \frac{\kappa_\gamma^{(1)}}{\gamma(1-\gamma)}\sum_{i=1}^{T-1} \eta_i^2\|\nabla f(\mathbf{y}_i;z_{j_i})\|^2\, (i+t_0)
\;\le\;  \frac{\kappa_\gamma^{(1)} L^2}{\gamma(1-\gamma)}   \frac{4}{\mu_S^2  } \log\!\frac{T+t_0-1}{t_0}.
\end{align*}
Thus we have 
\begin{align*}
-(B2)\ \le\ -\frac{\gamma}{\mu_S}\sum_{t=1}^{T-1}(t+t_0+1)\,\|\nabla F_S(\mathbf{w}_t)\|^2
\ +\ \frac{\kappa_\gamma  \beta^2\gamma^2}{\mu_S}\frac{\kappa_\gamma^{(1)} L^2}{\gamma(1-\gamma)}   \frac{4}{\mu_S^2  } \log\!\frac{T+t_0-1}{t_0}.
\end{align*}

\medskip\noindent\textbf{Step 5.4: Bounding $(B3)$.}
Let $M_t:=-\eta_t\,\langle \xi_t,\nabla F_S(\mathbf{w}_t)\rangle \, H_t(W)$, where $t=1,\dots,T-1$. Since $\mathbb{E}_{j_t} M_t=0$, $\{M_t\}$ is a martingale difference sequence and $(B3)=\sum_{t=1}^{T-1} M_t$.
We bound its increments and conditional variance. First,
\[
|M_t|\ \le\ \eta_t\,\|\xi_t\|\,\|\nabla F_S(\mathbf{w}_t)\|\, H_t(W)
\ \le\ 2L\cdot L\cdot \eta_t\,H_t(W)
\ \le\ \frac{4\kappa_\gamma L^2}{\mu_S}\,(t+t_0),
\]
where we have used the fact that $\|\xi_t\|\le \|\nabla f(\mathbf{y}_t;z_{j_t})\|+\|\nabla F_S(\mathbf{y}_t)\|\le 2L$ and
$\|\nabla F_S(\mathbf{w}_t)\|\le L$ by Assumption \ref{assum666}, and have used $\eta_t=\frac{2}{\mu_S(t+t_0)}$ and $H_t(W)\le\kappa_\gamma (t+t_0)^2$ in (\ref{eq:kernel-bds}).
Thus a uniform bound is
\begin{align}\label{bound-increment-on]}
|M_t|\ \le\  \frac{4\kappa_\gamma L^2}{\mu_S}\,(T+t_0-1).
\end{align}
Next, by Assumption \ref{assu8}, the conditional variance satisfies
\begin{align}\label{bound-increment-on]-varian}
\mathbb{E}_{j_t}[M_t^2 ]
\ \le\ \eta_t^2 (H_t(W))^2 \sigma^2 \|\nabla F_S(\mathbf{w}_t)\|^2
\ \le\ \frac{4\kappa_\gamma^2\sigma^2}{\mu_S^2}\,(t+t_0)^2 \|\nabla F_S(\mathbf{w}_t)\|^2.
\end{align}
Apply Part (b) of Lemma \ref{lemma36} with (\ref{bound-increment-on]}) and (\ref{bound-increment-on]-varian}) yields, with probability at least $1-\delta$,
\begin{equation}\label{eq:B3-final}
(B3)\ \le\ \frac{\rho \kappa_\gamma \sigma^2 }{\mu_S (T+t_0-1) L^2}\, \sum_{t=1}^{T-1}(t+t_0)^2 \|\nabla F_S(\mathbf{w}_t)\|^2\  +\ \frac{4\kappa_\gamma L^2}{\mu_S\rho}\,(T+t_0-1)\log\frac{1}{\delta}.
\end{equation}
Setting $\rho=\min \{1,\frac{\gamma L^2}{2\kappa_\gamma \sigma^2}  \}$, we have
\begin{align*}
(B3)\ \le\ \frac{\gamma }{2\mu_S }\, \sum_{t=1}^{T-1}(t+t_0) \|\nabla F_S(\mathbf{w}_t)\|^2\  +\ \frac{4\kappa_\gamma L^2}{\mu_S \min \{1,\frac{\gamma L^2}{2\kappa_\gamma \sigma^2}  \}}\,(T+t_0-1)\log\frac{1}{\delta}.
\end{align*}

\medskip\noindent\textbf{Step 5.5: Taken together.}
Collecting bounds on $(B1)$–$(B3)$ yields, with probability at least $1-\delta$,
\begin{align*}
\Sigma_1
&\le\; \frac{16\beta\kappa_\gamma^2}{1-\gamma}\cdot \frac{L^2}{\mu_S^2}\,T
 -\frac{\gamma}{\mu_S}\sum_{t=1}^{T-1}(t+t_0+1)\,\|\nabla F_S(\mathbf{w}_t)\|^2
\ +\ \frac{\kappa_\gamma \beta^2\gamma^2}{\mu_S}\frac{\kappa_\gamma^{(1)} L^2}{\gamma(1-\gamma)}   \frac{4}{\mu_S^2  } \log\!\frac{T+t_0-1}{t_0}\\&+ \ \frac{\gamma }{2\mu_S }\, \sum_{t=1}^{T-1}(t+t_0) \|\nabla F_S(\mathbf{w}_t)\|^2\  +\ \frac{4\kappa_\gamma L^2}{\mu_S \min \{1,\frac{\gamma L^2}{2\kappa_\gamma \sigma^2}  \}}\,(T+t_0-1)\log\frac{1}{\delta}
\\&\le\frac{16\beta\kappa_\gamma^2}{1-\gamma}\cdot \frac{L^2}{\mu_S^2}\,T + \frac{\kappa_\gamma  \beta^2\gamma^2}{\mu_S}\frac{\kappa_\gamma^{(1)} L^2}{\gamma(1-\gamma)}   \frac{4}{\mu_S^2  } \log\!\frac{T+t_0-1}{t_0}+\ \frac{4\kappa_\gamma L^2}{\mu_S \min \{1,\frac{\gamma L^2}{2\kappa_\gamma \sigma^2}  \}}\,(T+t_0-1)\log\frac{1}{\delta}.
\end{align*}

\medskip\noindent\textbf{Step 6: Final bound.}
Till here, we collect all these bounds on $\Sigma_{1}$, $\Sigma_{2a}$, $\Sigma_{2b}$, $\Sigma_{3}$ and $\Sigma_{4}$ together. Plugging these bounds into (\ref{eq:master}), with probability $1-2\delta$ we get
\begin{align*}
&\sum_{t=1}^T \frac{t+t_0-1}{2\mu_S}\|\nabla F_S(\mathbf{w}_t)\|^2
+ W_T\big(F_S(\mathbf{w}_{T+1})-F_S^{\ast}\big)
\\&\le\frac{16\beta\kappa_\gamma^2}{1-\gamma}\cdot \frac{L^2}{\mu_S^2}\,T + \frac{\kappa_\gamma \beta^2\gamma^2}{\mu_S}\frac{\kappa_\gamma^{(1)} L^2}{\gamma(1-\gamma)}   \frac{4}{\mu_S^2  } \log\!\frac{T+t_0-1}{t_0}+\ \frac{4\kappa_\gamma L^2}{\mu_S \min \{1,\frac{\gamma L^2}{2\kappa_\gamma \sigma^2}  \}}\,(T+t_0-1)\log\frac{1}{\delta}\\&+\frac{4\beta^2\gamma^2\kappa_\gamma^{(1)}}{(1-\gamma)\mu_S^3}\,L^2\,
\log\!\frac{T+t_0}{t_0}+\frac{16\beta\gamma^2\kappa_\gamma}{1-\gamma}\cdot \frac{L^2}{\mu_S^2}\,T + \frac{1}{ 4\mu_S}\sum_{t=1}^T (t+t_0-1)\|\nabla F_S(\mathbf{w}_t)\|^2 +\frac{4}{\mu_S}(T+t_0-1)\frac{L^2\log(1/\delta)}{\min\{1,\frac{L^2}{4\sigma^2}\}} +T\frac{4\beta L^2}{\mu_S^2},
\end{align*}
which means that with probability $1-\delta$
\begin{align*}
& W_T\big(F_S(\mathbf{w}_{T+1})-F_S^{\ast}\big)
 \le\frac{16\beta\kappa_\gamma^2}{1-\gamma}\cdot \frac{L^2}{\mu_S^2}\,T + \frac{\kappa_\gamma \beta^2\gamma^2}{\mu_S}\frac{\kappa_\gamma^{(1)} L^2}{\gamma(1-\gamma)}   \frac{4}{\mu_S^2  } \log\!\frac{T+t_0-1}{t_0}+\ \frac{4\kappa_\gamma L^2}{\mu_S \min \{1,\frac{\gamma L^2}{2\kappa_\gamma \sigma^2}  \}}\,(T+t_0-1)\log\frac{2}{\delta}\\&+\frac{4\beta^2\gamma^2\kappa_\gamma^{(1)}}{(1-\gamma)\mu_S^3}\,L^2\,
\log\!\frac{T+t_0}{t_0}+\frac{16\beta\gamma^2\kappa_\gamma}{1-\gamma}\cdot \frac{L^2}{\mu_S^2}\,T  +\frac{4}{\mu_S}(T+t_0-1)\frac{L^2\log(2/\delta)}{\min\{1,\frac{L^2}{4\sigma^2}\}} +T\frac{4\beta L^2}{\mu_S^2}.
\end{align*}
The inequality implies the claimed bound
\begin{align*}
&F_S(\mathbf{w}_{T+1})-F_S^{\ast} = \mathcal O\left(\frac{\log 1/\delta}{T}
\right).
\end{align*}
The proof is complete.
\end{proof}
\begin{remark}[Novelty of NAG]\rm{}\label{optimiofsnag}
Proving high-probability guarantees for NAG is strictly harder than for SGD: NAG couples the iterate $\mathbf w_t$ with the lookahead $\mathbf y_t$ and the momentum $\mathbf m_t$, creating momentum–gradient cross-terms and geometric noise accumulation that standard SGD analyses cannot tightly control. We tackle this by developing a Lyapunov (energy–potential) framework that explicitly tracks momentum and geometrically reweights history. To handle the stochastic coupling, we apply an absorption inequality to the cross-term $\langle \nabla F_S(\mathbf w_t), \mathbf m_t\rangle$ and combine it with a geometrically time-weighted martingale argument that controls both the accumulated momentum and the $\mathbf w_t$-$\mathbf y_t$ interaction. These ingredients yield high-probability, last-iterate optimization guarantees for nonconvex objectives. 
\end{remark}

\subsection{Proof of Theorem \ref{theo6nag}}\label{section-proof-nag}
\begin{proof}
When Assumptions \ref{assu4}, \ref{assu7} and \ref{assu8} hold and when $\eta_t = \eta_1 t^{- 1/2}$ with $\eta_1 \leq \min\!\left\{\frac{1-\gamma}{2\sqrt{2}\,\gamma \beta},\ 
\frac{(1-\gamma)^2}{32\,C_m(\gamma,\beta)}\right\}$,  we can apply Lemma \ref{thm:nag-noncvx} to obtain the following inequality with probability at least $1-\delta/3$,
\begin{align*}
&\sum_{t = 1}^T \eta_t \| \nabla F(\mathbf{w}_t) \|^2 = \sum_{t = 1}^T \eta_t \| \nabla F(\mathbf{w}_t) - \nabla F_S(\mathbf{w}_t) + \nabla F_S(\mathbf{w}_t)\|^2  \leq 2\sum_{t = 1}^T \eta_t \| \nabla F(\mathbf{w}_t) - \nabla F_S(\mathbf{w}_t) \|^2 + 2\sum_{t = 1}^T \eta_t\| \nabla F_S(\mathbf{w}_t)\|^2 \\
&\leq 2\sum_{t = 1}^T \eta_t \max_{t=1,...,T}\| \nabla F(\mathbf{w}_t) - \nabla F_S(\mathbf{w}_t) \|^2  + \mathcal{O} \left(\sum_{t=1}^T \eta_t^2 + \log \left(\frac{1}{\delta} \right) \right),
\end{align*}
which implies that
\begin{align}\label{boun-weighted-generalzaition}
&\left(\sum_{t = 1}^T \eta_t \right)^{-1} \sum_{t = 1}^T \eta_t \| \nabla F(\mathbf{w}_t) \|^2  \leq 2 \max_{t=1,...,T}\| \nabla F(\mathbf{w}_t) - \nabla F_S(\mathbf{w}_t) \|^2   + \left(\sum_{t = 1}^T \eta_t \right)^{-1}\mathcal{O} \left(\sum_{t=1}^T \eta_t^2 + \log \left(\frac{1}{\delta} \right) \right).
\end{align}
When Assumptions \ref{assu4} and \ref{assu5} are satisfied, we can apply Lemma \ref{erfgefge} to obtain the following inequality with probability at least $1-\frac{\delta}{3}$,
\begin{align}\label{ineqdie}  \max_{t=1,...,T}\| \nabla F(\mathbf{w}_t) - \nabla F_S(\mathbf{w}_t) \|^2   
\leq   \max_{t=1,...,T} \Big[  c' \beta \max \left \{ \| \mathbf{w}_t - \mathbf{w}^{\ast} \|, \frac{1}{n} \right\} 
\eta  + \frac{B_{\ast}\log(\frac{12}{\delta})}{n} + \sqrt{\frac{2 \mathbb{E} [ \| \nabla f(\mathbf{w}^{\ast};z) \|^2 ] \log(\frac{12}{\delta})}{n}} \Big]^2   
\end{align}
where $\eta = \sqrt{\frac{d + \log \frac{24 \log_2(2n R +2)}{\delta}}{n}}  +\frac{d + \log \frac{24 \log_2(2n R +2)}{\delta}}{n} $.

When Assumptions \ref{assu4}, \ref{assu7} and \ref{assu8} hold and when $\eta_t = \eta_1 t^{- 1/2}$ with $\eta_1 \leq \min\!\left\{\frac{1-\gamma}{2\sqrt{2}\,\gamma \beta},\ 
\frac{(1-\gamma)^2}{32\,C_m(\gamma,\beta)}\right\}$, from Lemma \ref{thm:nag-generalization}, we have the following inequality with probability $1-\delta/3$ uniformly for all $t = 1,..,T$
\begin{align}\label{ineq die} \| \mathbf{w}_{t+1} - \mathbf{w}^{\ast} \| \le \| \mathbf{w}_{t+1} \| +\| \mathbf{w}^{\ast} \|  \leq  \mathcal O \left(  ( \sum_{k=1}^T \eta_k^2  )^{1/2} + ( \sum_{k=1}^t \eta_k^2  )^{1/2} ( \sum_{k=1}^t \eta_k  )^{1/2} + 1 \right) \log(\frac{1}{\delta}) \leq   
             \mathcal{O}(\log(\frac{1}{\delta})) T^{\frac{1}{4}} \log^{1/2}T, 
\end{align}
where we have used Lemma \ref{lei32} in the last inequality.
Combining (\ref{ineq die}) and (\ref{ineqdie}),  we can derive the following result with probability at least $1-\frac{2\delta}{3}$,
\begin{align}\label{buineqdie} 
&\max_{t=1,...,T}\| \nabla F(\mathbf{w}_t) - \nabla F_S(\mathbf{w}_t) \|^2  = \mathcal{O} \left( \log^2(\frac{1}{\delta})) T^{\frac{1}{2}} \log T \frac{d + \log \frac{24 \log_2(2n R +2)}{\delta}}{n} \right) .
\end{align}
 Plugging (\ref{buineqdie}) into (\ref{boun-weighted-generalzaition}), and using Lemma \ref{lei32} for the term $(\sum_{t = 1}^T \eta_t )^{-1}\mathcal{O} (\sum_{t=1}^T \eta_t^2 + \log (\frac{1}{\delta} ) )$, we finally obtain the following inequality with probability at least $1-\delta$
\begin{align*}
&\hphantom{{}={}}\left(\sum_{t = 1}^T \eta_t \right)^{-1} \sum_{t = 1}^T \eta_t \| \nabla F(\mathbf{w}_t) \|^2  =
            \mathcal{O}\left(\frac{d + \log \frac{ 1 }{\delta}}{n}\log^2(1/\delta) T^{\frac{1}{2}} \log T \right)+\mathcal{O}(\log(T/\delta)T^{-\frac{1}{2}}).
\end{align*}

Selecting $T \asymp nd^{-1}$, we obtain the following result with probability at least $1-\delta$
\begin{align*}
\left(\sum_{t = 1}^T \eta_t \right)^{-1} \sum_{t = 1}^T \eta_t \| \nabla F(\mathbf{w}_t) \|^2  = 
\mathcal{O} \left( \sqrt{\frac{d}{n}} \log(\frac{n}{d\delta})  \log^3(1/\delta) \right).
\end{align*}
The proof is complete.
\end{proof}

\subsection{Proof of Theorem \ref{theo67nag}}\label{setion-proof-4.4}
\begin{proof}
When Assumptions \ref{assu4} and \ref{assu5} hold and $F$ satisfies Assumption \ref{assu10} with parameter $\mu$, and when $n \geq \frac{c\beta^2(d+ \log(\frac{8 \log(2n R +2)}{\delta}))}{\mu^2}$, we can apply Lemma \ref{erfgefge} to obtain the following inequality with probability at least $1- \delta$
\begin{align*}
&\|  \nabla F(\mathbf{w})  \| 
\leq 2\left\| \nabla F_S(\mathbf{w}) \right\| +  \frac{\mu}{n} + 2\frac{B_{\ast}\log(4/\delta)}{n} + 2\sqrt{\frac{2 \mathbb{E} [ \| \nabla f(\mathbf{w}^{\ast};z) \|^2 ] \log(4/\delta)}{n}},
\end{align*}
which implies, with probability at least $1- \delta/2$, 
\begin{align}\label{theo1111} \left(\sum_{t = 1}^T \eta_t \right)^{-1} \sum_{t = 1}^T \eta_t \| \nabla F(\mathbf{w}_t) \|^2  \leq 16\left(\sum_{t = 1}^T \eta_t \right)^{-1} \sum_{t = 1}^T \eta_t\left\| \nabla F_S(\mathbf{w}_t) \right\|^2 +  \frac{4\mu^2}{n^2}+ \frac{16B_{\ast}^2\log^2(8/\delta)}{n^2} + \frac{32 \mathbb{E} [ \| \nabla f(\mathbf{w}^{\ast};z) \|^2 ] \log(8/\delta)}{n}.
\end{align}
When Assumptions \ref{assu4}, \ref{assu10} and \ref{assu5} hold and when $\eta_t = \eta_1 t^{- 1/2}$ with $\eta_1 \leq \min\!\left\{\frac{1-\gamma}{2\sqrt{2}\,\gamma \beta},\ 
\frac{(1-\gamma)^2}{32\,C_m(\gamma,\beta)}\right\}$,  we can apply Lemma \ref{thm:nag-noncvx} to  obtain the following inequality with probability at least $1-\delta/2$,
\begin{align}\label{theo1112}
&\left(\sum_{t = 1}^T \eta_t \right)^{-1} \sum_{t = 1}^T \eta_t\left\| \nabla F_S(\mathbf{w}_t) \right\|^2  \leq \left(\sum_{t = 1}^T \eta_t \right)^{-1}\mathcal{O} \left(\sum_{t=1}^T \eta_t^2 + \log \left(\frac{1}{\delta} \right) \right).
\end{align}
Combining (\ref{theo1111}) and (\ref{theo1112}), with probability at least $1-\delta$ 
\begin{align*}
 \left(\sum_{t = 1}^T \eta_t \right)^{-1} \sum_{t = 1}^T \eta_t \| \nabla F(\mathbf{w}_t) \|^2  \leq \left(\sum_{t = 1}^T \eta_t \right)^{-1}\mathcal{O} \left(\sum_{k=1}^T \eta_k^2 + \log \left(\frac{1}{\delta} \right) \right)  + \mathcal{O} \left( \frac{\log^2(1/\delta)}{n^2} + \frac{ \mathbb{E} [ \| \nabla f(\mathbf{w}^{\ast};z) \|^2 ] \log(1/\delta)}{n} \right).
\end{align*}
Together with Lemma \ref{lei32} and (\ref{smoothco}),
we finally obtain the following inequality with probability at least $1-\delta$,
\begin{align*}
\left(\sum_{t = 1}^T \eta_t \right)^{-1} \sum_{t = 1}^T \eta_t \| \nabla F(\mathbf{w}_t) \|^2= 
             \mathcal{O} \left( \frac{\log^2(\frac{1}{\delta})}{n^2} + \frac{ F(\mathbf{w}^{\ast}) \log(\frac{1}{\delta})}{n} \right)  + \mathcal{O}\left(\log(\frac{T}{\delta}) T^{-\frac{1}{2}}\right).
\end{align*}
Moreover, when $F$ satisfies the PL condition with parameter $\mu$, we have
\begin{align*}
F(\mathbf{w}) - F^{\ast}  \leq \frac{\left\| \nabla F(\mathbf{w}) \right\|^2 }{2\mu}, \quad \forall  \mathbf{w} \in \mathcal{W}.
\end{align*} 
Selecting $T \asymp n^4$, then we obtain the following result with probability at least $1-\delta$
\begin{align*}
&\left(\sum_{t = 1}^T \eta_t \right)^{-1} \sum_{t = 1}^T \eta_t F(\mathbf{w}_{t}) - F^{\ast}=
\mathcal{O} \left(\frac{\log^2(\frac{1}{\delta})}{n^2} + \frac{ F(\mathbf{w}^{\ast}) \log(\frac{1}{\delta})}{n}\right).
\end{align*}
The proof is complete.
\end{proof}
\subsection{Proof of Theorem \ref{theo7nag}}\label{section-proof-4.5bag}
\begin{proof}
Since $F$ satisfies the PL condition with parameter $2\mu$, we have
\begin{align}\label{pll1}
F(\mathbf{w}) - F^{\ast}  \leq \frac{\left\| \nabla F(\mathbf{w}) \right\|^2 }{4 \mu}, \quad \forall  \mathbf{w} \in \mathcal{W}.
\end{align} 
To bound $F(\mathbf{w}_{T+1}) - F^{\ast}$, we need to bound the term $\left\| \nabla F(\mathbf{w}_{T+1}) \right\|^2$.
It is clear
\begin{align}\label{pll}
&\left\| \nabla F(\mathbf{w}_{T+1}) \right\|^2 \leq  2 \left\| \nabla F(\mathbf{w}_{T+1})- \nabla F_S(\mathbf{w}_{T+1}) \right\|^2 + 2 \| \nabla F_S(\mathbf{w}_{T+1}) \|^2.
\end{align}
When assumptions \ref{assum666}, \ref{assu4} and \ref{assu8} hold and $F_S$ satisfies the PL condition, we can apply Lemma \ref{thm:nag-corrected-final} and (\ref{pltogradientwithfs}) to obtain the following inequality with probability at least $1-\delta/2$
\begin{align}\label{last008}
\| \nabla F_S(\mathbf{w}_{T+1})  \|^2 = \mathcal{O} \left(\frac{ \log (1/\delta)}{T} \right).
\end{align}
When Assumptions \ref{assu4} and \ref{assu5} hold and $F$ satisfies the PL condition, and
 when $n \geq \frac{c\beta^2(d+ \log(\frac{16 \log(2n R +2)}{\delta}))}{\mu^2}$, we can apply  Lemma \ref{erfgefge} to obtain the following inequality with probability at least $1- \delta/2$ 
\begin{align*}
&\left\| \nabla F (\mathbf{w}_{T+1} )- \nabla F_S(\mathbf{w}_{T+1}) \right\|  \leq \left\| \nabla F_S(\mathbf{w}_{T+1}) \right\| +  \frac{2\mu}{n} + 2\frac{B_{\ast}\log(8/\delta)}{n} + 2\sqrt{\frac{2 \mathbb{E} [ \| \nabla f(\mathbf{w}^{\ast};z) \|^2 ] \log(8/\delta)}{n}}\\
&\leq \left\| \nabla F_S(\mathbf{w}_{T+1}) \right\| +  \frac{2\mu}{n} + 2\frac{B_{\ast}\log(8/\delta)}{n} + 2\sqrt{\frac{8 \beta F(\mathbf{w}^{\ast}) \log(8/\delta)}{n}}, 
\end{align*}
where the last inequality follows from (\ref{smoothco}).
Involving the bound in (\ref{last008}), we can derive that with probability at least $1- \delta$ 
\begin{align}\label{last007} 
&\left\| \nabla F(\mathbf{w}_{T+1})- \nabla F_S(\mathbf{w}_{T+1}) \right\|^2 
 = \mathcal{O} \left(\frac{ \log(\frac{1}{\delta})}{T} \right) + \mathcal{O} \left(\frac{\log^2(\frac{1}{\delta})}{n^2} + \frac{ F(\mathbf{w}^{\ast}) \log(\frac{1}{\delta})}{n}\right).
\end{align}
Substituting (\ref{last007}) and (\ref{last008}) into (\ref{pll}),
we have the following inequality with probability at least $1-\delta$
\begin{align}\label{lsj95zj} 
&\left\| \nabla F(\mathbf{w}_{T+1}) \right\|^2  =  \mathcal{O} \left(\frac{\log(\frac{1}{\delta})}{T}\right) + \mathcal{O} \left(\frac{\log^2(\frac{1}{\delta})}{n^2} + \frac{ F(\mathbf{w}^{\ast}) \log(\frac{1}{\delta})}{n}\right).
\end{align}
Further substituting (\ref{lsj95zj}) into (\ref{pll1}) and selecting $T \asymp n^2$,
we obtain the following inequality with probability at least $1-\delta$
\begin{align*}
F(\mathbf{w}_{T+1}) - F^{\ast}  = \mathcal{O} \left(\frac{ \log^2(1/\delta)}{n^2} + \frac{ F(\mathbf{w}^{\ast}) \log(\frac{1}{\delta})}{n}\right).
\end{align*}
The proof is complete.
\end{proof}

\end{document}